\documentclass{article}

\usepackage{microtype}
\usepackage{graphicx}
\usepackage{booktabs} 
\usepackage{xcolor}
\usepackage{enumitem}
\usepackage{todonotes}
\usepackage{natbib}
\usepackage{algorithm}
\usepackage{fullpage} 
\usepackage[noend]{algpseudocode}
\usepackage{xpatch}
\usepackage[margin=0.5cm]{subcaption}
\usepackage{stackengine}

\usepackage{amsmath}
\usepackage{amsfonts}
\usepackage{appendix}
\usepackage{amsthm}

\newtheorem{theorem}{Theorem}
\newtheorem{proposition}{Proposition}
\newtheorem{lemma}{Lemma}
\newtheorem{corollary}{Corollary}

\newtheoremstyle{remark}{}{}{}{}{\bfseries}{}{ }{}

\theoremstyle{remark}
\newtheorem*{remarks*}{Remarks.}
\newtheorem*{example*}{Example.}
\newtheorem*{comparison*}{Comparison.}

\usepackage{multirow}

\DeclareMathOperator{\card}{card}

\usepackage{hyperref}

\graphicspath{ {Figures/} }

\newcommand{\defn}{:=}
\newcommand{\slim}{\sum\limits}
\newcommand{\reals}{\mathbb{R}}
\newcommand{\gaussian}{\mathcal{N}}

\newcommand{\numpapers}{m}
\newcommand{\numreviewers}{n}
\newcommand{\sizeofconf}{k}
\newcommand{\maxrevload}{\mu}
\newcommand{\vectmaxrevload}{\overline{\mu}}
\newcommand{\paperload}{\lambda}
\newcommand{\similarity}{s}
\newcommand{\simmatrix}{S}
\newcommand{\scoregiven}{y}
\newcommand{\assignmentfamily}{\mathcal{A}}
\newcommand{\assignment}{A}
\newcommand{\assignmentelement}{A}

\newcommand{\fairassignment}{\assignment^{\text{PR4A}}}
\newcommand{\gargassignment}{\assignment^{\text{ILPR}}}

\newcommand{\estimator}{\estpaperquality}
\newcommand{\mle}{{\estimator}^{\text{MLE}}}
\newcommand{\avgest}{{\estimator}^{\text{MEAN}}}

\newcommand{\prob}[1]{\mathbb{P} \left\{ #1 \right\}}
\newcommand{\hamming}[2]{\mathcal{D}_H \left(#1, #2\right)}
\newcommand{\tol}{t}
\newcommand{\accepted}{\mathcal{T}}
\newcommand{\truebest}{\accepted^{\ast}_{\sizeofconf}}
\newcommand{\truebestsubj}{\accepted^{\star}_{\sizeofconf}}
\newcommand{\reviewerset}[1]{\mathcal{R}_{#1}}
\newcommand{\algo}{{\sc PeerReview4All} }
\newcommand{\algodot}{{\sc PeerReview4All}}
\newcommand{\algtpms}{{\sc TPMS} }
\newcommand{\algtpmsdot}{{\sc TPMS}}
\newcommand{\alghard}{{\sc Hard} }
\newcommand{\algharddot}{{\sc Hard}}
\newcommand{\alggarg}{{\sc ILPR} }
\newcommand{\alggargdot}{{\sc ILPR}}
\renewcommand{\paperquality}{\theta}
\newcommand{\truepaperquality}{\paperquality^{\ast}}
\newcommand{\std}{\sigma}
\newcommand{\estpaperquality}{\widehat{\paperquality}}
\newcommand{\threshold}{\Delta_{\sizeofconf}}
\newcommand{\thresholdham}{\Delta_{\sizeofconf, \tol}}
\newcommand{\criticalstd}{\widetilde{\std}}
\newcommand{\exponent}[1]{\exp\left\{ #1 \right\}}
\newcommand{\problemfamily}{\mathcal{F}_{\sizeofconf}}
\newcommand{\problemfamilyham}{\mathcal{F}_{\sizeofconf, \tol}}
\newcommand{\distance}{\delta}
\newcommand{\expectation}[1]{\mathbb{E}\left\{ #1 \right\}}
\newcommand{\simmatrixfamily}{\mathcal{S}}
\newcommand{\minsim}{\zeta}
\newcommand{\errorrate}{\epsilon}
\newcommand{\const}{c}
\newcommand{\bias}{b}
\newcommand{\subjectivescore}{\widetilde{\paperquality}}
\newcommand{\avgpaperquality}{\widetilde{\paperquality}^{\star}}

\newcommand{\tpms}{\assignment^{\text{TPMS}}}
\newcommand{\greedyassignmentqualitysym}{G^{\simmatrix}}
\newcommand{\greedyassignmentquality}[1]{\greedyassignmentqualitysym\left( #1 \right)}
\newcommand{\assignmentqualitysym}{\Gamma^{\simmatrix}}
\newcommand{\assignmentquality}[2][]{\assignmentqualitysym_{#1}\left( #2 \right)}
\newcommand{\assignmentqualitymanual}[3][]{\Gamma_{#1}^{#2}\left( #3 \right)}
\newcommand{\tmpcapacity}{\kappa}
\newcommand{\tmpsimmatrix}{\widetilde{\simmatrix}}
\newcommand{\criticalsim}{\similarity^{\ast}}
\newcommand{\hardassignment}{\assignment^{\text{HARD}}}
\newcommand{\approximation}{\tau}
\newcommand{\naturals}{\mathbb{N}}
\newcommand{\subjscorematrix}{\widetilde{\Theta}}
\newcommand{\genvariance}{h}

\newcommand{\highqualitycrowd}{q}
\newcommand{\singlehighqual}{{v}}
\newcommand{\goodavgassignment}{\simmatrixfamily}
\newcommand{\onegoodforeveryone}{\simmatrixfamily_{\tmpcapacity}}
\newcommand{\vectmaxrevloadtmp}{\vectmaxrevload^{\text{tmp}}}
\newcommand{\simmatrixtmp}{\simmatrix^{\text{tmp}}}
\newcommand{\cumvar}{\bar{\sigma}}
\newcommand{\tobeassigned}{\mathcal{M}}
\newcommand{\worstoff}{\mathcal{J}^{\ast}}

\newcommand{\transformation}{f}
\newcommand{\indicator}[1]{\mathbb{I}\left\{ #1 \right\}}

\newcommand{\tmpassignment}{\widetilde{\assignment}}

\newcommand{\smallnumrev}{\eta}
\newcommand{\estpaperqualitymle}{\estpaperquality^{\text{MLE}}}
\newcommand{\estpaperqualityavg}{\estpaperquality^{\text{MEAN}}}

\newcommand{\transfavg}{1 - \genvariance}
\newcommand{\transfmle}{\genvariance^{-1}}
\newcommand{\constnua}{\nu_1}
\newcommand{\constnub}{\nu_2}
\newcommand{\rand}{\assignment^{\text{RAND}}}

\newcommand{\goodrevs}{\gamma}
\newcommand{\numhyp}{L}
\newcommand{\hypset}{\mathcal{P}}

\newcommand{\problemvar}{P}
\newcommand{\scorematrix}{Y}
\newcommand{\scoredist}{\mathbb{P}}
\newcommand{\hypest}{\varphi}
\newcommand{\KL}[2]{\text{KL}\left[#1 || #2 \right]}
\newcommand{\levbound}{L}
\newcommand{\str}{b}
\newcommand{\argmax}{\operatornamewithlimits{arg~max}}
\newcommand{\argmin}{\operatornamewithlimits{arg~min}}

\newcommand{\paperset}{\mathcal{J}}
\newcommand{\tmpset}{\mathcal{M}}

\begin{document}
\title{{PeerReview4All}:\\ Fair and Accurate Reviewer Assignment in  Peer Review}
\author{\\
  Ivan Stelmakh, Nihar B. Shah and Aarti Singh\\~\\
  School of Computer Science \\ 
  Carnegie Mellon University\\
  \texttt{\{stiv,nihars,aarti\}@cs.cmu.edu} 
}
\date{}
\maketitle

\begin{abstract}
We consider the problem of automated assignment of papers to reviewers in conference peer review, with a focus on fairness and statistical accuracy. Our fairness objective is to maximize the review quality of the most disadvantaged paper, in contrast to the commonly used objective of maximizing the total  quality over all papers. We design an assignment algorithm based on an incremental max-flow procedure that we prove is near-optimally fair. Our statistical accuracy objective is to ensure correct recovery of the papers that should be accepted. We  provide a sharp minimax analysis of the accuracy of the peer-review process for a popular objective-score model as well as for a novel subjective-score model that we propose in the paper. Our analysis proves that our proposed assignment algorithm also leads to a near-optimal statistical accuracy. 
Finally, we design a novel experiment that allows for an objective comparison of various assignment algorithms, and overcomes the inherent difficulty posed by the absence of a ground truth in experiments on peer-review. The results of this experiment as well as of other experiments on synthetic and real data corroborate the theoretical guarantees of our algorithm.
\end{abstract}


\section{Introduction}
\label{intro}

Peer review is the backbone of academia. In order to provide high-quality peer reviews, it is of utmost importance to assign papers to the right reviewers~\citep{thurner2011peer,black1998makes,bianchi2015three}. Even a small fraction of incorrect reviews can have significant adverse effects on the quality of the published scientific standard~\citep{thurner2011peer} and dominate the benefits yielded by the peer-review process that may have high standards otherwise~\citep{squazzoni2012saint}. Indeed, researchers unhappy with the peer review process are somewhat more likely to link their objections to the quality or choice of reviewers~\citep{travis1991new}.

We focus on peer-review in conferences where a number of papers are submitted at once. These papers must simultaneously be assigned to multiple reviewers who  have load constraints. The importance of the reviewer-assignment stage of the peer-review process cannot be overestimated; quoting~\cite{rodriguez2007mapping}:
\begin{quote}
{\it ``one of the first and potentially most important stage is the one that attempts to distribute submitted manuscripts to competent referees.''  }
\end{quote}
Given the massive scale of many conferences such as NeurIPS and ICML, these reviewer assignments are largely performed in an automated manner. For instance, NeurIPS 2016 assigned 5 out of 6 reviewers per paper using an automated process~\citep{shah2017design}. This problem of automated reviewer assignments forms the focus of this paper.

Various past studies show that small changes in peer review quality can have far reaching consequences~\citep{thorngate2014numbers, squazzoni2012saint} not just for the papers under consideration but more generally also for the career trajectories of the researchers. These long term effects arise due to the widespread prevalence of the Matthew effect (``rich get richer'') in academia~\citep{merton1968matthew}.

It is also known~\citep{travis1991new,lamont2009professors} that works that are novel or not mainstream, particularly those interdisciplinary in nature, face significantly higher difficulty in gaining acceptance. A primary reason for this undesirable state of affairs is the absence of sufficiently many good ``peers'' to aptly review interdisciplinary research~\citep{porter1985peer}. 

These issues strongly motivate the dual goals of the reviewer assignment procedure we consider in this paper --- fairness and accuracy. By fairness, we specifically consider the notion of max-min fairness which is studied in various branches of science and engineering~\citep{rawls1971theory,Lenstra1990,hahne1991round,lavi2003towards,bonald2006queueing,Asadpour10maxmin}. In our context of reviewer assignments, max-min fairness posits maximizing the review-quality of the paper with the least qualified reviewers. The max-min fair assignment guarantees that no paper is discriminated against in favor of more lucky counterparts. That is, even the most ambivalent paper with a small number of reviewers being competent enough to evaluate its merits will receive as good treatment as possible. The max-min fair assignment also ensures that in \emph{any other assignment} there exists at least one paper with the fate at least as bad as the fate of the most disadvantaged paper in the aforementioned fair assignment.

Alongside, we also consider the requirement of statistical accuracy. One of the main goals of the conference peer-review process is to select the set of ``top'' papers for acceptance. Two key challenges towards this goal are to handle the noise in the reviews and subjective opinions of the reviewers; we accommodate these aspects in terms of existing~\citep{ge13bias, mcglohon10starquality, dai12yelp} and novel statistical models of reviewer behavior. Prior works on the reviewer assignment problem~\citep{Long13gooadandfair, Garg2010papers, Karimzadehgan08multiaspect, tang10constraied} offer a variety of algorithms that optimize the assignment for certain deterministic objectives, but do not study their assignments from the lens of statistical accuracy. In contrast, our goal is to design an assignment algorithm that can simultaneously achieve both the desired objectives of fairness and statistical accuracy.

We make several contributions towards this problem. We first present a novel algorithm, which we call \algodot, for assigning reviewers to papers. Our algorithm is based on a construction of multiple candidate assignments, each of which is obtained via an incremental execution of max-flow algorithm on a carefully designed flow network. These assignments cater to different structural properties of the similarities and a judicious choice between them provides the algorithm appealing properties.

Our second contribution is an analysis of the fairness objective that our \algo algorithm can achieve. We show that our algorithm is optimal, up to a constant factor, in terms of the max-min fairness objective. Furthermore, our algorithm can adapt to the underlying structure of the given similarity data between reviewers and papers
and in various cases yield better guarantees including the exact optimal solution in certain scenarios. Finally, after optimizing the outcome for the most worst-off paper and fixing the assignment for that paper, our algorithm aims at finding the most fair assignment for the next worst-off paper and proceeds in this manner until the assignment for each paper is fixed.

As a third contribution, we show that our \algo algorithm results in strong statistical guarantees in terms of correctly identifying the top papers that should be accepted. We consider a popular statistical model~\citep{ge13bias, mcglohon10starquality, dai12yelp} which assumes existence of some true objective score for every paper. We provide a sharp analysis of the minimax risk in terms of 
``incorrect'' accept/reject decisions, and show that our \algo algorithm leads to a near-optimal solution.

Fourth, noting that paper evaluations are typically subjective~\citep{kerr1977manuscript,mahoney1977publication, ernst1994reviewer,bakanic1987manuscript,lamont2009professors}, we propose a novel statistical model capturing subjective opinions of reviewers, which may be of independent interest. We provide a sharp minimax analysis under this subjective setting and prove that our assignment algorithm \algo is also near-optimal for this subjective-score setting.

Our fifth and final contribution comprises empirical evaluations. We designed and conducted an experiment on the Amazon Mechanical Turk crowdsourcing platform to objectively compare the performance of different reviewer-assignment algorithms. The design of the experiment is done carefully to circumvent the challenge posed by the absence of a ground truth in peer review settings, so that we can evaluate  accuracy objectively. In addition to the MTurk experiment, we provide an extensive evaluation of our algorithm on synthetic data, provide an evaluation on a reconstructed similarity matrix from the ICLR 2018 conference, and report the results of the experiment on real conference data conducted by~\citet{kobren19localfairness}. The results of these experiments highlight the promise of \algo in practice, in addition to the theoretical benefits discussed elsewhere in the paper. The dataset pertaining to the MTurk experiment, as well as the code for our \algo algorithm, are available on the first author's website.

The remainder of this paper is organized as follows. We discuss related literature in Section~\ref{sec:literature}. In Section~\ref{sec:problem_setting}, we present the problem setting formally with a focus on the objective of fairness.  In Section~\ref{sec:algorithm} we present our \algo algorithm. We establish deterministic approximation guarantees on the fairness of our \algo algorithm in Section~\ref{sec:algorithm_analysis}. We analyze the accuracy of our \algo algorithm under an objective-score model in Section~\ref{sec:obj_score_model}, and introduce and analyze a subjective score model in Section~\ref{sec:subj_score_model}. We empirically evaluate the algorithm in Section~\ref{sec:experiments} using synthetic and real-world experiments. We then provide the proofs of all the results in Section~\ref{sec:proofs}. We conclude the paper with a discussion in Section~\ref{sec:discussion}.


\section{Related literature}
\label{sec:literature}

The reviewer assignment process consists of two steps. First, a ``similarity'' between every (paper, reviewer) pair that captures the competence of the reviewer for that paper is computed. These similarities are computed based on various factors such as the text of the submitted paper, previous papers authored by reviewers, reviewers' bids and other features. Second, given the notion of good assignment, specified by the program chairs, papers are allocated to reviewers, subject to constraints on paper/reviewer loads. This work focuses on the second step (assignment), assuming the first step of computing similarities as a black box. In this section, we give a brief overview of the past literature on both of the steps of the reviewer-assignment process.

\smallskip
\noindent
{\bf Computing similarities.} 
The problem of identifying similarities between papers and reviewers is well-studied in data mining community. For example, \citet{mimno07topicbased} introduce a novel topic model to predict reviewers' expertise. \citet{liu14graphpropagation} use the random walk with restarts model to incorporate both expertise of reviewers and their authority in the final similarities. Co-authorship graphs~\citep{rodriguez08coauthorsip} and more general bibliographic graph-based data models~\citep{tran17expertsuggestion} give appealing methods which do not require a set of reviewers to be pre-determined by conference chair. Instead, these methods recommend reviewers to be recruited, which might be particularly useful for journal editors. 

One of the most widely used automated assignment algorithms today is the Toronto Paper Matching System or TPMS~\citep{charlin13tpms} which also computes estimations of similarities between submitted papers and available reviewers using techniques in natural language processing. These scores might be enhanced with reviewers' self-accessed expertise adaptively queried from them in an automatic manner. 

Our work uses these similarities as an input for our assignment algorithm, and considers the computation of these similarity values as a given black box. 

\smallskip
\noindent
{\bf Cumulative goal functions.}
With the given similarities, much of past work on reviewer assignments develop algorithms to maximize the cumulative similarity, that is, the sum of the similarities across all assigned reviewers and all papers. Such an objective is pursued by the organizers of SIGKDD conference~\citep{flach2010kdd} and by the widely employed TPMS assignment algorithm~\citep{charlin13tpms}. Various 
other popular conference management systems such as EasyChair (\url{easychair.org}) and HotCRP (\url{hotcrp.com}) and several other papers (see~\citealt{Long13gooadandfair, charlin12framework, goldsmith07aiconf, tang10constraied} and references therein) also aim to maximize various cumulative functionals in their automated reviewer assignment procedures. In the sequel, we argue however that optimizing such cumulative objectives is not fair --- in order to maximize them, these algorithms may discriminate against some subset of papers. Moreover, it is the non-mainstream submissions that are most likely to be discriminated against. With this motivation, we consider a notion of fairness instead. 
 
\newpage
\noindent
{{\bf Fairness.} In order to ensure that no papers are discriminated against, we aim at finding a \emph{fair assignment} --- an assignment that ensures that the most disadvantaged paper gets as competent reviewers as possible.
The issue of fairness is partially tackled by~\citet{Hartvigsen99assignment}, where they necessitate every paper to have at least one reviewer with expertise higher than certain threshold, and then maximize the value of that threshold. However, this improvement only partially solves the issue of discrimination of some papers: having assigned one strong reviewer to each paper, the algorithm may still discriminate against some papers while assigning remaining reviewers. Given that nowadays large conferences such as NeurIPS and ICML assign 4-6 reviewers to each paper, a careful assessment of the paper by one strong reviewer might be lost in the noise induced by the remaining weak reviews. In the present study, we measure the quality of assignment with respect to any particular paper as sum similarity over reviewers assigned to that paper.
Thus, the fairness of assignment is the minimum sum similarity across all papers; we call an assignment fair if it maximizes the fairness. We note that assignment computed by our \algo algorithm  is guaranteed to have   \emph{at least as large} max-min fairness as that proposed by~\citet{Hartvigsen99assignment}.}

~\citet{benferhat2001conference} discuss different approaches to selection of the ``optimal'' reviewer assignment. Together with considering a cumulative objective, they also note that one may define the optimal assignment as an assignment that minimizes a disutility of the most disadvantaged reviewer (paper). This approach resembles the notion of max-min fairness we study in this paper, but~\citet{benferhat2001conference} do not propose any algorithm for computing the fair assignment.   

The notion of max-min fairness was formally studied in context of peer-review by~\citet{Garg2010papers}. While studying a similar objective, our work develops both conceptual and theoretical novelties which we highlight here. First, \citet{Garg2010papers} measure the fairness in terms of reviewers' bids --- for every reviewer they compute a value of papers assigned to that reviewer based on her/his bids and maximize the minimum value across all reviewers.  While satisfying reviewers is a useful practice, we consider fairness towards the papers in their review to be of utmost importance. 
During a bidding process reviewers have limited  time resources and/or limited access to papers' content to evaluate their relevance, and hence reviewers' bids alone are not a good proxy towards the measure of fairness. In contrast, in this work we consider similarities --- scores that are designed to represent a competence of reviewer in assessing a paper. Besides reviewers' bids, similarities are computed based on the full text of the submissions and papers authored by reviewer and can additionally incorporate various factors such as quality of previous reviews, experience of reviewer and other features that cannot be self-assessed by reviewers.

The assignment algorithm proposed in~\citet{Garg2010papers} works in two steps. In the first step, the problem is set up as an integer programming problem and a linear programming relaxation is solved. The second step involves a carefully designed rounding procedure that returns a valid assignment. The algorithm is guaranteed to recover an assignment whose fairness is within a certain additive factor from the best possible assignment. However, the fairness guarantees provided in~\cite{Garg2010papers} turn out to be vacuous for various similarity matrices. As we discuss later in the paper, this is a drawback of the algorithm itself and not an artifact of their guarantees. In contrast, we design an algorithm with multiplicative approximation factor that is guaranteed to always provide a non-trivial approximation which is at most constant factor away from the optimal. 

Next, \citet{Garg2010papers} consider fairness of the assignment as an eventual metric of the assignment quality. However, we note that the main goal of the conference paper reviewing process is an accurate acceptance of the best papers. Thus, in the present work we both theoretically and empirically study the impact of the fairness of the assignment on the quality of the acceptance procedure. 

Finally, although~\citet{Garg2010papers} present their algorithm for the case of discrete reviewer's bids, we note that this assumption can be relaxed to allow real-valued similarities with a continuous range as in our setting. In this paper we refer to the corresponding extension of their algorithm as the Integer Linear Programming Relaxation (\alggargdot) algorithm.

\smallskip
\noindent
{\bf Fair division.}
A direction of research that is relevant to our work studies the problem of fair division where {max-min fairness} is extensively developed. The seminal work of~\citet{Lenstra1990} provides a constant factor approximation to the minimum makespan scheduling problem where the goal is to assign a number of jobs to the unrelated parallel machines such that the maximal running time is minimized. Recently~\citet{Asadpour10maxmin, bansal06santaclaus} proposed approximation algorithms for the problem of assigning a number of indivisible goods to several people such that the least happy person is as happy as possible. However, we note that techniques developed in these papers cannot be directly applied for reviewer assignments problem in peer review due to the various idiosyncratic constraints of this problem. In contrast to the classical formulation studied in these works, our problem setting requires each paper to be reviewed by a fixed number of reviewers and additionally has constraints on reviewers' loads. Such constraints allow us to achieve an approximation guarantee that is independent of the total number of papers and reviewers, and depends only on $\paperload$, the number of reviewers required per paper, as $\frac{1}{\paperload}$. In contrast, the approximation factor of~\citet{Asadpour10maxmin} gets worse at a rate of $\frac{1}{\sqrt{\numpapers} \log^3 \numpapers}$, where $\numpapers$ is a number of persons (papers in our setting).

\smallskip
\noindent
{\bf Statistical aspects.}
Different statistical aspects related to conference peer-review have been studied in the literature. \citet{mcglohon10starquality} and~\citet{dai12yelp} studied aggregation of consumers ratings to generate a ranking of restaurants or merchants. They come up with objective score model of reviewer which we also use in this work. 
\citet{ge13bias} also use similar model of reviewer and propose a Bayesian approach to calibrating reviewer' scores, which allows to incorporate different biases in context of conference peer-review. 
\citet{sajjadi16peergrading} empirically  compare different methods of score aggregation for peer grading of homeworks. Peer grading is a related problem to conference peer review, with the key difference that the questions and answers (``papers'') are more closed-ended and objective. They conclude that although more sophisticated methods are praised in the literature, the simple averaging algorithm demonstrates better performance in their experiment. Another interesting observation they make is an edge of cardinal grades over ordinal in their setup. In this work we also consider the conferences with cardinal grading scheme of submissions.  

To the best of our knowledge, no prior works on conference peer-review has studied the entire pipeline --- from assignment to acceptance --- from a statistical point of view. In this work we take the first steps to close this gap and provide a strong minimax analysis of na\"ive yet interesting procedure of determining top $\sizeofconf$ papers. Our findings suggest that higher fairness of the assignment leads to better quality of acceptance procedure. We consider both the objective score model \citep{ge13bias, mcglohon10starquality, dai12yelp} and a novel subjective-score model that we propose in the present paper.

\smallskip
\noindent
{\bf Coverage and Diversity.} For completeness, we also discuss several related works that study reviewer assignment problem. 

\citet{Li15concert} present a greedy algorithm that tries to avoid assigning a group of stringent reviewers or a group of lenient reviewers to a submission, thus maintaining diversity of the assignment in terms of having different combinations of reviewers assigned to different papers.

Another way to ensure diversity of the assignment is proposed by~\citet{liu14graphpropagation}. Instead of designing the special assignment algorithm, they try to incentivize the diversity by special construction of similarities. Besides incorporating expertise and authority of reviewers in similarities, they add an additional term to the optimization problem which balances similarities by increasing scores for reviewers from different research areas.

\citet{Karimzadehgan08multiaspect} consider topic coverage as an objective and propose several approaches to maintain broad coverage, requiring reviewers assigned to paper being expert in different subtopics covered by the paper. They empirically verify that given a paper and a set of reviewers, their algorithms lead to better coverage of paper's topics as compared to baseline technique that assigns reviewers based on some measure of similarity between text of submission and papers authored by reviewers, but does not do topic matching. 

A similar goal is formally studied by~\citet{Long13gooadandfair}. They measure the coverage of the assignment in terms of the total number of distinct topics of papers covered by the assigned reviewers. They propose a constant factor approximation algorithm that benefits from a sub-modular nature of the objective. As we show in Appendix~\ref{appendix:modifications}, the techniques of~\cite{Long13gooadandfair} can be combined with our proposed algorithm to obtain an assignment which maintains not only fairness, but also a broad topic coverage. 

\smallskip
\noindent
{\bf Research on peer review.} The explosion in the number of submissions in many conferences has spurred research in computer science on improving peer review. 
In addition to problems of fairness and accuracy of the reviewer-paper assignment process, there are a number of challenges in peer review which are addressed in the literature to various extents. These include problems of bias~\citep{tomkins2017reviewer,stelmakh2019testing}, miscalibration~\citep{ge13bias,roos2011calibrate,flach2010kdd,wang2018your}, subjectivity~\citep{noothigattu2018choosing}, strategic behavior~\citep{balietti2016peer,xu2018strategyproof,xu2019strategyproofArxiv}, and others~\citep{nips14experiment,gao2019does}. Of particular interest is the work by~\citet{fiez2019super} which optimizes the process by which reviewers can bid on which papers they prefer to review. In most automated reviewer-paper assignment systems, the bids and the text-matching similarities are then combined~\citep{shah2017design} to form the similarities used to compute the assignment. The bidding and the reviewer-paper assignments are executed separately in current systems, and given the intrinsic relations between the two, it is of interest to jointly design the two systems in the future.

\section{Problem setting}
\label{sec:problem_setting}

In this section we present the problem setting formally with a focus on the objective of fairness.  (We introduce the statistical models we consider in Sections~\ref{sec:obj_score_model} and~\ref{sec:subj_score_model}.)


\subsection{Preliminaries and notation}


Given a collection of $\numpapers \ge 2$ papers, suppose that there exists a true, unknown total ranking of the papers. The goal of the program chair (PC) of the conference is to recover top $\sizeofconf$ papers, for some pre-specified value $\sizeofconf < \numpapers$. In order to achieve this goal, the PC recruits $\numreviewers \ge 2$ reviewers and asks each of them to read and evaluate some subset of the papers. Each reviewer can review a limited number of papers. We let $\maxrevload$ denote the maximum number of papers that any reviewer is willing to review. Each paper must be reviewed by $\paperload$ distinct reviewers. In order to ensure this setting is feasible, we assume that $\numreviewers \maxrevload \ge \numpapers \paperload$. In practice, $\paperload$ is typically small (2 to 6) and hence should conceptually be thought of as a constant.

The PC has access to a similarity matrix $\simmatrix = \left\{\similarity_{ij} \right\} \in [0, 1]^{\numreviewers \times \numpapers}$, where $\similarity_{ij}$ denotes the similarity between any reviewer $i \in [\numreviewers]$ and any paper $j \in [\numpapers]$.\footnote{Here, we adopt the standard notation $[\nu] = \left\{1, 2, \ldots, \nu \right\}$ for any positive integer $\nu$.} 
These similarities are representative of the envisaged quality of the respective reviews: a higher similarity between any reviewer and paper is assumed to indicate a higher competence of that reviewer in reviewing that paper (this assumption is formalized later).  We do not discuss the design of such similarities, but often they are provided by existing systems~\citep{charlin13tpms, mimno07topicbased, liu14graphpropagation, rodriguez08coauthorsip, tran17expertsuggestion}. 

Our focus is on the assignment of papers to reviewers. We represent any assignment by a matrix $\assignment \in \{0,1\}^{\numreviewers \times \numpapers}$, whose $(i,j)^{\text{th}}$ entry is $1$ if reviewer $i$ is assigned paper $j$ and $0$ otherwise. We denote the set of reviewers who review paper $j$ under an assignment $\assignment$ as $\reviewerset{\assignment}(j)$.  We call an assignment  \emph{feasible} if it respects the $
(\maxrevload, \paperload)$ conditions on the reviewer and paper loads. We denote the set of all feasible assignments as $\assignmentfamily$:
\begin{align*}
 \assignmentfamily \defn \Big\{ \assignment \in \{0,1\}^{\numreviewers \times \numpapers} \mid 		\slim_{i \in [\numreviewers]} \assignmentelement_{ij} &= \paperload \ \forall j \in [\numpapers], 
		\sum_{j \in [\numpapers]} \assignmentelement_{ij} \le \maxrevload  \ \forall i \in [\numreviewers]\Big\}.
\end{align*}

Our goal is to design a reviewer-assignment algorithm with a two-fold objective: (i) fairness to all papers, (ii) strong statistical guarantees in terms of recovering the top papers. 

From a statistical perspective, we assume that when any reviewer $i$ is asked to evaluate any paper $j$, then she/he returns score $\scoregiven_{ij} \in \reals$. The end goal of the PC is to accept or reject each paper. In this work we consider a simplified yet indicative setup. We assume that the PC wishes to accept the $\sizeofconf$ ``top'' papers from the set of $\numpapers$ submitted papers. We denote the ``true'' set of top $\sizeofconf$ papers as $\accepted^{\ast}_{\sizeofconf}$. 
While the PC's decisions in practice would rely on several additional factors including the text comments by reviewers and the discussions between them, in order to quantify the quality of any assignment we assume that the top $\sizeofconf$ papers are chosen through some estimator $\estimator$ that operates on the scores provided by the reviewers. Such an estimator can be used in practice to serve as a guide to the program committee in order to help reduce their load. 
These acceptance decisions can be described by the chosen assignment and estimator $\left( \assignment, \estimator \right)$. We denote the set of accepted papers under an assignment $\assignment$ and estimator $\estimator$ as $\accepted_{\sizeofconf} = \accepted_{\sizeofconf}\left(\assignment, \estimator \right)$. The PC then wishes to maximize the probability of recovering the set $\accepted^{\ast}_{\sizeofconf}$ of top $\sizeofconf$ papers.

 Although the goal of exact recovering of top $\sizeofconf$ papers is appealing, given the large number of papers submitted to a conference such as ICML and NeurIPS, this goal might be too optimistic. Another alternative is to recover top $\sizeofconf$ papers allowing for a certain Hamming error tolerance $\tol \in \{0, \ldots, \sizeofconf - 1\}$. For any two subsets $\mathcal{M}_1, \mathcal{M}_2$ of $[\numpapers]$, we define their Hamming distance to be the number of items that belong to exactly one of the two sets --- that is
\begin{align}
	\label{eqn:hamming}
	\hamming{\tmpset_1}{\tmpset_2} = \card\left(\left\{\tmpset_1 \cup \tmpset_2\right\} \backslash \left\{\tmpset_1 \cap \tmpset_2\right\}\right).
\end{align}
The goal of PC under this scenario is to choose a pair $\left(\assignment, \estimator\right)$ such that for the given error tolerance parameter $\tol$, the probability $\prob{\hamming{\accepted_{\sizeofconf}}{\truebest} > 2\tol}$ is minimized. We return to more details on the statistical aspects later in the paper.


\subsection{Fairness objective}
\label{sec:assignment_strategy}

An assignment objective that is popular in past papers~\citep{charlin13tpms, charlin12framework, taylor08assignment} is to maximize the cumulative similarity over all papers. Formally, these works choose an assignment $\assignment \in \assignmentfamily$ which maximizes the quantity 
\begin{align}
\label{eqn:unfair_criteria}
	\greedyassignmentquality{\assignment} \defn 
	\slim_{j=1}^{\numpapers} \slim_{i \in \reviewerset{\assignment}(j)} \similarity_{ij}.
\end{align}
An assignment algorithm that optimizes this objective~\eqref{eqn:unfair_criteria} is implemented in the widely used  Toronto Paper Matching System~\citep{charlin13tpms}. We will refer to the feasible assignment that maximizes the objective~\eqref{eqn:unfair_criteria} as $\tpms$ and denote the algorithm which computes $\tpms$ as \algtpmsdot.

We argue that the objective~\eqref{eqn:unfair_criteria} does not necessarily lead to a \emph{fair} assignment. The optimal assignment can discriminate some papers in order to maximize the cumulative objective. To see this issue, consider the following example.
	
 Consider a toy problem with $\numreviewers = \numpapers = 3$ and $\maxrevload = \paperload = 1$, with a similarity matrix shown in Table~\ref{table:exmple_unfair}. In this example, paper $c$ is easy to evaluate, having non-zero similarities with all the reviewers, while papers $a$ and $b$ are more specific and weak reviewer $2$ has no expertise in reviewing them. Reviewer $1$ is an expert and is able to assess all three papers. Maximizing total sum of similarities~\eqref{eqn:unfair_criteria}, the  \algtpms algorithm  will assign reviewers $1$, $2$, and $3$ to papers $a$, $b$, and $c$ respectively. Observe that under this assignment, paper $b$ is assigned a reviewer who has insufficient expertise to evaluate the paper. On the other hand, the alternative assignment which assigns reviewers $1$, $2$, and $3$ to papers $a$, $c$, and $b$ respectively ensures that every paper has a reviewer with  similarity at least $1/5$. This ``fair'' assignment does not discriminate against papers $a$ and $b$ for improving the review quality of the already benefitting paper $c$. 

\begin{table}[t]
\vskip 0.15in
\begin{center}
\begin{small}
\begin{sc}
\begin{tabular}{lccr}
\toprule
          & Paper $a$ & Paper $b$ & Paper $c$ \\
\midrule
Reviewer $1$ & $1$   & $1$   & $1$  \\
Reviewer $2$ & $0$   & $0$   & $1/5$  \\
Reviewer $3$ & $1/4$ & $1/4$ & $1/2$  \\
\bottomrule
\end{tabular}
\end{sc}
\end{small}
\end{center}
\vskip -0.1in
\caption{Example similarity.} 
\label{table:exmple_unfair}
\end{table}

With this motivation, we now formally describe the notion of fairness that we aim to optimize in this paper.
Inspired by the notion of max-min fairness in a variety of other fields~\citep{rawls1971theory,Lenstra1990,hahne1991round,lavi2003towards,bonald2006queueing,Asadpour10maxmin}, we aim to find a feasible assignment $\assignment \in \assignmentfamily$ to maximize the following objective $\assignmentqualitysym$ for given similarity matrix $\simmatrix$:
\begin{align}
\label{eqn:fairness_criteria}
	 \assignmentquality{\assignment} = \min\limits_{j \in [\numpapers]} \slim_{i \in \reviewerset{\assignment}(j)} \similarity_{ij}.
\end{align}
The assignment optimal for~\eqref{eqn:fairness_criteria} maximizes the minimum sum similarity across all the papers. In other words, for \emph{every other assignment} there exists some paper which has the same or lower sum similarity. Returning to our example, the objective~\eqref{eqn:fairness_criteria} is maximized when reviewers $1$, $2$, and $3$ are assigned to papers $a$, $c$, and $b$ respectively. 

Our reviewer assignment algorithm presented subsequently guarantees the aforementioned fair assignment. Importantly, while aiming at optimizing~\eqref{eqn:fairness_criteria}, our algorithm does even more --- {having the assignment for the worst-off paper fixed, it finds an assignment that satisfies the second worst-off paper, then the next one and so on until all papers are assigned}.

It is important to note that similarities $\similarity_{ij}$ obtained by different techniques~\citep{charlin13tpms, mimno07topicbased, rodriguez08coauthorsip, tran17expertsuggestion} all have different meanings. Therefore, the PC might be interested to consider a slightly more general formulation and aim to maximize
\begin{align}
\label{eqn:gen_fair_criteria}
	\assignmentquality[\transformation]{\assignment} = \min\limits_{j \in [\numpapers]} \slim_{i \in \reviewerset{\assignment}(j)} \transformation(\similarity_{ij}),
\end{align}
for some reasonable choice of monotonically increasing function $\transformation: [0, 1] \to [0, \infty]$.\footnote{We allow $\transformation(\similarity_{ij}) = \infty$. When reviewer with similarity $\infty$ is assigned to paper, she/he is able to perfectly access the quality of the paper.} While the same effect might be achieved by redefining $\similarity_{ij}' = \transformation(\similarity_{ij})$ for all $i \in [\numreviewers], \ j \in [\numpapers]$, this formulation underscores the fact that assignment procedure is not tied to any particular method of obtaining similarities.  Different choices of $\transformation$ represent the different views on the meaning of similarities. As a short example, let us consider $\transformation(\similarity_{ij}) = \indicator{\similarity_{ij} > \minsim}$ for some $\minsim > 0$.\footnote{We use $\mathbb{I}$ to denote the indicator function, that is, $\indicator{x} = 1$ if $x$ is true and $\indicator{x} = 0$ otherwise.} This choice stratifies reviewers for each paper into strong (similarity higher than $\minsim$) and weak. The fair assignment would be such that the most disadvantaged paper is assigned to as many strong reviewers as possible. We discuss other variants of $\transformation$ later when we come to the statistical properties of our algorithm. In what follows we refer to the problem of finding reviewer assignment that maximizes the term~\eqref{eqn:gen_fair_criteria} as the \emph{fair assignment problem}. 

Unfortunately, the assignment optimal for~\eqref{eqn:gen_fair_criteria} is hard to compute for any reasonable choices of function $\transformation$. \citet{Garg2010papers} showed that finding a fair assignment is an NP-hard problem even if $\transformation(\similarity) \in \left\{1, 2, 3 \right\}$ and $\paperload = 2$.

With this motivation, in the next section we design a reviewer assignment algorithm that seeks to optimize the objective~\eqref{eqn:gen_fair_criteria} and provide associated approximation guarantees. We will refer to a feasible assignment that exactly maximizes $\assignmentquality[\transformation]{\assignment}$ as $\hardassignment_{\transformation}$ and denote the algorithm that computes $\hardassignment_{\transformation}$ as \algharddot. When the function $\transformation$ is clear from context, we drop the subscript $\transformation$ and denote the \alghard assignment as $\hardassignment$ for brevity.

Finally we note that for our running example (Table~\ref{table:exmple_unfair} above), the \alggarg algorithm~\citep{Garg2010papers}, despite trying to optimize fairness of the assignment, also returns an unfair assignment $\gargassignment$ which coincides with $\tpms$. The reason for this behavior lies in the inner-working of the \alggarg algorithm: a linear programming relaxation splits reviewers $1$ and $2$ in two and makes them review both paper $a$ and paper $b$. During the rounding stage, reviewer $1$ is assigned to either paper $a$ or paper $b$, ensuring that the remaining paper will be reviewed by reviewer $2$. Given that reviewer $2$ has zero similarity with both papers $a$ and $b$, the fairness of the resulting assignment will be $0$. Such an issue arises more generally in the \alggarg algorithm and is discussed in more detail subsequently in Section~\ref{sec:fairness_literature_compare} and in Appendix~\ref{appendix:garg_discussion}.


\section{Reviewer assignment algorithm}
\label{sec:algorithm}

In this section we first describe our \algo algorithm followed by an illustrative example.


\subsection{Algorithm}

A high level idea of the algorithm is the following. For every integer  $\tmpcapacity \in [\paperload]$, we try to assign each paper to $\tmpcapacity$ reviewers with maximum possible similarities while respecting constraints on reviewer loads. We do so via a carefully designed ``subroutine'' that is explained below. Continuing for that value of $\tmpcapacity$, we complement this assignment with $(\paperload - \tmpcapacity)$ additional reviewers for each paper. Repeating the procedure for each value of $\tmpcapacity \in [\paperload]$, we obtain $\paperload$ candidate assignments each with $\paperload$ reviewers assigned to each paper, and then choose the one with the highest fairness.
The assignment at this point ensures guarantees of worst-case fairness~\eqref{eqn:gen_fair_criteria}. We then also optimize for the second worst-off paper, then the third worst-off paper and so on in the following manner. In the assignment at this point, we find the most disadvantaged papers and permanently fix corresponding reviewers to these papers. Next, we repeat the procedure described above to find the most fair assignment among the remaining papers, and so on. By doing so, we ensure that our final assignment is not susceptible to bottlenecks which may be caused by irrelevant papers with small average similarities.   

The higher-level idea behind the aforementioned subroutine to obtain the candidate assignment for any value of $\tmpcapacity \in [\paperload]$ is as follows. The subroutine constructs a layered flow network graph with one layer for reviewers and one layer for papers, that captures the similarities and the constraints on the paper/reviewer loads.  Then the subroutine incrementally adds edges between (reviewer, paper) pairs in decreasing order of similarity and stops when the paper load constraints are met (each paper can be assigned to $\tmpcapacity$ reviewers using only edges added at this point). This iterative procedure ensures that the papers are assigned reviewers with approximately the highest possible similarities.

We formally present our main algorithm as Algorithm~\ref{alg:fair_assignment} and the subroutine as Subroutine~\ref{alg:fair_subroutine}. In what follows, we walk the reader through the steps in the subroutine and the algorithm in more detail.
\setcounter{algorithm}{0}
\floatname{algorithm}{Subroutine}

\begin{algorithm}[tb]
   \caption{\algo Subroutine}
   \label{alg:fair_subroutine}
   {\bfseries Input:} 
          $\tmpcapacity \in  [\paperload]$: number of reviewers required per paper \\
      \hphantom{{\bfseries Input:}} $\tobeassigned$: set of papers to be assigned \\
 	  \hphantom{{\bfseries Input:}}    $\simmatrix \in \left(\left\{ -\infty \right\} \cup [0,1]\right)^{\numreviewers \times |\tobeassigned|}$: similarity matrix \\
 	  \hphantom{{\bfseries Input:}} $(\maxrevload^{(1)}, \ldots, \maxrevload^{(\numreviewers)}) \in [\maxrevload]^\numreviewers$: reviewers' maximum loads 
 	  
   {\bfseries Output: } Reviewer assignment $\assignment$ \\
   {\bfseries Algorithm:} 
   \begin{enumerate}[noitemsep]
        \item Initialize $\assignment$ to an empty assignment\label{Substep:initA} 
	   	\item Initialize the flow network:\label{Substep:initflow}
	   		\begin{itemize}[noitemsep]
	   			\item {\bf Layer 1:} one vertex (source)
	   			\item {\bf Layer 2:} one vertex for every reviewer $i \in [\numreviewers]$, and directed edges of capacity $\maxrevload^{(i)}$ and cost $0$ from the source to every reviewer 
	   			\item {\bf Layer 3:} one vertex for every paper $j \in \tobeassigned$
	   			\item {\bf Layer 4:} one vertex (sink), and directed edges of capacity $\tmpcapacity$ and cost $0$ from each paper to the sink
	   		\end{itemize}
	   	\item Find (reviewer, paper) pair $(i, j)$ such that the following two conditions are satisfied: \label{Substep:find_pair}
	   			\begin{itemize}
	   				\item the corresponding vertices $i$ and $j$  are not connected in the flow network
	   				\item the similarity $\similarity_{ij}$ is maximal among the pairs which are not connected (ties are broken arbitrarily)
	   			\end{itemize}
	   		and call this pair $(i', j')$
	    \item Add a directed edge of capacity $1$ and cost $\similarity_{i'j'}$ between nodes $i'$ and $j'$  \label{Substep:add_pair}
	    \item Compute the max-flow from source to sink, if the size of the flow is strictly smaller than $|\tobeassigned| \tmpcapacity$, then go to Step~\ref{Substep:find_pair} \label{Sstep:test_flow}
		\item If there are multiple possible max-flows, choose any one arbitrarily (or use any heuristic such as max-flow with max cost) \label{Sstep:pick_flow}
		\item For every edge $(i, j)$ between layers 2 (reviewers) and 3 (papers) which carries a unit of flow in the selected max-flow, assign reviewer $i$ to paper $j$ in the assignment $\assignment$ \label{Sstep:final_step}
	   \end{enumerate}

\end{algorithm}
\noindent
{\bf Subroutine.} A key component of our algorithm is a construction of a flow network in a sequential manner in Subroutine~\ref{alg:fair_subroutine}. 
The subroutine takes as input, among other arguments, the set $\tobeassigned$ of papers that are not yet assigned and the required number of reviewers per paper $\tmpcapacity \le \paperload$. The goal of the subroutine is to assign each paper in $\tobeassigned$ with $\tmpcapacity$ reviewers, respecting the reviewer load constraints, in a way that minimum similarity across all paper-reviewer pairs in resulting assignment is maximized.

The output of the subroutine is an assignment (represented by variable $\assignment$) which is initially set as empty (Step~\ref{Substep:initA}). The subroutine begins (Step~\ref{Substep:initflow}) with a construction of a directed acyclic graph (a ``flow network'') comprising 4 layers in the following order: a source, all reviewers, all papers in $\tobeassigned$, and a sink. An edge may exist only between consecutive layers. The edges between the first two layers control the reviewers' workloads and edges between the last two layers represent the number of reviews required by the papers. Finally, costs of the all edges in this initial construction are set to $0$. Note that in subsequent steps, the edges are added only between the second and third layers. Thus, the maximum flow in the network is at most $|\tobeassigned| \tmpcapacity$. 
  
The crux of the subroutine is to incrementally add edges one at a time between the layers, representing the reviewers and papers, in a carefully designed manner (Steps~\ref{Substep:find_pair} and~\ref{Substep:add_pair}). The edges are added in order of decreasing similarities. These edges control a reviewer-paper relationship: they have a unit capacity to ensure that any reviewer can review any paper at most once and their costs are equal to the similarity between the corresponding (reviewer, paper) pair.

After adding each edge, the subroutine (Step~\ref{Sstep:test_flow}) tests whether a max-flow of size $|\tobeassigned| \tmpcapacity$ is feasible. Note that a feasible flow of size $|\tobeassigned| \tmpcapacity$ corresponds to a feasible assignment: by construction of the flow network described earlier, we know that the reviewer and paper load constraints are satisfied. The capacity of each edge in our flow network is a non-negative integer, thereby guaranteeing that the max-flow is an integer, that it can be found in polynomial time, and that the flow in every edge is a non-negative integer under the max-flow. Once the max-flow of size $|\tobeassigned| \tmpcapacity$ is reached, the subroutine stops adding edges. At this point, it is ensured that the value of the lowest similarity in the resulting assignment is maximized. 

Finally, the subroutine assigns each paper to $\tmpcapacity$ reviewers, using only the ``high similarity'' edges added to the network so far (Steps~\ref{Sstep:pick_flow} and~\ref{Sstep:final_step}).  The existence of the corresponding assignment is guaranteed by max-flow in the network being equal to $|\tobeassigned| \tmpcapacity$. There may be more than one feasible assignments that attain the max-flow. While any of these assignments would suffice from the standpoint of optimizing the worst-case fairness objective~\eqref{eqn:gen_fair_criteria}, the PC may wish to make a specific choice for additional benefits and specify the heuristic to pick the max-flow in Step~\ref{Sstep:pick_flow} of the subroutine. For example, if the max-flow with the maximum cost is selected, then the resulting assignment nicely combines fairness with the high average quality of the assignment. Another choice, discussed in Appendix~\ref{appendix:modifications}, helps with broad topic coverage of the assignment.  Importantly, the approximation guarantees established in Theorem~\ref{thm:deterministic} and Corollary~\ref{corr:seq_deterministic}, as well as statistical guarantees from Sections~\ref{sec:obj_score_model} and~\ref{sec:subj_score_model} hold for any max-flow assignment chosen in Steps~\ref{Sstep:pick_flow} and~\ref{Sstep:final_step}.

For comparison, we note that the \algtpms algorithm can equivalently be interpreted in this framework as follows. The \algtpms algorithm would first \emph{connect all reviewers to all papers} in layers 2 and 3 of the flow graph. It will then compute a max-flow with max cost in this fully connected flow network and make reviewer-paper assignments corresponding to the edges with unit flow between layers 2 and 3. In contrast, our sequential construction of the flow graph prevents papers from being assigned to weak reviewers and is crucial towards ensuring the fairness objective.

\setcounter{algorithm}{0}
\floatname{algorithm}{Algorithm}
\begin{algorithm}[tb]
   \caption{\algo Algorithm}
   \label{alg:fair_assignment}
    {\bfseries Input:} $\paperload \in [\numreviewers]$: number of reviewers  required per paper\\
    \hphantom{{\bfseries Input:}} $\simmatrix \in [0,1]^{\numreviewers \times \numpapers}$: similarity matrix\\ 
    \hphantom{{\bfseries Input:}} $\maxrevload \in [\numpapers]$: reviewers' maximum load\\
    \hphantom{{\bfseries Input:}} $\transformation$: transformation of similarities\\
   {\bfseries Output: } Reviewer assignment $\fairassignment_{\transformation}$\\
   {\bfseries Algorithm:} 
   \begin{enumerate}[noitemsep]
    \item Initialize $\vectmaxrevload = ({\maxrevload, \ldots, \maxrevload}) \in [\maxrevload]^\numreviewers$ \\  \hphantom{{Initialize}} $\fairassignment_{\transformation}, \assignment_0: $ empty assignments \\ \hphantom{{Initialize}} $\tobeassigned = [\numpapers]$:  set of papers to be assigned
     \item For $\tmpcapacity = 1$ to $\paperload$ \label{Algostep:loopkappa}
        \begin{enumerate}[noitemsep]
            \item Set $\vectmaxrevloadtmp = \vectmaxrevload, \simmatrixtmp = \simmatrix$
    		\item Assign $\tmpcapacity$ reviewers to every paper using subroutine:  $\assignment_{\tmpcapacity}^1 =$ Subroutine$(\tmpcapacity, \tobeassigned, \simmatrixtmp, \vectmaxrevloadtmp)$   		\label{Algostep:subroutine_call}
    		\item Decrease $\vectmaxrevloadtmp$ for every reviewer by the number of papers she/he is assigned in $\assignment_{\tmpcapacity}^1$, set corresponding similarities in $\simmatrixtmp$ to $-\infty$
    		\label{Algostep:adjust_loads}
    		
    		\item Run subroutine with adjusted $\vectmaxrevloadtmp$ and $\simmatrixtmp$ to assign remaining $\paperload - \tmpcapacity$ reviewers to every paper:
    		$\assignment_{\tmpcapacity}^2 = $  Subroutine$(\paperload - \tmpcapacity, \tobeassigned, \simmatrixtmp, \vectmaxrevloadtmp)$ \label{Algostep:subroutine_call2}
    		\item Create assignment $\assignment_{\tmpcapacity}$ such that for every pair $(i, j)$ of reviewer $i \in [\numreviewers]$ and paper $j \in \tobeassigned$, reviewer $i$ is assigned to paper $j$ if she/he is assigned to this paper in either $\assignment_{\tmpcapacity}^1$ or $\assignment_{\tmpcapacity}^2$ \label{Algostep:join}
    	\end{enumerate}
    \item Choose $\tmpassignment \in \argmax\limits_{\tmpcapacity \in [\paperload]\cup\{ 0 \} } \assignmentquality[\transformation]{\assignment_{\tmpcapacity}}$ with ties broken arbitrarily \label{Algostep:select_candidate}
   	\item For every paper $j \in \worstoff \defn \argmin\limits_{\ell \in \tobeassigned} \sum\limits_{i \in \reviewerset{\tmpassignment}(\ell)} \transformation(\similarity_{i\ell})$, assign all reviewers $\reviewerset{\tmpassignment}(j)$ to paper $j$ in $\fairassignment_{\transformation}$  \label{Algostep:assignweakest}
   	\item For every reviewer $i \in [\numreviewers]$, decrease $\maxrevload^{(i)}$ by the number of papers in $\worstoff$ assigned to $i$ \label{Algostep:adjust_abilities}
   	\item Delete columns corresponding to the papers $\worstoff$ from $\simmatrix$ and $\tmpassignment$, update $\tobeassigned = \tobeassigned \backslash \worstoff$ \label{Algostep:adjust_sets}
   	\item Set $\assignment_0 = \tmpassignment$ \label{Algostep:enditer}
   	\item If $\tobeassigned$ is not empty, go to Step~\ref{Algostep:loopkappa}
   	\end{enumerate}
\end{algorithm}

\smallskip
\noindent
{\bf Algorithm.}
The algorithm calls the subroutine iteratively and uses the outputs of these iterates in a carefully designed manner. Initially, all papers belong to a set $\tobeassigned$ which represents papers that are not yet assigned. The algorithm repeats Steps~\ref{Algostep:loopkappa} to~\ref{Algostep:enditer} until all papers are assigned. In every iteration, for every value of $\tmpcapacity \in [\paperload]$, the algorithm first calls the subroutine to assign $\tmpcapacity$ reviewers to each paper from $\tobeassigned$ (Step~\ref{Algostep:subroutine_call}), and then adjusts reviewers' capacities and the similarity matrix (Step~\ref{Algostep:adjust_loads}) to prevent any reviewer being assigned to the same paper twice. Next, the subroutine is called again (Step~\ref{Algostep:subroutine_call2}) to assign another $(\paperload - \tmpcapacity)$ reviewers to each paper. As a result, after completion of Step~\ref{Algostep:loopkappa}, $\paperload$ feasible candidate assignments $\assignment_1, \ldots, \assignment_{\paperload}$ are constructed. Each assignment $\assignment_{\tmpcapacity}, \tmpcapacity \in [\paperload]$, 
	is guaranteed (through the Step~\ref{Algostep:subroutine_call}) to maximize the minimum similarity across pairs $(i, j)$ where $j \in \tobeassigned$ and reviewer $i$ is among $\tmpcapacity$ strongest reviewers assigned to paper $j$ in $\assignment_{\tmpcapacity}$; and  (through the Steps~\ref{Algostep:subroutine_call2} and~\ref{Algostep:join}) to have each paper assigned with exactly $\paperload$ reviewers.

In Step~\ref{Algostep:select_candidate}, the algorithm chooses the assignment with the highest fairness~\eqref{eqn:gen_fair_criteria} among the $\paperload$ candidate assignments and the assignment $\assignment_0$ from the previous iteration (empty in the first iteration). Note that since $\assignment_0$ is also included in the maximizer, the fairness cannot decrease in subsequent iterations.

In the chosen assignment, the algorithm identifies the papers that are most disadvantaged, and fixes the assignment for these papers (Step~\ref{Algostep:assignweakest}). The assignment for these papers will not be changed in any subsequent step. The next steps (Steps~\ref{Algostep:adjust_abilities} and~\ref{Algostep:adjust_sets}) update the auxiliary variables to account for this assignment that is fixed --- decreasing the corresponding reviewer capacities and removing these assigned papers from the set $\tobeassigned$. Step~\ref{Algostep:enditer} then keeps a track of the present assignment $\tmpassignment$ for use in subsequent iterations, ensuring that fairness cannot decrease as the algorithm proceeds.

\begin{remarks*}

We make a few additional remarks regarding the \algo algorithm.

\emph{1. Computational cost:} A  na\"ive implementation of the \algo algorithm has a computational complexity $\widetilde{\mathcal{O}}\left(\paperload (\numpapers + \numreviewers)\numpapers^2 \numreviewers \right)$. We give more details on implementation and computational aspects in Appendix~\ref{appendix:computational_aspects}. 
 
\emph{2. Variable reviewer or paper loads:} More generally, the \algo algorithm allows for specifying different loads for different reviewers and/or papers. For general paper loads, we consider $\tmpcapacity \le \max_{j \in [\numpapers]} \paperload^{(j)}$ and define the capacity of edge between node corresponding to any paper $j$ and sink as $\min\{\tmpcapacity, \paperload^{(j)} \}$.

\emph{3. Incorporating conflicts of interest:} One can easily incorporate any conflict of interest between any reviewer and paper by setting the corresponding similarity to $- \infty$.  

\emph{4. Topic coverage:} The techniques developed in~\cite{Long13gooadandfair} can be employed to modify our algorithm in a way that it first ensures fairness and then, among all approximately fair assignments, picks one that approximately maximizes the number of distinct topics of papers covered. We discuss this modification in Appendix~\ref{appendix:modifications}.

\end{remarks*}


\subsection{Example}
\label{sec:example}

To provide additional intuition behind the design of the algorithm, we now present an  example that we also use in the next section to explain our approximation guarantees. 
 
Let for a moment assume that $\transformation(\similarity) = \similarity$ and let $\minsim$ be a constant close to $1$. Consider the following two scenarios:
{
\renewcommand{\theenumi}{(S\arabic{enumi})}
\renewcommand{\labelenumi}{\theenumi}
\begin{enumerate}[noitemsep]
	\item The optimal assignment $\hardassignment$ is such that all the papers are assigned to reviewers with high similarity:
		\begin{align}
			\label{eqn:easy_scenario}
		 \min\limits_{i \in \reviewerset{\hardassignment}(j)} \similarity_{ij}> \minsim \qquad 	\forall j \in [\numpapers].	
		\end{align} \label{case:ideal}

	\item The optimal assignment $\hardassignment$ is such that there are some  ``critical'' papers which have $\smallnumrev < \paperload$ assigned reviewers with similarities higher than $\minsim$ and the remaining assigned reviewers with small similarities. All other papers are assigned to $\paperload$ reviewers with similarity higher than $\minsim$. \label{case:hard}
\end{enumerate}
}
Intuitively, the first scenario corresponds to an ideal situation since there exists an assignment such that each paper has $\paperload$ competent reviewers (with similarity $\minsim \approx 1$). In contrast, in the second scenario, even in the fair assignment, some papers lack expert reviewers. Such a scenario may occur, for example, if some non-mainstream papers were submitted to a conference. This case entails identifying and treating these disadvantaged papers as well as possible. To be able to find the fair assignment in both scenarios, the assignment algorithm should distinguish between them and adapt its behavior to the structure of similarity matrix. Let us track the inner-workings of \algo algorithm to demonstrate this behaviour.

We note that by construction, the fairness of the resulting assignment $\fairassignment$ is determined in the first iteration of Steps~\ref{Algostep:loopkappa} to~\ref{Algostep:enditer} of Algorithm~\ref{alg:fair_assignment}, so we restrict our attention to $\tobeassigned = [\numpapers]$.  First, consider scenario~\ref{case:ideal}. The subroutine called with parameter $\tmpcapacity = \paperload$ will add edges to the flow network until the maximal flow of size $\numpapers \paperload$ is reached. Since the optimal assignment $\hardassignment$ is such that the lowest similarity is higher than $\minsim$, the last edge added to the flow network will have similarity at least $\minsim$, implying that the fairness of the candidate assignment $\assignment_{\paperload}$, which is a lower bound for the fairness of resulting assignment, will be at least $\paperload \minsim$. Given that $\minsim$ is close to one, we conclude that in this case algorithm is able to recover an assignment which is at least very close to optimal.

Now, let us consider scenario~\ref{case:hard}. In this scenario, the subroutine called with $\tmpcapacity = \paperload$ may return a poor assignment. Indeed, since there is a lack of competent reviewers  for critical papers,  there is no way to assign each paper with $\paperload$ reviewers having a high minimum similarity in the assignment. However, the subroutine called with parameter $\tmpcapacity = \smallnumrev$ will find $\smallnumrev$ strong reviewers for each paper (including the critical papers), thereby leading to a fairness  $\assignmentquality{\fairassignment} \ge \smallnumrev \minsim$. The obtained lower bound guarantees that the assignment recovered by the \algo algorithm is also close to the optimal, because in the fair assignment $\hardassignment$ some papers have only $\smallnumrev$ strong reviewers.

This example thus illustrates how the \algo algorithm can adapt to the structure of the similarity matrix in order to guarantee fairness, as well as other guarantees that are discussed subsequently in the paper.


\section{Approximation guarantees}
\label{sec:algorithm_analysis}

In this section we provide guarantees on the fairness of the reviewer-assignment by our algorithm. We first establish guarantees on the max-min fairness objective introduced earlier (Section~\ref{sec:maxmin_approx}). We subsequently show that our algorithm optimizes not only the worst-off paper but recursively optimizes all papers (Section~\ref{sec:allpapers_approx}). We then conclude this section on deterministic approximation guarantees with a comparison to past literature (Section~\ref{sec:fairness_literature_compare}).

\subsection{Max-min fairness} 
\label{sec:maxmin_approx}

We begin with some notation that will help state our main approximation guarantees. For each value of $\tmpcapacity \in [\paperload]$, consider the reviewer-assignment problem but where each paper requires $\tmpcapacity$ (instead of $\paperload$) reviews (each reviewer still can review up to $\maxrevload$ papers). Let us denote the family of all feasible assignments for this problem as $\assignmentfamily_{\tmpcapacity}$. Now define the quantities
\begin{align}
\label{eqn:criticalsim}
		\criticalsim_{\tmpcapacity} &\defn \max\limits_{\assignment \in \assignmentfamily_{\tmpcapacity}} \min\limits_{j \in [\numpapers]} \min\limits_{i \in \reviewerset{\assignment}(j)} \similarity_{ij}, \\
	\criticalsim_0 &\defn \max\limits_{i \in [\numreviewers]}\max\limits_{j \in [\numpapers]} \similarity_{ij}, \quad \text{and} \nonumber
	 \\
	\criticalsim_{\infty} &\defn \min\limits_{i \in [\numreviewers]}\min\limits_{j \in [\numpapers]} \similarity_{ij}. \nonumber
\end{align}
Intuitively, for \emph{every} assignment from the family $\assignmentfamily_{\tmpcapacity}$, the quantity $\criticalsim_{\tmpcapacity}$ upper bounds the minimum similarity for any assigned (reviewer, paper) pair. It also means that the value $\criticalsim_{\tmpcapacity}$ is achievable by some assignment in $\assignmentfamily_{\tmpcapacity}$. The value $\criticalsim_0$ captures the value of the largest entry in the similarity matrix $\simmatrix$ and gives a trivial upper bound $\assignmentquality[\transformation]{\assignment} \le \paperload \transformation(\criticalsim_0)$ for every feasible assignment $\assignment \in \assignmentfamily $. Likewise, the value $\criticalsim_{\infty}$ captures the smallest entry in the similarity matrix $\simmatrix$ and yields a lower bound $\assignmentquality[\transformation]{\assignment} \ge \paperload \transformation(\criticalsim_{\infty})$ for every feasible assignment $\assignment \in \assignmentfamily $.

We are now ready to present the main result on the approximation guarantees for the \algo algorithm as compared to the optimal assignment $\hardassignment$. 
\begin{theorem}
\label{thm:deterministic}
	Consider any feasible values of $(\numreviewers, \numpapers, \paperload, \maxrevload)$, any monotonically increasing function $\transformation: [0, 1] \to [0, \infty]$, and any similarity matrix $\simmatrix$. The assignment $\fairassignment_{\transformation}$ given by the \algo algorithm guarantees the following lower bound on the fairness objective~\eqref{eqn:gen_fair_criteria}:
	\begin{subequations}
	\label{eqn:deterministic}
	\begin{align}
		\frac{\assignmentquality[\transformation]{\fairassignment_{\transformation}}}{\assignmentquality[\transformation]{\hardassignment_{\transformation}}} &\ge \frac{\max\limits_{\tmpcapacity \in [\paperload]} \left( \tmpcapacity \transformation(\criticalsim_{\tmpcapacity}) + (\paperload - \tmpcapacity) \transformation(\criticalsim_{\infty})\right)}{\min\limits_{\tmpcapacity \in [\paperload]}  \left(  (\tmpcapacity - 1) \transformation(\criticalsim_0) + \left(\paperload - \tmpcapacity + 1 \right) \transformation(\criticalsim_{\tmpcapacity}) \right)} \label{eqn:deterministic_complicated}\\ 
		& \ge {1}/{\paperload}.\label{eqn:deterministic_simplified}
	\end{align}
	\end{subequations}
\end{theorem} 
\begin{remarks*}
	The numerator of~\eqref{eqn:deterministic_complicated} is a lower bound on the fairness of the assignment returned by our algorithm. It is important to note that if $\paperload = 1$, that is, if we only need to assign one reviewer for each paper, then our \algo Algorithm finds exact solution for the problem, recovering the classical results of~\citet{garfinkel71bap} as a special case.   
	
	In practice, the number of reviewers $\paperload$ required per paper is  a small constant (typically set as $3$), and in that case, our algorithm guarantees a constant factor approximation. Note that the fraction in the right hand side of~\eqref{eqn:deterministic_complicated} can become $0/0$ or $\infty/\infty$, and in both cases it should be read as $1$.

\end{remarks*}

The bound~\eqref{eqn:deterministic_complicated} can be significantly tighter than $1/\paperload$, as we illustrate in the following example.
\begin{example*}
Consider two scenarios~\ref{case:ideal} and~\ref{case:hard}  from Section~\ref{sec:example}, and consider $\transformation(\similarity) = \similarity$. One can see that under scenario~\ref{case:ideal}, we have $\criticalsim_{\paperload} \ge \minsim$. Setting $\tmpcapacity = \paperload$ in the numerator and $\tmpcapacity = 1$ in the denominator of the bound~\eqref{eqn:deterministic_complicated}, and recalling that $\minsim \approx 1$,  we obtain:
\begin{align*}
 \frac{\assignmentquality{\fairassignment}}{\assignmentquality{\hardassignment}} \ge \frac{\minsim}{\criticalsim_1} \approx 1,
\end{align*}
where we have also used the fact that $\criticalsim_1 \leq 1$. Let us now consider the second scenario~\ref{case:hard} in the example of Section~\ref{sec:example}. In this scenario, since each paper can be assigned to $\smallnumrev$ strong reviewers with similarity higher than $\minsim$, we have $\criticalsim_{\smallnumrev} = \minsim \approx 1$. We then also have $\criticalsim_{0} \le 1$. Moreover, there are some papers which have only $\smallnumrev$ strong reviewers in optimal assignment $\hardassignment$, and hence we have $\criticalsim_{\smallnumrev + 1} \ll \criticalsim_{0}$. Setting $\tmpcapacity = \smallnumrev$ in the numerator and $\tmpcapacity = \smallnumrev +1$ in the denominator of the bound~\eqref{eqn:deterministic_complicated}, some algebraic simplifications yield the bound
\begin{align*}
		\frac{\assignmentquality{\fairassignment}}{\assignmentquality{\hardassignment}} \ge \frac{\smallnumrev \criticalsim_{\smallnumrev} + (\paperload - \smallnumrev) \criticalsim_{\infty}}{\smallnumrev \criticalsim_{0} + (\paperload - \smallnumrev) \criticalsim_{\smallnumrev + 1}} \ge \frac{\criticalsim_{\smallnumrev}}{\criticalsim_{0}} - \frac{(\paperload - \smallnumrev)}{\smallnumrev}\frac{\criticalsim_{\smallnumrev + 1}}{\criticalsim_{0}} \approx 1.
\end{align*}
\end{example*}

~\\

We now briefly provide more intuition on the bound~\eqref{eqn:deterministic_complicated} by interpreting it in terms of specific steps in the algorithm. Setting $\transformation(\similarity) = \similarity$, let us consider the first iteration of the algorithm. Recalling the definition~\eqref{eqn:criticalsim} of $\criticalsim_{\tmpcapacity}$, the \algo subroutine called with parameter $\tmpcapacity$ on Step~\ref{Algostep:subroutine_call} finds an assignment such that all the similarities are at least $\criticalsim_{\tmpcapacity}$. This guarantee in turn implies that the fairness of the corresponding assignment $\assignment_{{\tmpcapacity}}$ is at least $\tmpcapacity \criticalsim_{\tmpcapacity} + (\paperload - \tmpcapacity) \criticalsim_{\infty}$, thereby giving rise to the numerator of~\eqref{eqn:deterministic_complicated}. The denominator is an upper bound of the fairness of the optimal assignment $\hardassignment$. The expression for any value of $\tmpcapacity$ is obtained by simply appealing to the definition of $\criticalsim_{\tmpcapacity}$ which is defined in terms of the optimal assignment. By definition~\eqref{eqn:criticalsim} of $\criticalsim_{\tmpcapacity}$, for every feasible assignment $\assignment$ exists at least one paper such that at most $\tmpcapacity - 1$ of the assigned reviewers are of similarity larger than $\criticalsim_{\tmpcapacity}$. Thus, the fairness of the optimal assignment is upper-bounded by the sum similarity of the paper that has $\tmpcapacity - 1$ reviewers with similarity $\criticalsim_0$ (the highest possible similarity), and $\paperload - \tmpcapacity + 1$ reviewers with similarity $\criticalsim_{\tmpcapacity}$.

Finally, one may wonder whether optimizing the objective~\eqref{eqn:unfair_criteria} as done by prior works~\citep{charlin13tpms, charlin12framework} can also guarantee fairness. It turns out that this is not the case (see the example in Table~\ref{table:exmple_unfair} for intuition), and optimizing the objective~\eqref{eqn:unfair_criteria} is not a suitable proxy towards the fairness objective~\eqref{eqn:gen_fair_criteria}. In Appendix~\ref{appendix:tpms_failure} we show that in general the fairness objective value of the \algtpms algorithm which optimizes~\eqref{eqn:unfair_criteria} may be \emph{arbitrarily bad} as compared to that attained by our \algo algorithm. 

In Appendix~\ref{appendix:pr4a_failure} we show that the analysis of the approximation factor of our algorithm is tight in a sense that there exists a similarity matrix for which the bound~\eqref{eqn:deterministic_simplified} is met with equality. That said, the approximation factor of our \algo algorithm can be much better than $\frac{1}{\paperload}$ for various other similarity matrices, as demonstrated in examples~\ref{case:ideal} and~\ref{case:hard}.


\subsection{Beyond worst case}
\label{sec:allpapers_approx}

\newcommand{\totaliter}{p}
\newcommand{\currentiter}{r}

The previous section established guarantees for the \algo algorithm on the fairness of the assignment in terms of the worst-off paper. In this section we formally show that the algorithm does more: having the assignment for the worst-off paper fixed, the algorithm then satisfies the second worst-off paper, and so on.

Recall that Algorithm~\ref{alg:fair_assignment} iteratively repeats Steps~\ref{Algostep:loopkappa} to~\ref{Algostep:enditer}. In fact, the first time that Step~\ref{Algostep:select_candidate} is executed, the resulting intermediate assignment $\tmpassignment$ achieves the max-min guarantees of Theorem~\ref{thm:deterministic}. However, the algorithm does not terminate at this point. Instead, it  finds the most disadvantaged papers in the selected assignment and fixes them in the final output $\fairassignment_{\transformation}$ (Step~\ref{Algostep:assignweakest}), attributing these papers to reviewers according to $\tmpassignment$. Then it repeats the entire procedure (Steps~\ref{Algostep:loopkappa} to~\ref{Algostep:enditer}) again to identify and fix the assignment for the most disadvantaged papers among the remaining papers and so on until the all papers are assigned in $\fairassignment_{\transformation}$. We denote the total number of  iterations of Steps~\ref{Algostep:loopkappa} to~\ref{Algostep:enditer} in  Algorithm~\ref{alg:fair_assignment} as $\totaliter~(\le \numpapers)$. For any iteration $\currentiter \in [\totaliter]$, we let $\paperset_{\currentiter}$ be the set of papers which the algorithm,  in this iteration, fixes in the resulting assignment. We also let $\tmpassignment_{\currentiter}, \currentiter \in [\totaliter],$ denote the assignment selected in Step~\ref{Algostep:select_candidate} of the $\currentiter^{\text{th}}$ iteration. Note that eventually all the papers are fixed in the final assignment $\fairassignment_{\transformation}$, and hence we must have $\bigcup\limits_{\currentiter \in [\totaliter]} \paperset_{\currentiter} = [\numpapers]$.

Once papers are fixed in the final output $\fairassignment_{\transformation}$, the assignment for these papers are not changed any more. Thus, at the end of each iteration  $\currentiter \in [\totaliter]$  of Steps~\ref{Algostep:loopkappa} to~\ref{Algostep:enditer}, the algorithm deletes (Step~\ref{Algostep:adjust_sets}) the columns of similarity matrix that correspond to the papers fixed in this iteration. For example, at the end of the first iteration, columns which correspond to $\paperset_1$ are deleted from $\simmatrix$. For each iteration $\currentiter \in [\totaliter]$, we let $\simmatrix_{\currentiter}$ denote the similarity matrix at the beginning of the iteration. Thus, we have $\simmatrix_1 = \simmatrix$, because at the beginning of the first iteration, no papers are fixed in the final assignment $\fairassignment_{\transformation}$.

Moving forward, we are going to show that for every iteration $\currentiter \in [\totaliter]$, the sum similarity of the worst-off papers $\paperset_{\currentiter}$ (which coincides with the fairness of $\tmpassignment_{\currentiter}$) is close to the best possible, given the assignment for the all papers fixed in the previous iterations. As in Theorem~\ref{thm:deterministic}, we will compare the fairness $\assignmentquality[\transformation]{\tmpassignment_{\currentiter}}$ with the fairness of the optimal assignment that \alghard algorithm would return if called at the beginning of the $\currentiter^{\text{th}}$ iteration. We stress that for every $\currentiter \in [\totaliter]$, the \alghard algorithm assigns papers $\bigcup\limits_{l = \currentiter}^{\totaliter} \paperset_l$ and respects the constraints on reviewers' loads,  adjusted for the assignment of papers $\bigcup\limits_{l = 1}^{\currentiter - 1} \paperset_l$ in $\fairassignment_{\transformation}$. We denote the corresponding assignment as $\hardassignment_{\transformation}(\paperset_{\{\currentiter : \totaliter \}})$. Note that $\hardassignment_{\transformation}(\paperset_{\{1 : \totaliter \}}) = \hardassignment_{\transformation}$. The following corollary summarizes the main result of this section:

\begin{corollary}
\label{corr:seq_deterministic}
	For any integer $\currentiter \in [\totaliter]$, the assignment $\tmpassignment_{\currentiter}$, selected by the \algo algorithm in Step~\ref{Algostep:select_candidate} of the $\currentiter^{\text{th}}$ iteration, guarantees the following lower bound on the fairness objective~\eqref{eqn:gen_fair_criteria}:
	\begin{align}
	\label{eqn:seq_deterministic}
			\frac{\assignmentquality[\transformation]{\tmpassignment_{\currentiter}}}{\assignmentquality[\transformation]{\hardassignment_{\transformation}(\paperset_{\{r:p\}})}} \ge \frac{\max\limits_{\tmpcapacity \in [\paperload]} \left( \tmpcapacity \transformation(\criticalsim_{\tmpcapacity}) + (\paperload - \tmpcapacity) \transformation(\criticalsim_{\infty})\right)}{\min\limits_{\tmpcapacity \in [\paperload]}  \left(  (\tmpcapacity - 1) \transformation(\criticalsim_0) + \left(\paperload - \tmpcapacity + 1 \right) \transformation(\criticalsim_{\tmpcapacity}) \right)} \ge {1}/{\paperload},
	\end{align}
    where values $\criticalsim_{\tmpcapacity}, \tmpcapacity \in \{0, \ldots, \paperload \} \cup \{\infty\}$, are defined with respect to the similarity matrix $\simmatrix_r$ and constraints on reviewers' loads adjusted for the assignment of papers $\bigcup\limits_{l = 1}^{\currentiter - 1} \paperset_l$ in $\fairassignment_{\transformation}$.
\end{corollary}
The corollary guarantees that each time the algorithm fixes the assignment for some papers $j \in \tobeassigned$ in $\fairassignment_{\transformation}$, the sum similarity for these papers (which is smallest among papers from $\tobeassigned$) is close to the optimal fairness, where optimal fairness is conditioned on the previously assigned papers. In case $r=1$, the bound~\eqref{eqn:seq_deterministic} coincides with the bound~\eqref{eqn:deterministic} from Theorem~\ref{thm:deterministic}. Hence, once the assignment for the most worst-off papers is fixed, the \algo algorithm adjusts maximum reviewers' loads and looks for the most fair assignnment of the remaining papers.


\subsection{Comparison to past literature}
\label{sec:fairness_literature_compare}

In this section we discuss how the approximation results established in previous sections relate to the past literature.

First, we note that the assignment $\assignment_{1}$, computed in Step~\ref{Algostep:loopkappa} in the first iteration of Steps~\ref{Algostep:loopkappa} to~\ref{Algostep:enditer} of Algorithm~\ref{alg:fair_assignment}, recovers the assignment of~\citet{Hartvigsen99assignment}, thus ensuring that our algorithm is \emph{at least as fair} as theirs. Second, if the goal is to assign only one reviewer ($\paperload = 1$) to each of the papers, then our \algo algorithm finds the optimally fair assignment and recovers the classical result of~\citet{garfinkel71bap}.

In the remainder of this section, we provide a comparison between the guarantees of the \algo algorithm established in Theorem~\ref{thm:deterministic} and the  guarantees of the \alggarg algorithm~\citep{Garg2010papers}. Rewriting the results of~\citet{Garg2010papers} in our notation, we have the bound:
\begin{align}
\label{eqn:garg_bound}
	\frac{\assignmentquality[\transformation]{\gargassignment_{\transformation}}}{\assignmentquality[\transformation]{\hardassignment_{\transformation}}}	 \ge \frac{\assignmentquality[\transformation]{\hardassignment_{\transformation}} - \left(\transformation(\criticalsim_0) - \transformation(\criticalsim_{\infty}) \right)}{\assignmentquality[\transformation]{\hardassignment_{\transformation}}} = 1 - \frac{\transformation(\criticalsim_0) - \transformation(\criticalsim_{\infty})}{\assignmentquality[\transformation]{\hardassignment_{\transformation}}}.
\end{align}
Note that our bound~\eqref{eqn:deterministic} for our \algo algorithm is multiplicative and bound for the \alggarg algorithm is additive which makes them incomparable in a sense that neither one dominates another. However, we stress the following differences. First, if we assume $\transformation$ to be upper-bounded by one, then assignment $\gargassignment$ satisfies the bound
\begin{align}
\label{eqn:garg_bound_simplified}
	\assignmentquality[\transformation]{\gargassignment_{\transformation}} \ge \assignmentquality[\transformation]{\hardassignment_{\transformation}} - 1.	
\end{align}

This bound gives a nice additive approximation factor --- for a large value of the optimal fairness $\assignmentquality[\transformation]{\hardassignment_{\transformation}}$, the constant additive factor is negligible. However, if the optimal fairness is small, which can happen if some papers do not have a sufficient number of high-expertise reviewers, then the lower bound on the fairness of the \alggarg assignment~\eqref{eqn:garg_bound_simplified} becomes negative, making the guarantees vacuous as any arbitrary assignment will achieve a non-negative fairness. Note that this issue is not an artifact of the analysis but is inherent in the \alggarg algorithm itself, as we demonstrate in the example presented in Table~\ref{table:exmple_unfair} and in Appendix~\ref{appendix:garg_discussion}. 
In contrast, our algorithm in the worst case has a multiplicative approximation factor $1/\paperload$ ensuring that it always returns a non-trivial assignment. 

This discrepancy becomes more pronounced if the function $\transformation$ is allowed to be unbounded, and the similarities are significantly heterogeneous. Suppose there is some reviewer $i \in [\numreviewers]$ and paper $j \in [\numpapers]$ such that $\transformation(\similarity_{ij}) \gg \assignmentquality[\transformation]{\hardassignment}$. 
Then the bound~\eqref{eqn:garg_bound} for the \alggarg algorithm again becomes vacuous, while the bound~\eqref{eqn:deterministic} for the \algo algorithm continues to provide a non-trivial approximation guarantee. 

Finally, we note that the bound~\eqref{eqn:garg_bound} is also extended by~\cite{Garg2010papers} to obtain guarantees on the fairness for the second worst-off paper and so on.


\section{Objective-score model}
\label{sec:obj_score_model}

We now turn to establishing statistical guarantees for our \algo algorithm from Section~\ref{sec:algorithm}. We begin by considering an ``objective'' score model which we borrow from past works.



\subsection{Model setup}

The objective-score model assumes that each paper $j \in [\numpapers]$ has a true, unknown quality  $\truepaperquality_j \in \reals$ and each reviewer $i \in [\numreviewers]$ assigned to paper $j$ gives her/his estimate $\scoregiven_{ij}$ of $\truepaperquality_j$. The eventual goal is to estimate top $\sizeofconf$ papers according to true qualities $\truepaperquality_j, j 
\in [\numpapers]$.  
Following the line of works by~\citet{ge13bias, mcglohon10starquality, dai12yelp, sajjadi16peergrading}, we assume the score $\scoregiven_{ij}$ given by any reviewer $i \in [\numreviewers]$ to any paper $j \in [\numpapers]$ to be independently and normally distributed around the true paper qualities:
\begin{align}
\label{eqn:reviewer_model}
	\scoregiven_{ij} \sim \gaussian\left(\truepaperquality_j, \std_{ij}^2 \right).	
\end{align}

Note that~\citet{mcglohon10starquality, dai12yelp} and~\citet{sajjadi16peergrading} consider the restricted setting with $\std_{ij} = \std_i$ for all $(i, j) \in [\numreviewers] \times [\numpapers]$, which implies that the variance of the reviewers' scores depends only on the reviewer, but not on the paper reviewed. We claim that this assumption is not appropriate for our peer-review problem: conferences today (such as ICML and NeurIPS) cover a wide spectrum of research areas and it is not reasonable to expect the reviewer to be equally competent in all of the areas.

In our analysis, we assume that the noise variances are some function of the underlying computed similarities.\footnote{Recall that the similarities can capture not only affinity in research areas but may also incorporate the bids or preferences of reviewers, past history of review quality, etc.} We assume that for any $i \in [\numreviewers]$ and $j \in [\numpapers]$, the noise variance 
\begin{align*} 
\std^2_{ij} = \genvariance(\similarity_{ij}),
\end{align*} 
for some  monotonically decreasing function $\genvariance:[0,1] \rightarrow [0, \infty)$. We assume that this function $\genvariance$ is known; this assumption is reasonable as the function can, in principle, be learned from the data from the past conferences.

We note that the model~\eqref{eqn:reviewer_model} does not consider reviewers' biases. However, some reviewers might be more stringent while others are more lenient. This difference results in score of any reviewer $i$ for any paper $j$ being centered not at $\truepaperquality_j$, but at $(\truepaperquality_j + \bias_i)$. A common approach to reduce biases in reviewers' scores is a post-processing. For example, \citet{ge13bias} compared different statistical models of reviewers in attempt to calibrate the biases; the techniques developed in that work may be extended  to the reviewer model~\eqref{eqn:reviewer_model}. Thus, we leave that bias term out for simplicity.


\subsection{Estimator}
  
Given a valid assignment $\assignment \in \assignmentfamily$, the goal of an estimator is to recover the top $\sizeofconf$ papers. A natural way to do so is to compute the estimates of true paper scores $\truepaperquality_j$ and return top $\sizeofconf$ papers with respect to these estimated scores. The described estimation procedure is a significantly simplified version of what is happening in the real-world conferences. Nevertheless, this fully-automated procedure may serve as a guideline for area chairs, providing a first-order estimate of the total ranking of submitted papers. In what follows, we refer to any estimator as $\estimator$ and to the estimated score of any paper $j$ as $\estpaperquality_j$. Specifically, we consider the following two estimators: 

\begin{itemize}[]
	\item {Maximum likelihood estimator (MLE)} $\mle$ 
		\begin{align}
		\label{eqn:mle_score}
			\estpaperqualitymle_j = \frac{1}{\slim_{i \in \reviewerset{\assignment}(j)} \frac{1}{\std_{ij}^2}}	\slim_{i \in \reviewerset{\assignment}(j)} \frac{\scoregiven_{ij}}{\std_{ij}^2} \sim \gaussian\left( \truepaperquality_j, \frac{1}{\slim_{i \in \reviewerset{\assignment}(j)} \frac{1}{\std_{ij}^2}} \right).
		\end{align}
	Under the model~\eqref{eqn:reviewer_model}, $\estpaperqualitymle_j$ is known to have minimal variance across all linear unbiased estimations. The choice of $\mle$ follows a paradigm that more experienced reviewers should have higher weight in decision making.
	
	\item {Mean score estimator (MEAN)} $\avgest$
		\begin{align}
			\label{eqn:avg_score}
			\estpaperqualityavg_j = \frac{1}{\paperload} \slim_{i \in \reviewerset{\assignment}(j)} \scoregiven_{ij} \sim \gaussian\left(\truepaperquality_j, \frac{1}{\paperload^2} \slim_{i \in \reviewerset{\assignment}(j)} \std_{ij}^2  \right)	.
		\end{align}
		The mean score estimator is convenient in practice because it is not tied to the assumed statistical model, and in the past has been found to be predictive of final acceptance decisions in peer-review settings such as National Science Foundation grant proposals~\citep{cole1981chance} and  homework grading~\citep{sajjadi16peergrading}. This observation is supported by the program chair of ICML 2012 John Langford, who notices in his blog~\citep{LangfordBlog} that in ICML 2012 the decisions on the acceptance were ``surprisingly uniform as a function of average score in reviews''. 
\end{itemize}

\subsection{Analysis}

Here we present statistical guarantees for both $\mle$ and $\avgest$ estimators and for both exact top $\sizeofconf$ recovery and recovery under a  Hamming error tolerance.

\subsubsection{Exact top $\sizeofconf$ recovery}

Let us use $(\sizeofconf)$ and $(\sizeofconf + 1)$ to denote the indices of the papers that are respectively ranked $\sizeofconf^{\text{th}}$ and $\left(\sizeofconf + 1\right)^{\text{th}}$ according to their true qualities. Similar to the past work by~\citet{shah15pairwise} on top $\sizeofconf$ item recovery, a central quantity in our analysis is a $\sizeofconf${-\it separation threshold} $\threshold$ defined as:
\begin{align}
\label{eqn:sep_threshold}
	\threshold \defn \truepaperquality_{(\sizeofconf)} - \truepaperquality_{(\sizeofconf + 1)} > 0.	
\end{align}
Intuitively, if the difference between $\sizeofconf^{\text{th}}$ and $\left(\sizeofconf + 1 \right)^{\text{th}}$ papers is large enough, it should be easy to recover top $\sizeofconf$ papers. To formalize this intuition, for any value of a parameter $\distance \geq 0$, consider a family $\problemfamily$ of papers' scores
\begin{align}
\label{eqn:score_family}
	\problemfamily(\distance) \defn \left\{ \left(\paperquality_1, \ldots, \paperquality_{\numpapers}\right) \in \reals^{\numpapers}  \Big |  \paperquality_{(\sizeofconf)} - \paperquality_{(\sizeofconf + 1)} \ge \distance \right\}. 	
\end{align}

For the first half of this section, we assume that function $\genvariance$ is bounded, that is, $\genvariance: [0, 1] \to [0, 1]$.\footnote{More generally, we could consider bounded function $\genvariance$ with range $[0, \const]$ for some $\const > 0$. Without loss of generality, we set $\const = 1$ which can always be achieved by appropriate scaling. }  This assumption implicitly assumes that every reviewer $i \in [\numreviewers]$ can provide a minimum level of expertise while reviewing any paper $j \in [\numpapers]$ even if she/he has zero similarity $\similarity_{ij} = 0$ with that paper.  

In addition to the gap $\threshold$, the hardness of the problem also depends on the similarities between reviewers and papers. For instance, if all reviewers have near-zero similarity with all the papers, then recovery is impossible unless the gap is extremely large. In order to quantify the tractability of the problem in terms of the similarities
we introduce the following set $\goodavgassignment$ of families of similarity matrices parameterized by a non-negative value $\highqualitycrowd$:
\begin{align}
\label{eqn:good_assignment}
    \goodavgassignment(\highqualitycrowd)  \defn  \left\{\simmatrix \in [0, 1]^{\numreviewers \times \numpapers} \Big |\assignmentquality[1 - \genvariance]{\hardassignment_{1 - \genvariance}} \ge \highqualitycrowd  \right\}.
\end{align}

In words, if similarity matrix $S$ belongs to $\goodavgassignment(\highqualitycrowd)$, then the fairness of the optimally fair (with respect to $\transformation = \transfavg$) assignment is at least $\highqualitycrowd$.

Finally, we define a quantity $\approximation_\highqualitycrowd$ that captures the quality of approximation provided by \algodot:
\begin{align}
    \label{eqn:approximationDefn}
    \approximation_\highqualitycrowd \defn \inf_{\simmatrix \in  \goodavgassignment(\highqualitycrowd)} \frac{\assignmentquality[\transfavg]{\fairassignment_{\transfavg}}}{\assignmentquality[\transfavg]{\hardassignment_{\transfavg}}}.
\end{align}
Note that Theorem~\ref{thm:deterministic} gives
lower bounds on the value of $\approximation_\highqualitycrowd$.

Having defined all the necessary notation, we are ready to present the first result of this section on recovering the set of top $\sizeofconf$ papers $\truebest$. 
\begin{theorem}
	\label{thm:main_theorem}
 (a)	For any $\errorrate \in (0, 1/4)$, $\highqualitycrowd \in [\paperload\left(1 - \genvariance(0)\right), \paperload]$ and any monotonically decreasing $\genvariance: [0, 1] \to [0, 1]$, if $\distance > \frac{2 \sqrt{2}}{\paperload}\sqrt{\left(\paperload - \highqualitycrowd \approximation_\highqualitycrowd  \right) \ln\frac{m}{\sqrt{\errorrate}}}$, then for $\left( \assignment, \estimator \right) \in \left\{\left( \fairassignment_{\transfavg}, \avgest \right), \left( \fairassignment_{\transfmle}, \mle \right) \right\}$
	\begin{align}
		\label{eqn:upped_bound}
		\sup\limits_{\substack{ \left(\truepaperquality_1, \ldots, \truepaperquality_{\numpapers} \right) \in \problemfamily(\distance) \\ \simmatrix \in \goodavgassignment(\highqualitycrowd) }} \prob{\accepted_{\sizeofconf}\left( \assignment, \estimator \right) \ne \truebest}  \le \errorrate.	
	\end{align}
(b)	Conversely, for any continuous strictly monotonically decreasing $\genvariance: [0, 1] \to [0, 1]$ and any $\highqualitycrowd \in [ \paperload\left(1 - \genvariance(0)\right), \paperload]$, there exists a universal constant $\const > 0$ such that if $\numpapers > 6$ and $\distance < \frac{c}{\paperload} \sqrt{\left(\paperload - \highqualitycrowd \right) \ln m}$,  then 
	\begin{align*}
		\sup\limits_{\simmatrix \in \goodavgassignment(\highqualitycrowd)}  \inf\limits_{\left(\estimator, \assignment \in \assignmentfamily \right)} \sup\limits_{ \left(\truepaperquality_1, \ldots, \truepaperquality_{\numpapers} \right) \in \problemfamily(\distance) } \prob{\accepted_{\sizeofconf}\left( \assignment, \estimator \right) \ne \truebest}  \ge \frac{1}{2}.		
	\end{align*}
\end{theorem}

\begin{remarks*}

1. The \algo assignment algorithm thus leads to a strong minimax guarantee on the recovery of the top $k$ papers: the upper and lower bounds differ by at most a $\approximation_{\highqualitycrowd} \geq \frac{1}{\paperload}$ term in the requirement on $\distance$ and constant pre-factor. Also note that as discussed in Section~\ref{sec:maxmin_approx}, approximation factor $\approximation_{\highqualitycrowd}$ of the \algo algorithm can be much better than $1/\paperload$ for various similarity matrices.

2. In addition to quantifying the performance of \algodot, an important contribution of Theorem~\ref{thm:main_theorem} is a sharp minimax analysis of the performance of \emph{every} assignment algorithm. Indeed, the approximation ratio $\approximation_\highqualitycrowd$~\eqref{eqn:approximationDefn} can be defined for any assignment algorithm, by substituting corresponding assignment instead of $\fairassignment_{\transfavg}$. For example, if one has access to the optimal assignment $\hardassignment$ (e.g., by using \algo if $\paperload=1$) then we will have corresponding approximation ratio $\approximation_\highqualitycrowd = 1$ thereby yielding bounds that are sharp up to constant pre-factors. 

3. While on one hand the estimator $\mle$ is preferred over $\avgest$ when model~\eqref{eqn:reviewer_model} is correct, on the other hand, if  $\genvariance(\similarity) \in [0, 1]$, then the estimator $\avgest$ is more robust to model mismatches. 

4. The technical assumption $\highqualitycrowd \in [\paperload\left(1 - \genvariance(0)\right), \paperload]$ is made without loss of any generality, because values of $\highqualitycrowd$ outside this range are vacuous. In more detail, for any similarity matrix $\simmatrix \in [0, 1]^{\numreviewers \times \numpapers}$, it must be that $\assignmentquality[\transfavg]{\hardassignment_{\transfavg}} \ge \paperload \left(1 - \genvariance(0) \right)$. Moreover, the co-domain of function $\genvariance$ comprises only non-negative real values, implying that $\assignmentquality[\transfavg]{\hardassignment_{\transfavg}} \le \paperload$ for any similarity matrix $\simmatrix \in  [0, 1]^{\numreviewers \times \numpapers}$.

5. The upper bound of the theorem holds for a slightly more general model of reviewers --- reviewers with sub-Gaussian noise. Formally, in addition to the Gaussian noise model~\eqref{eqn:reviewer_model}, the proof of Theorem~\ref{thm:main_theorem}(a) also holds for the following class of distributions of the score $\scoregiven_{ij}$:
\begin{align}
\label{eqn:sg_model_of_rev}
	\scoregiven_{ij} = \truepaperquality_{ij} + sG\left(\genvariance(\similarity_{ij})\right),	
\end{align}
where $sG\left(\std^2 \right)$ is an arbitrary mean zero sub-Gaussian random variable with scale parameter $\std^2$.

\end{remarks*}

The conditions of Theorem~\ref{thm:main_theorem} require function $\genvariance$ to be bounded. We now relax our earlier boundedness assumption on $\genvariance$ and consider $\genvariance: [0, 1] \to [0, \infty)$.

In what follows we restrict our attention to MLE estimator $\mle$ which represents the paradigm that reviewers with higher similarity should have more weight in the final decision. In order to demonstrate that our \algo algorithm is able to adapt to different structures of similarity matrices --- from hard cases when optimal assignment provides only one strong reviewer for some of the papers, to ideal cases when there are $\paperload$ strong reviewers for every paper --- let us consider the following set $\onegoodforeveryone$ of families of similarity matrices  parametrized by a non-negative value $\singlehighqual$ and integer parameter $\tmpcapacity \in [\paperload]$:
\begin{align}
\label{eqn:one_good_for_everyone}
	\onegoodforeveryone(\singlehighqual) \defn \left\{\simmatrix \in [0, 1]^{\numreviewers \times \numpapers} \Big | \criticalsim_{\tmpcapacity} \ge \singlehighqual \right\}.	
\end{align}
Here $\criticalsim_{\tmpcapacity}$ is as defined in~\eqref{eqn:criticalsim}.

In words, the parameter $\singlehighqual$ defines the notion of strong reviewer while parameter $\tmpcapacity$ denotes the maximum number of strong (with similarity higher than $\singlehighqual$) reviewers that can be assigned to each paper without violating the $(\maxrevload, \paperload)$ conditions.

Then the following adaptive analogue of Theorem~\ref{thm:main_theorem} holds:
\begin{corollary}
\label{thm:main_mle}
	 (a) For any $\errorrate \in (0, 1/4)$, $\singlehighqual \in [0, 1]$, $\tmpcapacity \in [\paperload]$ and any monotonically decreasing $\genvariance: [0, 1] \to [0, \infty)$, if 
$
		 \distance > 2 \sqrt{2}
	 \sqrt{\frac{\genvariance(\singlehighqual) \genvariance(0)}{\tmpcapacity \genvariance(0) + ({\paperload - \tmpcapacity}) \genvariance(\singlehighqual)} \ln\frac{\numpapers}{\sqrt{\errorrate}}}
$,
then
	\begin{align*}
		\sup\limits_{\substack{ (\truepaperquality_1, \ldots, \truepaperquality_{\numpapers} ) \in \problemfamily(\distance) \\ \simmatrix \in \onegoodforeveryone(\singlehighqual) }} \prob{\accepted_{\sizeofconf}( \fairassignment_{\transfmle}, \mle ) \ne \truebest}  \le \errorrate.	
	\end{align*}
(b)	Conversely, for any continuous strictly monotonically decreasing $\genvariance: [0, 1] \to [0, \infty)$, any  $\singlehighqual \in [0, 1]$, and any $\tmpcapacity \in [\paperload]$,  there exists a universal constant $\const > 0$ such that if $\numpapers > 6$ and $\distance \le \const \sqrt{\frac{\genvariance(\singlehighqual) \genvariance(0)}{\tmpcapacity \genvariance(0) + ({\paperload - \tmpcapacity}) \genvariance(\singlehighqual)}  \ln \numpapers}$, then
	\begin{align*}
		\sup\limits_{\simmatrix \in \onegoodforeveryone(\singlehighqual)} \inf\limits_{\left(\estimator, \assignment \in \assignmentfamily \right)} \sup\limits_{{ (\truepaperquality_1, \ldots, \truepaperquality_{\numpapers} ) \in \problemfamily(\distance) }} \prob{\accepted_{\sizeofconf}( \assignment, \estimator ) \ne \truebest}  \ge \frac{1}{2}.		
	\end{align*}
\end{corollary}

\begin{remarks*}
	1. Observe that there is no approximation factor in the upper bound. Thus,  the \algo algorithm together with $\mle$ are simultaneously minimax optimal up to a constant pre-factor in classes of similarity matrices $\onegoodforeveryone(\singlehighqual)$ for all $\tmpcapacity \in [\paperload]$, $\singlehighqual \in [0,  1]$.
	
	2.\hspace{3pt} Corollary~\ref{thm:main_mle}(a)  remains valid for generalized sub-Gaussian model of reviewer~\eqref{eqn:sg_model_of_rev}.
	 
    3. Corollary~\ref{thm:main_mle} together with Theorem~\ref{thm:main_theorem} show that our \algo algorithm produces the assignment $\fairassignment_{\transfmle}$ which is simultaneously minimax (near-)optimal for various classes of similarity matrices.  We thus see that our \algo algorithm is able to adapt to the underlying structure of similarity matrix $\simmatrix$ in order to construct an assignment in which even the most disadvantaged paper gets reviewers with sufficient expertise to estimate the true quality of the paper.
\end{remarks*}


\subsubsection{Approximate recovery under Hamming error}

Although our ultimate goal is to recover set $\truebest$ of top $\sizeofconf$ papers exactly, we note that often scores of boundary papers are close to each other so it may be impossible to distinguish between the $\sizeofconf^{\text{th}}$ and $(\sizeofconf + 1)^{\text{th}}$ papers in the total ranking. Thus, a more realistic goal would be to try to accept papers such that the set of accepted papers is in some sense 	``close'' to the set $\truebest$. In this work we consider the standard notion of Hamming distance~\eqref{eqn:hamming} as a measure of closeness. We are interested in minimizing the quantity:
\begin{align*}
	\prob{\hamming{\accepted_{\sizeofconf} \left(\assignment, \estimator \right)}{\truebest} > 2 \tol}	
\end{align*}
for some user-defined value of $\tol \in [\sizeofconf - 1]$.

Similar to the exact recovery setup, the key role in the analysis is played by generalized separation threshold (compare with equation~\ref{eqn:sep_threshold}):
\begin{align*}
	\thresholdham \defn \truepaperquality_{(\sizeofconf - \tol)} - \truepaperquality_{(\sizeofconf + \tol + 1)},	
\end{align*}
where $(\sizeofconf - \tol)$ and $(\sizeofconf + \tol + 1)$ are indices of papers that take $(\sizeofconf - \tol)^{\text{th}}$ and $(\sizeofconf + \tol + 1)^{\text{th}}$ positions respectively in the underlying total ranking. For any value of $\distance > 0$ we consider the following generalization of the set $\problemfamily(\distance)$ defined in~\eqref{eqn:score_family}:
\begin{align*}
	\problemfamilyham(\distance) \defn \left\{ \left(\paperquality_1, \ldots, \paperquality_{\numpapers}\right) \in \reals^{\numpapers}  \Big | \paperquality_{(\sizeofconf - \tol)} - \paperquality_{(\sizeofconf + \tol + 1)} \ge \distance \right\}.
\end{align*}
Also recall the family of matrices $\goodavgassignment(\highqualitycrowd)$ from~\eqref{eqn:good_assignment} and the approximation factor $\approximation_{\highqualitycrowd}$ from~\eqref{eqn:approximationDefn} for any parameter $\highqualitycrowd$. With this notation in place, we now present the analogue of Theorem~\ref{thm:main_theorem} in case of approximate recovery under the Hamming error.

\begin{theorem}
	\label{thm:main_theorem_ham}
	(a) For any $\errorrate \in (0, 1/4)$, $\highqualitycrowd \in [\paperload\left(1 - \genvariance(0)\right), \paperload]$, $\tol \in [\sizeofconf - 1]$, and any monotonically decreasing $\genvariance : [0, 1] \to [0, 1]$, if $\distance > \frac{2 \sqrt{2}}{\paperload}\sqrt{\left(\paperload - \highqualitycrowd \approximation_\highqualitycrowd  \right) \ln\frac{m}{\sqrt{\errorrate}}}$, then for $\left( \assignment, \estimator \right) \in \left\{\left( \fairassignment_{\transfavg}, \avgest \right), \left( \fairassignment_{\transfmle}, \mle \right) \right\}$
	\begin{align*}
		\sup\limits_{\substack{ \left(\truepaperquality_1, \ldots, \truepaperquality_{\numpapers} \right) \in \problemfamilyham(\distance) \\ \simmatrix \in \goodavgassignment(\highqualitycrowd) }} \prob{\hamming{\accepted_{\sizeofconf}\left( \assignment, \estimator \right)}{\truebest} > 2 \tol}  \le \errorrate.	
	\end{align*}
	(b) Conversely, for any continuous strictly monotonically decreasing $\genvariance: [0, 1] \to [0, 1]$, any $\highqualitycrowd \in [\paperload \left(1 - \genvariance(0)\right), \paperload]$, and any $0 < \tol < \sizeofconf$,  there exists a universal constant $\const > 0$ such that for given constants $\constnua \in (0; 1)$ and $\constnub \in (0, 1)$ if $2 \tol \le \frac{1}{1 + \constnub} \min \left\{\numpapers^{1 - \constnua}, \sizeofconf, \numpapers - \sizeofconf \right\}$  and $\distance \le \frac{c}{\paperload} \sqrt{\left(\paperload - \highqualitycrowd \right) \constnua \constnub \ln m}$, then for $\numpapers$ larger than some $(\constnua, \constnub)$-dependent constant,
	\begin{align*}
		\sup\limits_{\simmatrix \in \goodavgassignment(\highqualitycrowd)}  \inf\limits_{\left(\estimator, \assignment \in \assignmentfamily \right)} \sup\limits_{ \left(\truepaperquality_1, \ldots, \truepaperquality_{\numpapers} \right) \in \problemfamilyham(\distance) } \prob{\hamming{\accepted_{\sizeofconf}\left( \assignment, \estimator \right)}{\truebest} > 2 \tol}  \ge \frac{1}{2}.		
	\end{align*}
\end{theorem}
\begin{remarks*}
This theorem provides a strong minimax characterization of the \algo algorithm for approximate recovery. Note that upper and lower bounds differ by the approximation factor $\approximation_{\highqualitycrowd}$, which is at most $\frac{1}{\paperload}$, and a pre-factor which depends only on the constants $\constnua$ and $\constnub$.
\end{remarks*}

To conclude the section, we state the result for the family $\onegoodforeveryone(\singlehighqual)$ of similarity matrices  defined in~\eqref{eqn:one_good_for_everyone} for any parameter $\singlehighqual$, showing that adaptive behavior of \algo algorithm (Corollary~\ref{thm:main_mle}) also carries over to the Hamming error metric.
\begin{corollary}
\label{corr:hamming_adaptive}
 (a)	\label{thm:main_theorem_ham_mle}
		 For any $\errorrate \in (0, 1/4)$, $\singlehighqual \in [0, 1]$, $\tmpcapacity \in [\paperload]$, $\tol \in [\sizeofconf - 1]$, and any monotonically decreasing $h : [0, 1] \to [0, \infty)$, if $\distance > 2 \sqrt{2}
	 \sqrt{\frac{\genvariance(\singlehighqual) \genvariance(0)}{\tmpcapacity \genvariance(0) + ({\paperload - \tmpcapacity}) \genvariance(\singlehighqual)} \ln\frac{\numpapers}{\sqrt{\errorrate}}}$, then
	\begin{align*}
		\sup\limits_{\substack{ \left(\truepaperquality_1, \ldots, \truepaperquality_{\numpapers} \right) \in \problemfamilyham(\distance) \\ \simmatrix \in \onegoodforeveryone(\singlehighqual) }} \prob{\hamming{\accepted_{\sizeofconf}\left( \fairassignment_{\transfmle}, \mle \right)}{\truebest} > 2 \tol}  \le \errorrate.	
	\end{align*}
(b)	Conversely, for any continuous strictly monotonically decreasing $\genvariance: [0, 1] \to [0, \infty)$, any $\singlehighqual \in [0, 1]$, $\tmpcapacity \in [\paperload]$ and any $\tol \in [\sizeofconf - 1]$, there exists a universal constant $\const > 0$ such that for given constants $\constnua \in (0; 1)$ and $\constnub \in (0, 1)$ if $2 \tol \le \frac{1}{1 + \constnub} \min \left\{\numpapers^{1 - \constnua}, \sizeofconf, \numpapers - \sizeofconf \right\}$ and $\distance \le \const \sqrt{\frac{\genvariance(\singlehighqual) \genvariance(0)}{\tmpcapacity \genvariance(0) + ({\paperload - \tmpcapacity}) \genvariance(\singlehighqual)} \constnua \constnub \ln{\numpapers}}$, then for $\numpapers$ larger than some $(\constnua, \constnub)$-dependent constant,

	\begin{align*}
		\sup\limits_{\simmatrix \in \onegoodforeveryone (\singlehighqual)}  \inf\limits_{\left(\estimator, \assignment \in \assignmentfamily \right)} \sup\limits_{ \left(\truepaperquality_1, \ldots, \truepaperquality_{\numpapers} \right) \in \problemfamilyham(\distance) } \prob{\hamming{\accepted_{\sizeofconf}\left( \assignment, \estimator \right)}{\truebest} > 2 \tol}  \ge \frac{1}{2}.		
	\end{align*}
\end{corollary}

The results established in this section thus show that our \algo algorithm produces an assignment which is minimax (near-)optimal for both exact and approximate recovery of the top $\sizeofconf$ papers.


\section{Subjective-score model}
\label{sec:subj_score_model}
In the previous section, we analyzed the performance of our \algo assignment algorithm under a model with objective scores. Indeed, various past works on peer-review (as well as various other domains of machine learning) assume existence of some ``true'' objective scores or ranking of the underlying items (papers). However, in practice, reviewers' opinions on the quality of any paper are typically highly subjective~\citep{kerr1977manuscript,mahoney1977publication, ernst1994reviewer,bakanic1987manuscript,lamont2009professors}. 
Even two highly experienced researchers with vast experience and expertise may have considerably differing opinions about the contributions of a paper. Following this intuition, we wish to move away from the assumption of some true objective scores $\{\truepaperquality_j\}_{j \in [\numpapers]}$ of the paper.

With this motivation, in this section we develop a novel  model to capture such subjective opinions and present a statistical analysis of our assignment algorithm under this subjective-score model. 


\subsection{Model}

The key idea behind our subjective score model is to separate out the subjective part in any reviewer's opinion from the noise inherent in it.
Our model is best described by first considering a hypothetical situation where every reviewer \emph{spends an infinite time and effort on reviewing every paper, gaining a perfect expertise in the field of that paper and a perfect understanding of the paper's content}. We let $\subjectivescore_{ij} \in \reals$ denote the score that this fully competent version of reviewer $i \in [\numreviewers]$ would provide to paper $j \in [\numpapers]$, and denote the matrix of reviewers subjective scores as $\subjscorematrix = \left\{\subjectivescore_{ij} \right\}_{i \in [{\numreviewers}], j \in [\numpapers]}$.  Continuing momentarily in this hypothetical world, when  all the reviewers are fully competent in evaluating all the papers, every feasible reviewer-assignment is of the same quality since there is no noise in the reviewers' scores. Since all reviewers have an equal, full competence, a natural choice of scoring any paper $j \in [\numpapers]$ is to take the mean score provided by the fully competent reviewers who review that paper:
		\begin{align}
		\label{eqn:induced_rank}
			\avgpaperquality_j(\assignment) \defn \frac{1}{\paperload} \slim_{i \in \reviewerset{\assignment}(j)} \subjectivescore_{ij}.
		\end{align}

Let us now exit our hypothetical world and return to reality. In a real conference peer-review setting the reviews will be noisy. Following the previous noise assumptions, we assume that score of any reviewer $i \in [\numreviewers]$ for any paper $j \in [\numpapers]$ that she/he reviews is distributed as 
\begin{align*}
	\scoregiven_{ij} \sim \gaussian( \subjectivescore_{ij}, {\genvariance(\similarity_{ij})}),
\end{align*}
for some known continuous strictly monotonically decreasing function $\genvariance:[0,1] \rightarrow [0, 1]$.
Under this model, the higher the similarity $\similarity_{ij}$, the better the score $\scoregiven_{ij}$ represents the subjective score $\subjectivescore_{ij}$ which reviewer $i \in [\numreviewers]$ would give to paper $j \in [\numpapers]$ if she/he had infinite expertise.

The goal under this model is to assign reviewers to papers such that reviewers are of enough ability to convey their opinions $\subjectivescore_{ij}$ from the hypothetical full-competence world to the real world with scores $\scoregiven_{ij}$. In other words, the goal of the assignment is to ensure the recovery of the top $k$ papers in terms of the mean full-competence subjective scores $\{\avgpaperquality_j\}_{j \in [\numpapers]}$.

\subsection{Analysis}

\newcommand{\trubestsubj}{}

In this section we present statistical guarantees for $\avgest$ in context of subjective-score model.

\subsubsection{Exact top $\sizeofconf$ recovery}

Since the true scores for any reviewer-paper pair are subjective, and since we are interested in mean full-competence subjective scores, a natural choice for estimating $\{\avgpaperquality_j\}$ from the actual provided scores $\{\scoregiven_{ij} \}$ is the averaging estimator $\avgest$ which for every paper $j \in [\numpapers]$ estimates $\avgpaperquality_j$ as $\estpaperqualityavg_j = \frac{1}{\paperload} \slim_{i \in \reviewerset{\assignment}(j)} \scoregiven_{ij}$.  Having defined the model and estimator, we now provide a sharp minimax analysis for the subjective-score model. In order to state our main result, we recall the family of similarity matrices $
    \goodavgassignment(\highqualitycrowd)$ defined earlier in~\eqref{eqn:good_assignment} and the approximation ratio $\approximation_\highqualitycrowd$ defined in~\eqref{eqn:approximationDefn}, both parameterized by some non-negative value $\highqualitycrowd$. 

Note that the notion of the $\sizeofconf$-separation threshold~\eqref{eqn:sep_threshold} does not carry over directly from the objective score model to the subjective score model. The reason is that the ranking now is induced by the assignment and changes as we change the assignment. Consequently, we introduce the following family of papers' scores that are governed by the assignment $\assignment$ and parametrized by a positive real value $\distance$:
\begin{align}
\label{eqn:subj_family_score}
	& \problemfamily(\assignment, \distance) = \left\{\subjscorematrix \in \reals^{\numreviewers \times \numpapers} \Big | \avgpaperquality_{(\sizeofconf)}(\assignment) - \avgpaperquality_{(\sizeofconf + 1)}(\assignment) \ge \distance  \right\}.
\end{align}

Since in this section we consider only mean score estimator $\avgest$, we omit index $\transfavg$ from $\fairassignment_{\transfavg}$, but always imply that assignment $\fairassignment$ is built with respect to the function $\transfavg$. For every feasible assignment $\assignment$, we augment the notation $\truebest$ with $\truebestsubj\left(\assignment, \avgpaperquality(\assignment) \right)$ to highlight that the set of the top $\sizeofconf$ papers is induced by the assignment $\assignment$. Let us now present the main result of this section.

\begin{theorem}
	\label{thm:main_theorem_subj}
 (a)	For any $\errorrate \in (0, 1/4)$, $\highqualitycrowd \in [\paperload\left(1 - \genvariance(0)\right), \paperload]$ and any monotonically decreasing $\genvariance: [0, 1] \to [0, 1]$, if $\distance > \frac{2\sqrt{2}}{\paperload}\sqrt{\left(\paperload -  \highqualitycrowd \approximation_\highqualitycrowd \right) \ln\frac{m}{\sqrt{\errorrate}}}$, then
	\begin{align*}
	 \sup\limits_{\substack{ \subjscorematrix \in \problemfamily(\fairassignment, \distance) \\ \simmatrix \in \goodavgassignment(\highqualitycrowd) }} 	& \!\! \prob{\!\accepted_{\sizeofconf}( \fairassignment, \avgest ) \!\ne \!\truebestsubj\left(\fairassignment, \avgpaperquality(\fairassignment) \right)\!}  \!\le \! \errorrate. 	
	\end{align*}
(b)	Conversely, for any continuous strictly monotonically decreasing $\genvariance: [0, 1] \to [0, 1]$ and any $\highqualitycrowd \in [ \paperload\left(1 - \genvariance(0)\right), \paperload]$, there exists a universal constant $\const > 0$ such that if $\numpapers > 6$ and $\distance < \frac{c}{\paperload} \sqrt{\left(\paperload - \highqualitycrowd \right) \ln m}$, then
	\begin{align*}
		\sup\limits_{\simmatrix \in \goodavgassignment(\highqualitycrowd) } \inf\limits_{\left(\estimator, \assignment \in \assignmentfamily\right)} \sup\limits_{\subjscorematrix \in \problemfamily(\assignment, \distance)} &\!\! \prob{\!\accepted_{\sizeofconf}( \assignment, \estimator ) \! \ne \! \truebestsubj\left(\assignment, \avgpaperquality(\assignment) \right)\!}  \!\ge\! \frac{1}{2}. 		
	\end{align*}
\end{theorem}

We thus see that our assignment algorithm \algo not only leads to the strong guarantees under the objective-score model but simultaneously also under the setting where the opinions of reviewers may be subjective.


\subsubsection{Approximate recovery under Hamming error}

We now present guarantees for approximate recovering under the Hamming error for the \algo algorithm. 
We generalize the family of score matrices~\eqref{eqn:subj_family_score}, for which we consider any integer error tolerance parameter $\tol \in \{0,\ldots, \sizeofconf-1\}$ and any any feasible assignment $\assignment$. Then we define the following family of subjective papers' scores, parameterized by non-negative value $\distance$: 
\begin{align*}
	& \problemfamilyham(\assignment, \distance) = \left\{\subjscorematrix \in \reals^{\numreviewers \times \numpapers} \Big | \avgpaperquality_{(\sizeofconf - \tol )}(\assignment) - \avgpaperquality_{(\sizeofconf + \tol + 1)}(\assignment) \ge \distance  \right\}.
\end{align*}
Observe that the class $\problemfamilyham(\assignment, \distance)$ coincides with the class $\problemfamily(\distance)$ from~\eqref{eqn:subj_family_score} when $\tol = 0$. 
\begin{theorem}
	\label{thm:main_theorem_ham_subj}
	(a) For any $\errorrate \in (0, 1/4)$, $\highqualitycrowd \in [0, \paperload]$, $\tol \in [\sizeofconf - 1]$, and any monotonically decreasing $\genvariance : [0, 1] \to [0, 1]$, if $\distance > \frac{2 \sqrt{2}}{\paperload}\sqrt{\left(\paperload - \highqualitycrowd \approximation_\highqualitycrowd  \right) \ln\frac{m}{\sqrt{\errorrate}}}$, then 
	\begin{align*}
		\sup\limits_{\substack{  \subjscorematrix \in \problemfamilyham(\fairassignment, \distance) \\ \simmatrix \in \goodavgassignment(\highqualitycrowd) }} \prob{\hamming{\accepted_{\sizeofconf}\left( \fairassignment, \estimator \right)}{\truebest\left(\fairassignment, \avgpaperquality(\fairassignment) \right)} > 2 \tol}  \le \errorrate.	
	\end{align*}
	Conversely, for any continuous strictly monotonically decreasing $\genvariance: [0, 1] \to [0, 1]$, any $\highqualitycrowd \in [\paperload \left(1 - \genvariance(0)\right), \paperload]$, and any $0 < \tol < \sizeofconf$,  there exists a universal constant $\const > 0$ such that for given constants $\constnua \in (0, 1)$ and $\constnub \in (0, 1)$ if $2 \tol \le \frac{1}{1 + \constnub} \min \left\{\numpapers^{1 - \constnua}, \sizeofconf, \numpapers - \sizeofconf \right\}$  and $\distance \le \frac{c}{\paperload} \sqrt{\left(\paperload - \highqualitycrowd \right) \constnua \constnub \ln m}$, then for $\numpapers$ larger than some $(\constnua, \constnub)$-dependent constant,
	\begin{align*}
		\sup\limits_{\simmatrix \in \goodavgassignment(\highqualitycrowd)}  \inf\limits_{\left(\estimator, \assignment \in \assignmentfamily \right)} \sup\limits_{ \subjscorematrix \in \problemfamilyham(\assignment, \distance)} \prob{\hamming{\accepted_{\sizeofconf}\left( \assignment, \estimator \right)}{\truebest\left(\assignment, \avgpaperquality(\assignment) \right)} > 2 \tol}  \ge \frac{1}{2}.		
	\end{align*}
\end{theorem}
Similar to Theorem~\ref{thm:main_theorem_subj}, Theorem~\ref{thm:main_theorem_ham_subj} shows that \algo algorithm is minimax optimal up to a constant pre-factor and approximation factor \emph{given} that reviewers' subjective scores $\subjscorematrix$ belong to the class $\problemfamilyham(\assignment, \distance)$.

\section{Experiments}
\label{sec:experiments}
In this section we conduct empirical evaluations of the \algo algorithm and compare it with the \algtpmsdot~\citep{charlin13tpms}, \alggargdot~\citep{Garg2010papers} and \alghard algorithms. Our implementation of the \algo algorithm picks max-flow with maximum cost in Step~\ref{Sstep:pick_flow} of Subroutine~\ref{alg:fair_subroutine}.

Previous work on the conference paper assignment problem~\citep{Garg2010papers, Long13gooadandfair, Karimzadehgan08multiaspect, tang10constraied} conducted evaluations of the proposed algorithms in terms of various objective functions that measure the quality of the assignment. For example,~\citet{Garg2010papers} compared fairness from reviewers' perspective using the number of satisfied bids as a criteria. While these evaluations allow to compare algorithms in terms of particular objective, we note that the main goal of the peer-review system is to accept the best papers. It is not straightforward whether an improvement of some other objective will lead to the improvement of the quality of the paper acceptance process.   

In contrast to the prior works, in this section we not only consider the fairness objective (Subsections~\ref{subsection:iclr} and~\ref{sec:kdd}), but also design experiments (Subsections~\ref{sec:synth} and~\ref{sec:mturk}) to directly evaluate the accuracy resulting from the assignment procedures.


\subsection{Synthetic simulations}
\label{sec:synth}

\newcommand{\dbeta}[2]{\mathcal{B}\left(#1, #2 \right)}
\newcommand{\casea}{Non-mainstream papers.}
\newcommand{\caseb}{Many weak reviewers.}
\newcommand{\casec}{Few super-strong reviewers.}
\newcommand{\cased}{Adverse case.}
\newcommand{\casee}{Sparse similarities.}

We begin with synthetic simulations. We consider the instance of the reviewer assignment problem with $\numpapers = \numreviewers = 100$ and $\paperload = \maxrevload = 4$. We select the moderate values of $\numpapers$ and $\numreviewers$ to keep track of the optimal assignment $\hardassignment$ which we find as a solution of the corresponding integer linear programming problem. For every real-valued constant $\const$, we denote the matrix with all entries being equal to $\const$ as $\mathbf{c}$. Similarly, we denote the matrix with entries independently sampled from a Beta distribution with parameters $\left(\alpha, \beta \right)$ as $\mathbf{\dbeta{\alpha}{\beta}}$.

We consider the objective-score model of reviewers~\eqref{eqn:reviewer_model} with $\genvariance(\similarity) = 1 - \similarity$ together with estimator $\mle$. Thus, assignments $\fairassignment$, $\gargassignment$ and $\hardassignment$ aim to optimize $\assignmentquality[{\left(1 - \similarity \right)^{-1}}]{\assignment}$ while assignment $\tpms$ aims to maximize the cumulative sum of similarities $\greedyassignmentquality{\assignment}$ as defined in~\eqref{eqn:unfair_criteria}. 

In what follows we simulate the following problem instances:

{
\renewcommand{\theenumi}{({C\arabic{enumi}})}
\renewcommand{\labelenumi}{\theenumi}
\begin{enumerate}
\setlength\itemsep{0.3em}
\item 
	\label{sim:case_niche}
	{\bf \casea} There are $\numpapers_1 = 80$ conventional papers for which there exist $\numreviewers_1 = 80$ expert reviewers with high similarity, and $\numpapers_2 = 20$ non-mainstream papers for which all the reviewers have similarity smaller than or equal to $0.5$. There are also $n_2 = 20$ weak reviewers who have moderate similarities with papers from the first group and low similarities with papers from the second group. The similarities are given by the block matrix:
	\begin{align*}
        \simmatrix_1 = \begin{array}{c@{\!\!\!}l}
            \left[ \begin{array}{c|c} 
         		\mathbf{0.9} & \mathbf{0.5} \\ 
        		\hline
        		 \smash{\underbrace{ \addstackgap[2pt] {$\mathbf{0.5}$}}_{80}} & \smash{\underbrace{ \addstackgap[2pt] {$\mathbf{0.15}$}}_{20}}\\
         	\end{array}  \right]
        &
         \begin{array}[c]{@{}l@{\,}l}
           \left. \,\,\, \right\} & \text{$80$} \\
           \left.  \,\,\, \right\} & \text{$20$} \\
           \end{array}
           \end{array}
        \end{align*}
        \label{sin:niche}
	\item {\bf  \caseb} In this scenario there are $\numreviewers_1 = 25$ strong reviewers with high similarity with every paper and $\numreviewers_2 = 75$ weak reviewers with small similarity with every paper:
	\begin{align*}
        \simmatrix_2 = \begin{array}{c@{\!\!\!}l}
            \left[ \begin{array}{c} 
         		\mathbf{0.8} + 0.2 \times \dbeta{1}{3}  \\ 
        		\hline
        		 \smash{\underbrace{ \addstackgap[2pt] {$\mathbf{0.1} + 0.2 \times \dbeta{1}{3}$}}_{100}} \\
         	\end{array}  \right]
        &
         \begin{array}[c]{@{}l@{\,}l}
           \left. \,\,\, \right\} & \text{$25$} \\
           \left.  \,\,\, \right\} & \text{$75$} \\
           \end{array}
           \end{array}
        \end{align*}
        \label{sin:weak}
        
	\item {\bf  \casec} The following example tests the algorithms in scenario when some small number of the reviewers are much stronger than the others. Similarities for this scenario are given by the block matrix:
		\begin{align*}
        \simmatrix_3 = \begin{array}{c@{\!\!\!}l}
            \left[ \begin{array}{c|c} 
                  	\mathbf{0.98}   &  \mathbf{0.9} \\ 
					\hline 
        		    \mathbf{0} & \mathbf{0.7} \\
		 		    \hline
        		 \smash{\underbrace{ \addstackgap[2pt] {$\mathbf{0.9}$}}_{60}} & \smash{\underbrace{ \addstackgap[2pt] {$\mathbf{0.9}$}}_{40}}\\
         	\end{array}  \right]
        &
         \begin{array}[c]{@{}l@{\,}l}
           \left. \,\,\, \right\} & \text{$10$} \\
           \left.  \,\,\, \right\} & \text{$50$} \\
           \left.  \,\,\, \right\} & \text{$40$} \\
           \end{array}
           \end{array}
        \end{align*}
    \label{sin:strong}
    \setlength\itemsep{0.5em}
	\item {\bf  \cased} Having analyzed the inner working of our \algo algorithm, we construct a similarity matrix which is hard for the algorithm to compute the fair assignment.\footnote{We do not give an explicit expression of the matrix $\simmatrix_4$ for this case, due to its complicated structure.} 
	\label{sin:adverse}
	\item {\bf \casee}
	Each entry of similarity matrix $\simmatrix_5$ is zero with probability $0.8$ or otherwise is drawn independently and uniformly at random from $[0.1, 0.9]$.
	\label{sin:sparce}	
	\end{enumerate}
}

\subsubsection{Fairness}

\begin{table}[t]
\vskip 0.15in
\begin{center}
\begin{small}
\begin{sc}
\begin{tabular}{lccccccccccr}
\toprule
          & \multicolumn{5}{c}{Fairness $\assignmentquality[(1 - \similarity)^{-1}]{\assignment}$} & \multicolumn{5}{c}{Sum of Similiarities $\greedyassignmentquality{\assignment}$}   \\
          \cmidrule(lr){2-6}
          \cmidrule(lr){7-11}
          & Case 1 & Case 2 & Case 3 & Case 4 & Case 5 & Case 1 & Case 2 & Case 3 & Case 4 & Case 5          \\
\midrule
$\tpms$               & $4.7$  & $5.1$   & $13.3$  & $4.0 $  & $10.9$  & $300$ & $168$ & $295$ & $296$ & $311$    \\
$\hardassignment$     & $8.0$  & $13.1$  & $26.6$  & $14.0$  & $10.9$  & $296$ & $162$ & $232$ & $234$ & $175$    \\
$\gargassignment$     & $8.0$  & $5.0$   & $4.0$   & $14.0$  & $10.9$  & $296$ & $165$ & $188$ & $293$ & $296$    \\
$\fairassignment$ & $8.0$  & $13.1$  & $22.0$  & $6.5$   & $10.9$  & $296$ & $166$ & $239$ & $290$ & $309$    \\
\bottomrule
\end{tabular}
\end{sc}
\end{small}
\end{center}
\vskip -0.1in
\caption{Comparison of assignment produced by \algodot, \algharddot, \alggarg and \algtpms algorithms in terms of the fairness and the sum of similarities (higher values are better).}
\label{table:syn_comp}
\end{table}

In this section we analyze the quality of assignments produced by \algodot, \algharddot, \alggarg and \algtpms algorithms and for all the five cases described above. The results are summarized in Table~\ref{table:syn_comp} where we compute the measures of fairness $\assignmentquality[(1 - \similarity)^{-1}]{\assignment}$  and the conventional sum of similarities $\greedyassignmentquality{\assignment}$ for each of the assignments.

The results in Table~\ref{table:syn_comp} show that in all five cases  \algo algorithm finds an assignment $\fairassignment$  with at least as much fairness as $\tpms$. At the same time, the max cost heuristic that we use in Step~\ref{Sstep:pick_flow} of Subroutine~\ref{alg:fair_subroutine} helps the average quality (total sum similarity) of the assignment $\fairassignment$ to be either close to or larger than average quality of both $\gargassignment$ and $\hardassignment$.

In Case~\ref{sin:niche}, the \algtpms algorithm sacrifices the quality of reviewers for non-mainstream papers, assigning them to weak reviewers. In contrast, all other algorithms assign four best possible reviewers to these unconventional papers in order to maintain fairness. In Case~\ref{sin:weak}, the \algo and \alghard algorithms assign one strong reviewer for each paper while \algtpmsdot, in attempt to maximize the value of its goal function, assigns strong reviewers according to their highest similarities which leads to an unfair assignment. The \alggarg algorithm fails to find a fair assignment in Cases~\ref{sin:weak} and~\ref{sin:strong}: the poor performance of \alggarg algorithm is caused by the fact that some of the reviewers in our examples have similarities close to maximal, making the value of $\transformation(\similarity) = \frac{1}{1 - \similarity}$ large, which, in turn, makes the approximation guarantee~\eqref{eqn:garg_bound} of \alggarg algorithm weak. In Case~\ref{sin:adverse}, the \algo algorithm was unable to recover the fair assignment. Instead, the assignment within approximation ratio $1/3$, which is a bit better than the worst case $1/\paperload = 1/4$ approximation, was discovered. Finally, in Case~\ref{sin:sparce}, the all algorithms managed to recover fair assignment. However, we note that the total sum similarity of the $\hardassignment$ assignment is low as compared to other algorithms. The reason is that the corresponding solution of the integer linear programming problem in the \alghard algorithm is optimized for the fairness towards the worst-off paper and does not try to continue optimization, once the assignment for that paper is fixed. In contrast, both \algo and \alggarg algorithms try to maximize the fate of the second worst-off paper, when the assignment for the most worst-off paper is fixed.  

\subsubsection{Statistical accuracy}

{
\begin{figure}
\centering
\begin{subfigure}[t]{.33\textwidth}%
  \centering
 \raisebox{0.5\height}{\includegraphics[width=.75\linewidth]{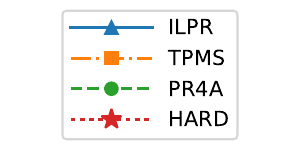}}
  \label{fig:legend}
\end{subfigure}%
\renewcommand*\thesubfigure{C\arabic{subfigure}}
\begin{subfigure}[t]{.33\textwidth}%
  \centering
  \includegraphics[width=0.95\linewidth]{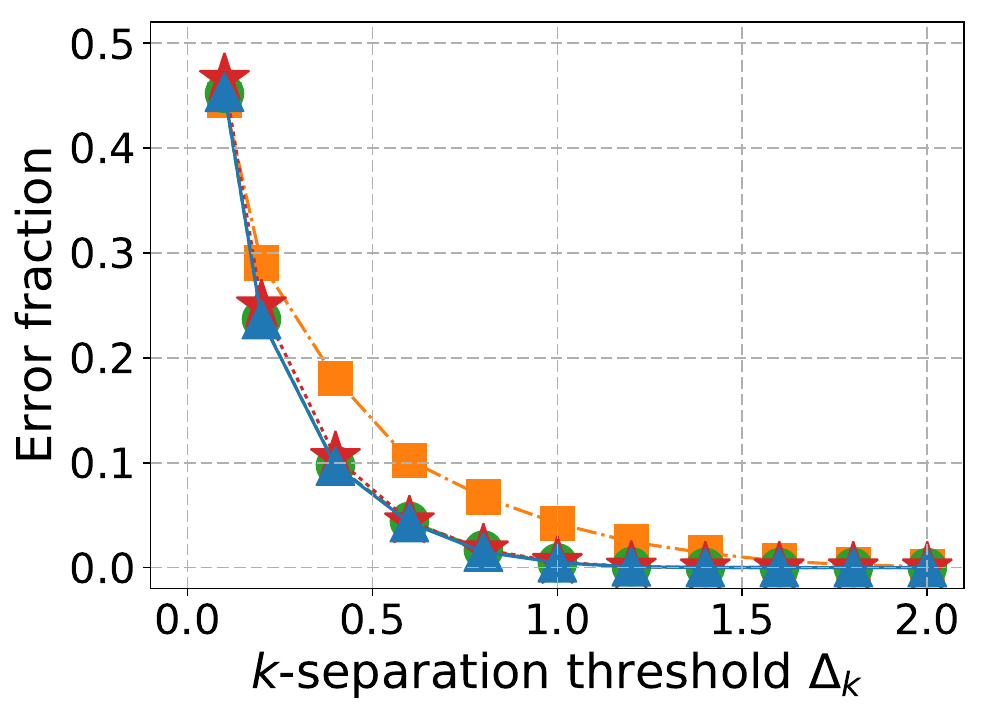}
  \caption{\casea}
  \label{fig:case1}
\end{subfigure}%
\begin{subfigure}[t]{.33\textwidth}%
  \centering
  \includegraphics[width=0.95\linewidth]{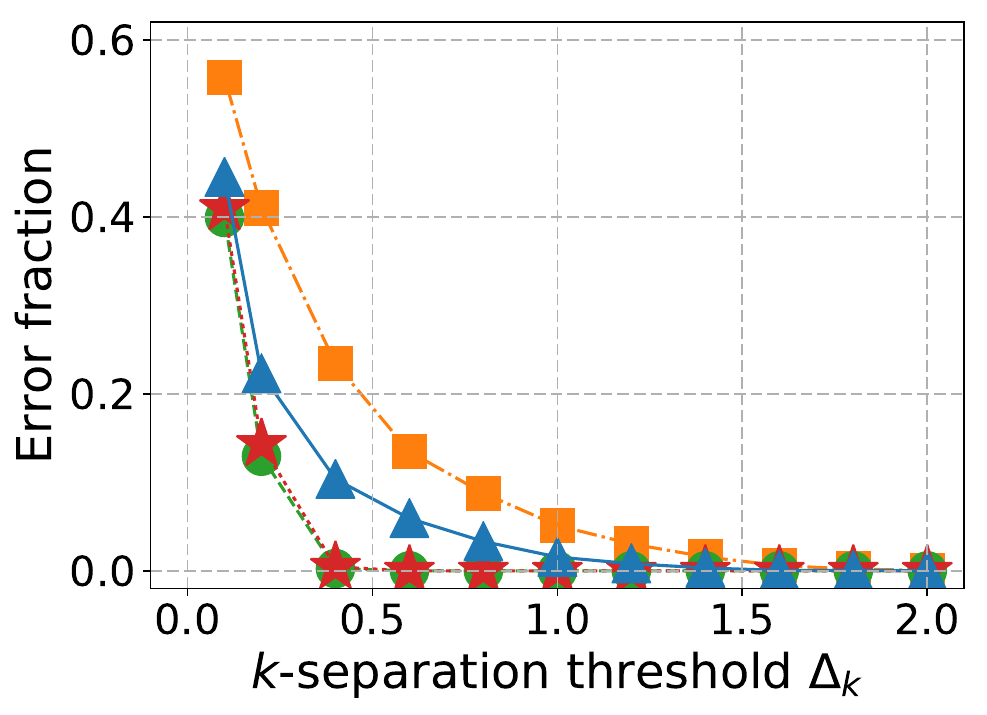}
  \caption{\caseb}
  \label{fig:case2}
\end{subfigure}%
\\\vspace{5pt}
\begin{subfigure}[t]{.33\textwidth}%
  \centering
  \includegraphics[width=0.95\linewidth]{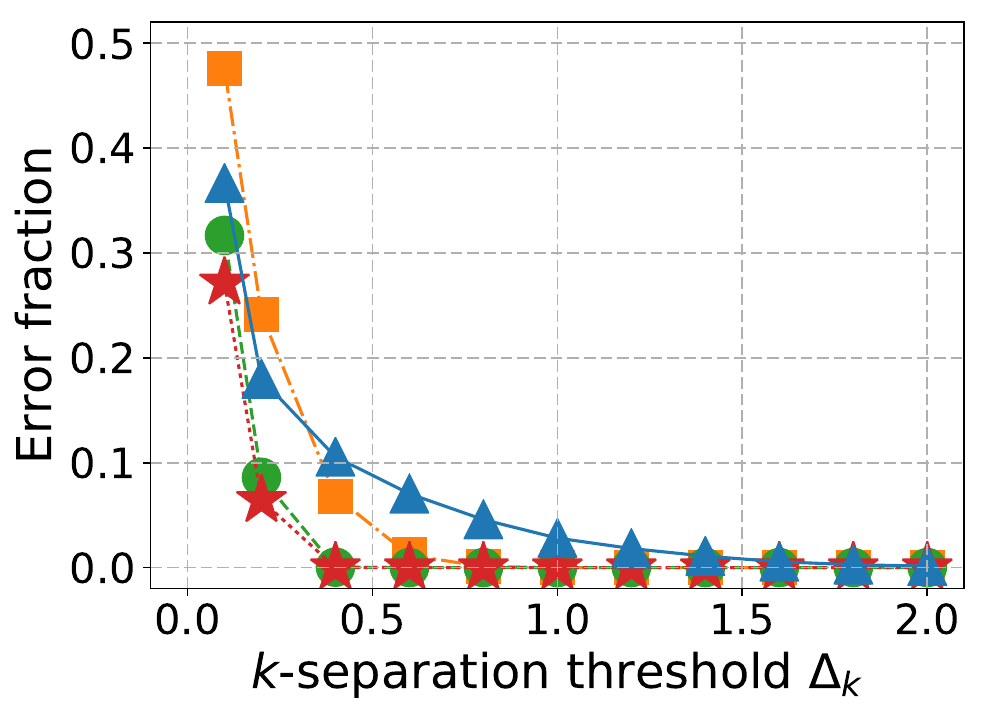}
  \caption{\casec}
  \label{fig:case3}
\end{subfigure}%
\begin{subfigure}[t]{.33\textwidth}%
  \centering
  \includegraphics[width=0.95\linewidth]{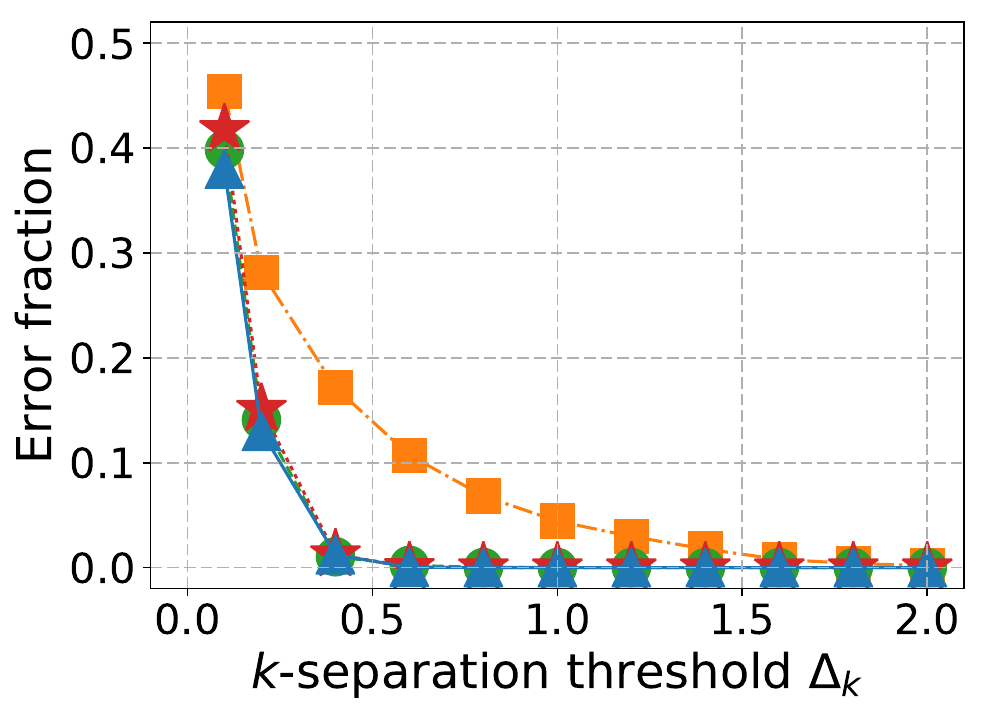}
  \caption{\cased}
  \label{fig:case4}
\end{subfigure}%
\begin{subfigure}[t]{.33\textwidth}%
  \centering
  \includegraphics[width=0.95\linewidth]{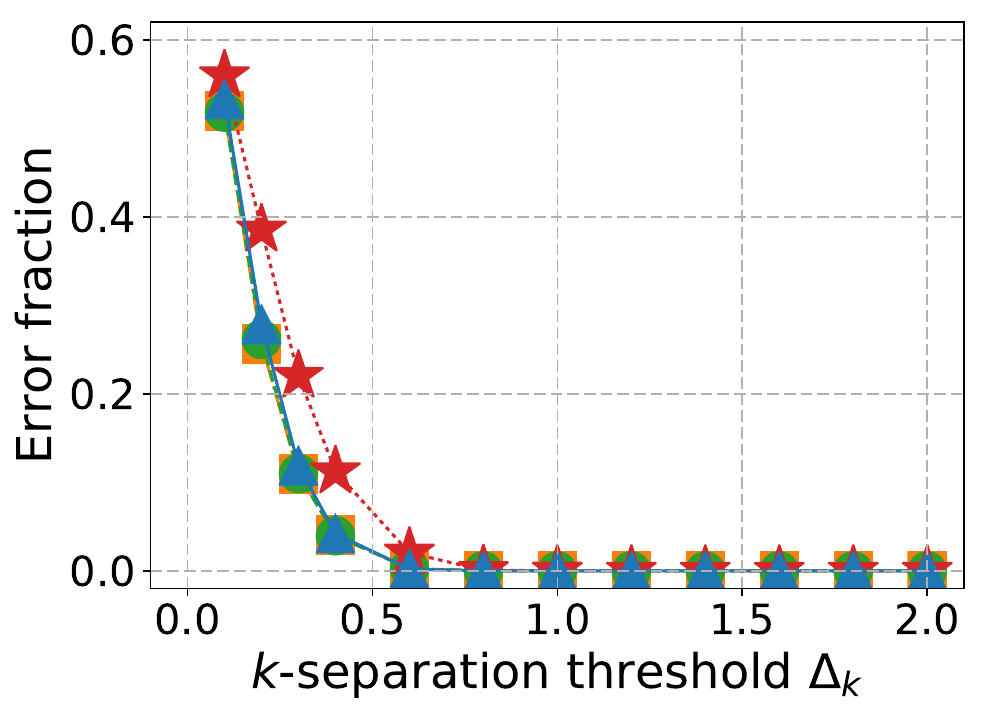}
  \caption{\casee}
  \label{fig:case5}
\end{subfigure}%
\caption{Fraction of papers incorrectly accepted by $\mle$ based on assignments produced by \algodot, \algharddot, \alggarg and \algtpms for different values of the separation threshold.}
\label{fig:error_rate}
\end{figure}
}

As we have pointed out, the main goal of the assignment procedure is to ensure the acceptance of the $\sizeofconf$ best papers $\truebest$. While in real conferences the acceptance process is complicated and involves discussions between reviewers and/or authors, here we consider a simplified scenario. Namely, we assume an objective-score model defined in Section~\ref{sec:obj_score_model} and reviewer model~\eqref{eqn:reviewer_model} with $\genvariance(\similarity) = 1 - \similarity$.

The experiment executes 1,000 iterations of the following procedure. We randomly choose $\sizeofconf = 20$ indices of the ``true best'' papers $\truebest = \left\{j_1, \ldots, j_{\sizeofconf}\right\} \subset [\numpapers]$. Each of these papers $j \in \truebest$ is assigned score $\truepaperquality_j = 1$, while for each of the remaining papers $j \in [\numpapers]\backslash\truebest$ we set $\truepaperquality_j = 1 - \threshold$, where $\threshold \in (0, 2]$. Next, given the similarity matrix $\simmatrix$, we compute assignments $\fairassignment$, $\hardassignment$, $\gargassignment$ and $\tpms$. For each of these assignments we compute the estimations of the set of top $\sizeofconf$ papers using the $\mle$ estimator and calculate the fraction of wrongly accepted papers.

For every similarity matrix $\simmatrix_{r}, r \in [5]$, and for every value of $\threshold \in \left\{0.1 k \ | k \in [20] \right\}$, we compute the mean of the obtained values over the 1,000 iterations. Figure~\ref{fig:error_rate} summarizes the dependence of the fraction of incorrectly accepted papers on the value of separation threshold $\threshold$ for all five cases~\ref{sin:niche}-\ref{sin:sparce}. 

The obtained results suggest that the increase in fairness of the assignment leads to an increase in the accuracy of the acceptance procedure, provided that the average sum similarity of the assignment does not decrease dramatically. The \algo algorithm significantly outperforms \algtpms both in terms of fairness and in terms of fraction of incorrectly accepted papers for the first four cases. The low fairness of assignments computed by \alggarg in Cases~\ref{sin:weak} and~\ref{sin:strong} lead to the large fraction of errors in the acceptance procedure. As we noted earlier, the \alggarg algorithm has weak approximation guarantees when the function $\transformation$ is allowed to be unbounded.  In section~\ref{sec:mturk} we will consider the mean score estimator ($\transformation(\similarity) = \similarity$) which is more suitable scenario for \alggarg algorithm. 

Interestingly, in Case~\ref{sin:adverse}, the \algo algorithm recovers sub-optimal assignment in terms of fairness, but still performs well in terms of the accuracy of the acceptance procedure. To understand this effect, for each of the assignments $\tpms, \hardassignment, \gargassignment$ and $\fairassignment$ we compute the sum similarity for all papers in the assignments and plot these values for $50$ the most worst-off papers in each of the assignment in Figure~\ref{fig:fairness}. Despite the inability of \algo  to find the fair assignment for the most worst-off paper, Corollary~\ref{corr:seq_deterministic} guarantees that sum similarities for the remaining papers will not be too far from the optimal, and we see this aspect in Figure~\ref{fig:fairness}\ref{sin:adverse}. As one can see, the sum similarity for all but tiny fraction of papers in $\fairassignment$ is large enough, thus ensuring the low fraction of incorrectly accepted papers. 
{
\begin{figure}
\centering
\renewcommand*\thesubfigure{C\arabic{subfigure}}
\begin{subfigure}[t]{.33\textwidth}
  \centering
  \setcounter{subfigure}{3}
  \raisebox{0.4\height}{\includegraphics[width=.75\linewidth]{legend.pdf}}
  \label{fig:legend2}
\end{subfigure}%
\begin{subfigure}[t]{.33\textwidth}
  \centering  \includegraphics[height=3.5cm]{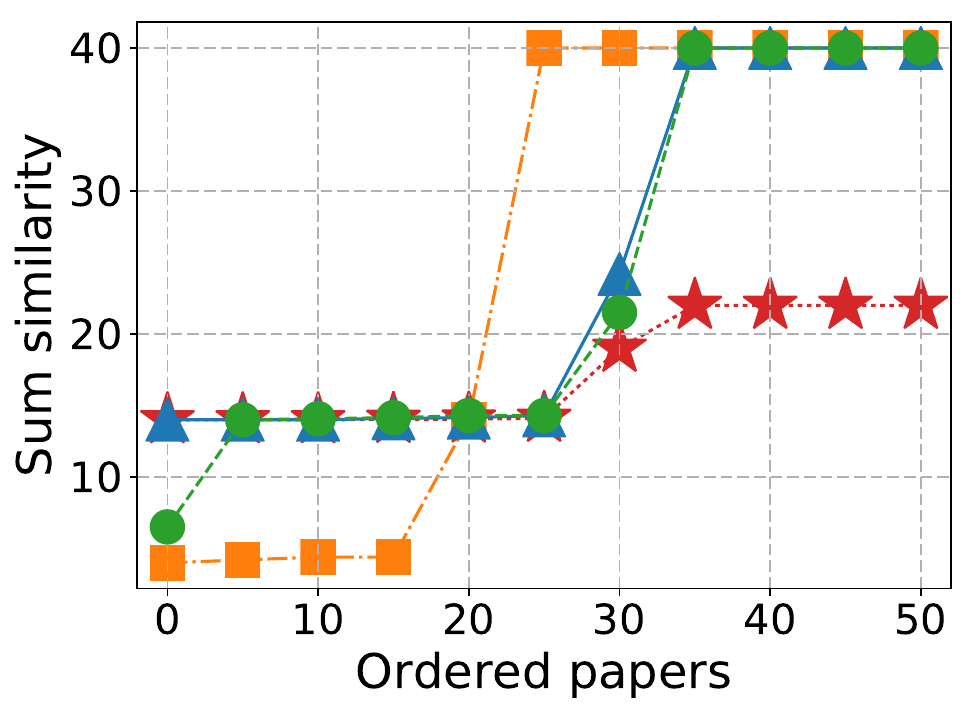}
  \caption{\cased}
  \label{fig:fairness_case4}
\end{subfigure}%
\begin{subfigure}[t]{.33\textwidth}
  \centering
  \includegraphics[height=3.5cm]{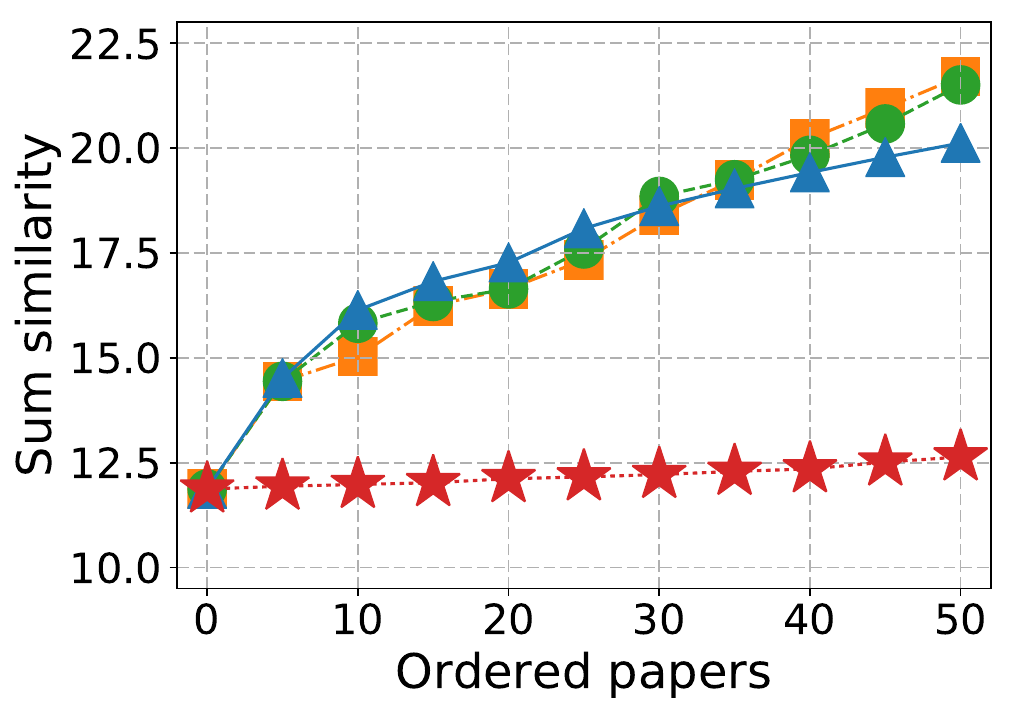}
  \caption{\casee}
  \label{fig:fairness_case5}
\end{subfigure}
\caption{Sum similarity for the 50 most worst-off papers in assignments produced by \algodot, \algharddot, \alggarg and \algtpmsdot.}
\label{fig:fairness}
\end{figure}
}

Finally, note that in Case~\ref{sin:sparce}, the \alghard algorithm, while having optimal fairness, has a lower accuracy as compared to other algorithms. As Figure~\ref{fig:fairness}\ref{sin:sparce} demonstrates, the \alghard algorithm does not optimize for the second worst off paper and recovers sub-optimal assignment for all but the most disadvantaged paper. In contrast, as Figure~\ref{fig:fairness}  suggests, the \alggarg and \algo algorithms do not stop their work after the most disadvantaged paper is satisfied, but instead continue to optimize the assignment for the remaining papers and eventually ensure not only fairness, but also high average quality of the assignment.


\subsection{Experiment on the approximation of ICLR similarity matrix}
\label{subsection:iclr}

\newcommand{\iclr}{ICLR'18}
\newcommand{\conflicts}{C}

In absence of publicly available similarity matrices from conferences, we are unable to compare the assignment computed by the \algo algorithm to the actual conference assignment. To circumvent this issue, we use an approximate version of the similarity matrix from the International Conference on Learning Representations (\iclr{}) that was constructed by~\citet{xu2019strategyproofArxiv} and compare the performance of the \algodot{} and \algtpmsdot{} assignment algorithms on this matrix. 

\subsubsection{Matrix construction}

The similarity matrix we use for comparison was constructed by~\citet{xu2019strategyproofArxiv} as follows. OpenReview (\url{openreview.net}) --- increasingly popular conference management system --- maintains a public database of all papers (with author identities being visible) submitted to the \iclr{} conference, thereby giving access to the pool of submissions.  Next, it was assumed that all authors of submissions are simultaneously reviewers and that there are no additional reviewers. The publication profiles of reviewers were constructed by scraping the data from databases of scientific publications. Finally, the open-source code (\url{bitbucket.org/lcharlin/tpms/}) and the material of the original paper~\citep{charlin13tpms} were used to compute the similarity matrix according to the TPMS procedure.

The process outlined above resulted in the similarity matrix $\simmatrix$ that has $\numreviewers = 2435$ reviewers and $\numpapers = 911$ papers. Additionally, it was assumed that any reviewer has a conflict of interests with the submitted papers that she/he has authored; these conflicts are represented by a binary matrix $\conflicts$ whose $(i, j)^{\text{th}}$ entry equals $1$ if and only if reviewer $i$ has a conflict with paper $j$. Similarity matrix $\simmatrix$ possesses a considerable heterogeneity as demonstrated by some papers having mean similarity with non-conflicting reviewers almost four times larger than others.

The large size of the similarity matrix makes computation of the optimally fair assignment infeasible, and hence in this section we do not compute the \algharddot{} assignment. Additionally, our implementation of the \alggargdot{} assignment algorithms was computationally inefficient and in absence of the publicly available source code we also exclude this algorithm from comparison. 
\begin{table}[t]
\vskip 0.15in
\begin{center}
\begin{small}
\begin{sc}
\begin{tabular}{lcccr}
\toprule
     Algorithm     & Fairness $\assignmentquality{\assignment}$ & Mean sum of sim. $\greedyassignmentquality{\assignment}$   \\
\midrule
$\tpms$                & 0.12 & 0.413     \\
$\fairassignment_1$ (one iteration)     & 0.15 ($+25$\%)       & 0.408  ($-1$\%)        \\
$\fairassignment$ (full)     & 0.15   ($+25$\%)  & 0.406 ($-2$\%)       \\
\bottomrule
\end{tabular}
\end{sc}
\end{small}
\end{center}
\vskip -0.1in
\caption{Results of the experiment on the approximation of \iclr{} similarity matrix. Values in brackets represent relative changes as compared to the TPMS assignment.}
\label{table:resultsiclr}
\end{table}

\subsubsection{Evaluation}

Having defined the similarity matrix and matrix of conflicts, we compute assignments of papers to reviewers with $\paperload = 4$ (each paper is assigned to 4 reviewers) and $\maxrevload = 2$ (each reviewer is allocated at most 2 papers) using the \algtpmsdot{} and \algodot{} assignment algorithms with the identity transformation function $\transformation(\similarity) = \similarity$. In addition to the standard load constraints, we require that no paper is assigned to a conflicting reviewer.  Finally, as pointed out in Section~\ref{sec:allpapers_approx}, the fairness guarantees of Theorem~\ref{thm:deterministic} are achieved after the first iteration of Steps~\ref{Algostep:loopkappa} to~\ref{Algostep:enditer} of Algorithm~\ref{alg:fair_assignment}. Hence, we include the corresponding assignment for comparison and denote it as $\fairassignment_1$.

Table~\ref{table:resultsiclr} summarizes the results of the experiment, comparing the resulting assignments in terms of fairness~\eqref{eqn:fairness_criteria} and cumulative similarity~\eqref{eqn:unfair_criteria}. We see that the fairness of the assignment computed by the~\algodot{} algorithm is significantly higher than the fairness of the \algtpmsdot{} algorithm. Similar to the case of synthetic simulations, the max cost heuristic used in Step~\ref{Sstep:pick_flow} of Subroutine~\ref{alg:fair_subroutine} helps our algorithm to maintain a high value of cumulative similarity, which is only marginally below the optimal value. 

The large size of the similarity matrix at hands makes evaluation of the optimal fairness achieved by $\hardassignment$ computationally prohibitive. However, we can still find an upper bound on $\assignmentquality{\hardassignment}$ by dropping reviewer load constraints and allowing all reviewers to review unlimited number of papers. The resulting bound allows us to compute a lower bound on the approximation ratio of the \algodot{} algorithm: 
\begin{align*}
    \frac{\assignmentquality{\fairassignment}}{\assignmentquality{\hardassignment}} \ge 0.98,
\end{align*}
which shows that in practice the approximation factor of the \algodot{} algorithm can be much better than the worst-case approximation factor $\frac{1}{\paperload}$ guaranteed by Theorem~\ref{thm:deterministic}.

Continuing the analysis, for each of the assignments $\tpms$, $\fairassignment_1$ and $\fairassignment$ we compute the sum similarity for all papers in the assignments and plot these values for $100$ the most worst-off papers in each of the assignment in Figure~\ref{fig:worst100}. This figure demonstrates that while the fairness guarantees of Theorem~\ref{thm:deterministic} can be achieved by a single iteration of Steps~\ref{Algostep:loopkappa} to~\ref{Algostep:enditer}, subsequent iterations help to improve the assignment for the second worst-off paper and so on. Finally, for each of the assignments $\tpms$ and $\fairassignment$ we sort papers in order of increasing sum similarity of assigned reviewers and plot the ratios (\algodot{} to \algtpmsdot) of these sums in Figure~\ref{fig:ratio}.  Figure~\ref{fig:ratio} shows that the \algodot{} algorithm indeed balances the assignment by improving the quality for the worst-off papers at the expense of decreasing the quality for the most benefiting papers.   

{
\begin{figure}
\centering
\begin{subfigure}[t]{.48\textwidth}
  \centering  
  \includegraphics[height=5cm]{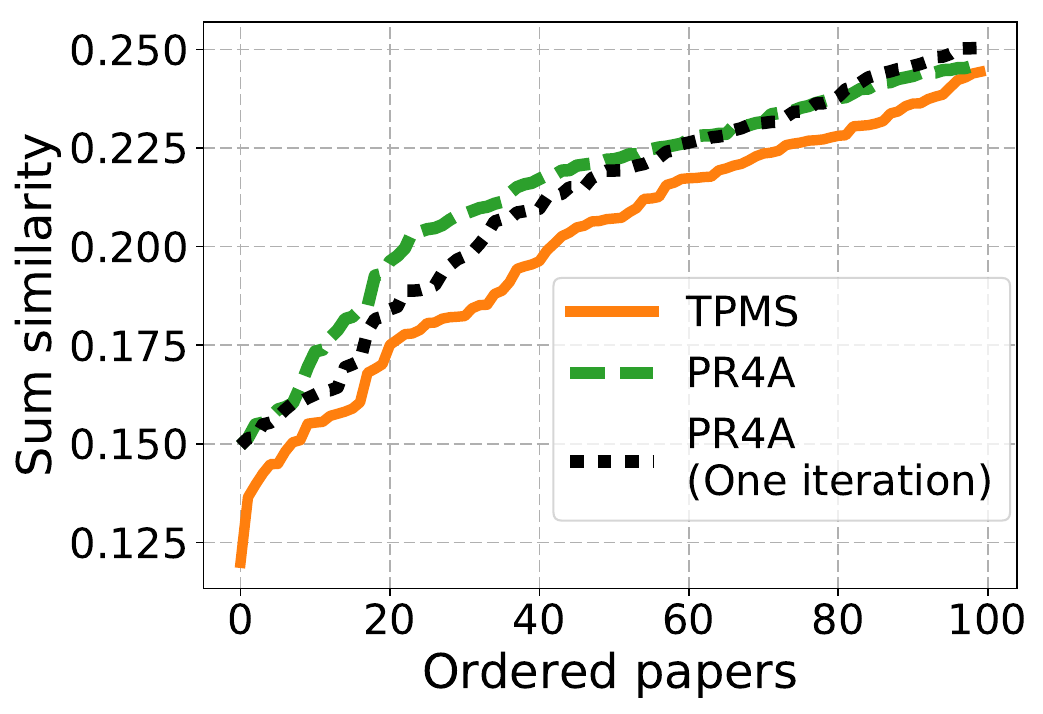}
  \caption{Sum similarity for 100 most disadvantaged papers in each assignment.}
  \label{fig:worst100}
\end{subfigure}%
\begin{subfigure}[t]{.48\textwidth}
  \centering
  \includegraphics[height=5cm]{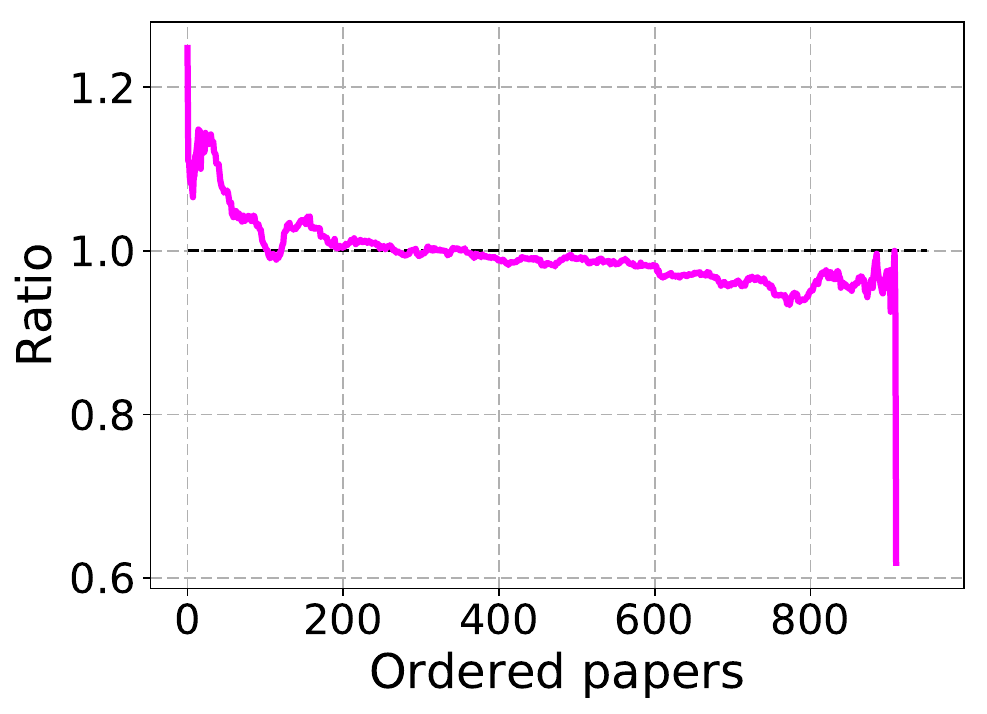}
  \caption{Ratio of ordered sum similarities in $\fairassignment$ to ordered sum similarities in $\tpms$.}
  \label{fig:ratio}
\end{subfigure}
\caption{Comparison on the approximation of \iclr{} similarity matrix.}
\label{fig:iclr}
\end{figure}
}


\subsection{Experiment on MIDL and CVPR similarity matrices}
\label{sec:kdd}

\newcommand{\fairirassignment}{\assignment^{\text{\sc FairIr}}}
\newcommand{\fairflowassignment}{\assignment^{\text{\sc FairFlow}}}
\newcommand{\fflow}{{\sc FairFlow}}
\newcommand{\fir}{{\sc FairIr}}

Subsequent to the publication of the first version of this paper~\citep{stelmakh2019pr4a}, a follow-up paper by~\citet{kobren19localfairness} has been published. There authors propose two novel assignment algorithms that also aim at ensuring the fairness of the assignment. In that work, the \algodot{} algorithm with the identity transformation function $(\transformation(\similarity) = \similarity)$ was compared with other assignment algorithms on similarity matrices from three real conferences:  Medical Imaging and Deep Learning Conference (MIDL), and two editions of the Conference on Computer Vision and Pattern Recognition (CVPR'17 and CVPR'18). With the kind permission of Ari Kobren, we describe the results of their experiments in which our algorithm was evaluated. 

\subsubsection{Brief discussion of the algorithms by Kobren et al.} 

We begin with a brief theoretical comparison of the \algodot{} algorithm with the algorithms proposed by~\citet{kobren19localfairness}. Recall that the \algodot{} algorithm aims at optimizing fairness of the assignment~\eqref{eqn:fairness_criteria} and does not directly optimize for the total sum similarity. However, when in its inner workings the algorithm faces a choice between different suitable similarity matrices (Step~\ref{Sstep:pick_flow} of the Subroutine~\ref{alg:fair_subroutine}), it can heuristically optimize for the total sum similarity by using the  max cost heuristic. In contrast, \citet{kobren19localfairness} consider a problem of optimizing for the total sum similarity of the assignment \emph{with an additional constraint of each paper having the sum similarity larger than some threshold $T$}, which can be specified by user or found by the binary search. They design two novel algorithms which we refer to as \fir{} and \fflow{}.

Given a feasible instance of the reviewer assignment problem, the \fir{} algorithm is able to compute the assignment with the optimal value of the total sum similarity, violating the fairness constraints by an additive factor which is upper bounded by the maximum entry of the similarity matrix. In that, fairness guarantees of \fir{} are equivalent to those of \alggargdot{} (and hence may become vacuous when similarity matrix is significantly heterogeneous), but additionally the  \fir{} algorithm achieves the highest possible value of sum similarity.\footnote{Observe that this value is lower than those achieved by \algtpmsdot{} as \fir{} has additional constraint on the fairness of the assignment.} The \fflow{} algorithm is a heuristic which does not have theoretical guarantees, but in return has much lower computational complexity.

Another difference between \algodot{} and the algorithms proposed by~\citet{kobren19localfairness} is that both \fir{} and \fflow{} allow to specify a \emph{lower bound on reviewer load}, thereby ensuring that each reviewer reviews at least some number of papers. In our work, we do not study such constraints and \algodot{} does not support such constraints as is. Hence, below we report only those comparisons in which our algorithm was evaluated by~\citet{kobren19localfairness}, that is, the comparisons in which the lower bound on reviewer load was not enforced. 

Overall, the \fir{} and \fflow{} algorithms aim at balancing the fairness and the total sum similarity of the assignment. By choosing an appropriate heuristic in Step~\ref{Sstep:pick_flow} of the Subroutine~\ref{alg:fair_subroutine}, one can ensure that \algodot{} also heuristically optimizes for the total sum similarity. Let us now report the experimental results of~\citet{kobren19localfairness} that allows to compare the algorithms on both objectives of fairness and total sum similarity.

\subsubsection{Summary of the experiments} 

\begin{table}[t]
\vskip 0.15in
\begin{center}
\begin{small}
\begin{sc}
\begin{tabular}{lllccc}
\toprule
     Conference & Parameters & Algorithm & Time (s) &  Fairness $\assignmentquality{\assignment}$ & Mean sum of sim. $\greedyassignmentquality{\assignment}$  \\
\midrule
\multirow{4}{*}{MIDL}  &                       & $\tpms$                 & 0.1    & 0.90    & 1.71  \\
& $\numreviewers = 177, \ \numpapers = 118$    &$\fairassignment$        & 293.8  & 0.92    & 1.67  \\
& {$\paperload = 3, \ \maxrevload = 4$}        &$\fairirassignment$      & 1.6    & 0.93    & 1.71  \\
&                                              &$\fairflowassignment$    & 1.2    & 0.94    & 1.68  \\
\midrule
\multirow{4}{*}{CVPR'17} &                     & $\tpms$                 & 47     & 0      & 2.08  \\
&  $\numreviewers = 1373, \ \numpapers = 2623$ & $\fairassignment_1$  & 3251   & 0.77   & 1.96  \\
& {$\paperload = 3, \ \maxrevload = 6$}        & $\fairirassignment$     & 595    & 0.27   & 2.05  \\
&                                              & $\fairflowassignment$   & 225    & 0.77   & 1.69  \\
\midrule
\multirow{4}{*}{CVPR'18} 
 &                                              & $\tpms$                & 257    & 1.37   & 22.23  \\
& $\numreviewers = 2840, \ \numpapers = 5062$   & $\fairassignment_1$  & 8684   & 12.68  & 21.48  \\
& $\paperload = 3, \ \maxrevload = 9$           & $\fairirassignment$    & 3786   & 7.19   & 22.18  \\
&                                               & $\fairflowassignment$  & 910    & 11.12  & 17.98  \\
\bottomrule
\end{tabular}
\end{sc}
\end{small}
\end{center}
\vskip -0.1in
\caption{Results of the experiment conducted by Kobren et al. on similarity matrices from real conferences. On large instances only a single iteration of the \algodot{} algorithm was computed and the corresponding assignment is denoted $\fairassignment_1$.}
\label{table:kdd_paper}
\end{table}

The key summary statics of the~\citet{kobren19localfairness} experiments are represented in Table~\ref{table:kdd_paper}.\footnote{We omit some statistics which are not of direct interest (for example, max sum similarity in the assignment).} For each similarity matrix, the assignments respecting the corresponding paper and reviewer load constraints were computed by the \algtpmsdot{}, \algodot{}, \fir{} and \fflow{} algorithms. These assignments were then compared based on (a) running time of the algorithm, (b) fairness of the assignment and (c) mean sum similarity of the assignment. First, we notice that our naive implementation of the \algodot{} algorithm is significantly slower than all other algorithms, and for large instances only a single iteration of the algorithm can be computed in a reasonable time (recall that even one iteration is sufficient to satisfy the fairness guarantees of Theorem~\ref{thm:deterministic}). Nonetheless, even on the largest instance with more than 5,000 papers the running time of the first iteration of our algorithm took less than three hours which is still feasible given that the full assignment procedure needs to be run only once in the conference timeline.

The remaining two dimensions of comparison represent two notions of quality of the assignment: fairness and total sum similarity. Ideally, we would like to have an algorithm which simultaneously optimizes both of these notions. Figure~\ref{fig:2dcomp} visualizes the comparison of the algorithms and is constructed as follows. For each of the three experiments, we compute the maximum value of fairness achieved by any of the algorithms. Using this value, for each algorithm we compute its ``competetiveness'' as the fairness achieved by that algorithm divided by the maximum fairness. We then repeat the same for the total sum similarity. As a result, in each experiment the performance of each algorithm can be represented as a data point in two-dimensional space where x-axis represents the competitiveness in terms of fairness and y-axis represents the competitiveness in terms of the total sum similarity.

Figure~\ref{fig:2dcomp} demonstrates that in each of the three experiments the \algodot{} algorithm (even with one iteration) was able to achieve maximum or close-to-maximum values of both fairness and total sum similarity. In contrast, each of the other algorithms under consideration achieved considerably lower value of either fairness or total sum similarity in two out of three experiments. 

Overall, we conclude that while being considerably (but not prohibitively) slower than other algorithms, \algodot{} managed to achieve the best balance of fairness and total sum similarity, despite optimizing the latter objective only heuristically.

\begin{figure}
    \centering
    \includegraphics[width=0.5\textwidth]{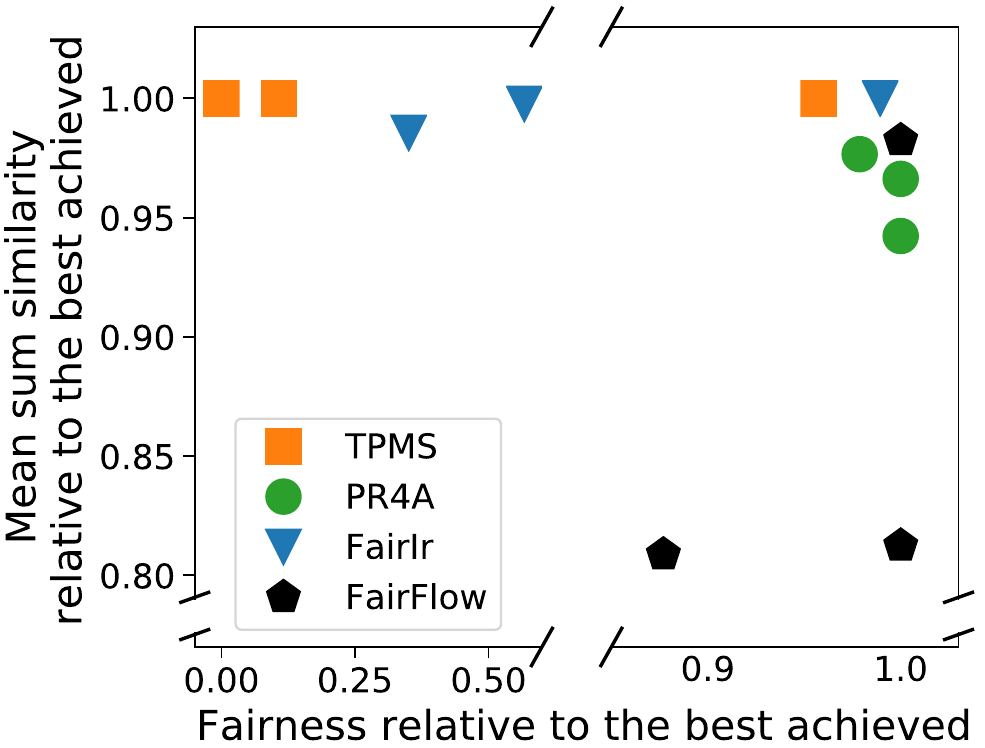}
    \caption{Visualization of comparison of the algorithms based on fairness and total sum similarity.}
    \label{fig:2dcomp}
\end{figure}


\subsection{Experiment on Amazon Mechanical Turk}
\label{sec:mturk}
Even if peer-review data from conferences was available to us, it would not allow for an objective evaluation of any assignment algorithm with respect to accuracy of the acceptance procedure. There are two reasons for this hinderance: (a) No ground truth ranking is available; and (b) The data contains only reviews that correspond to one particular assignment and has missing reviews for other assignments. 

In this section we present an experiment which we carefully design to overcome the fundamental issues with objective empirical evaluations of reviewer assignments. Our experiment allows us to directly measure the accuracy of final decisions to evaluate any assignment.  We execute our experiment on the Amazon Mechanical Turk (\url{mturk.com}) crowdsourcing platform. 


\subsubsection{Design of experiment}

We designed the experiment in a manner that allows us to objectively evaluate the performance of any assignment algorithm. Specifically, the experiment should provide us access to some similarities between reviewers and papers, execute any assignment algorithm, and eventually objectively evaluate the final outcome. 

The experiment considers crowdsourcing workers as reviewers and a number of general knowledge questions as papers. Specifically, 80 workers were recruited and presented with a list of 60 flags of different countries. The workers were asked to determine the country of each flag, choosing one of five options for each question. The interface of the task is represented in Figure~\ref{fig:interface}. Unknown to the worker, the 60 countries comprised 10 countries each from 6 different geographic regions. Three participants did not attempt some of the questions and their responses were discarded from the dataset. The dataset is available on the first author's website.


\subsubsection{Evaluation}

After obtaining the data from Amazon Mechanical Turk, we executed the following procedure for 1,000 iterations. In each of the 6 regions, we first split the 10 questions into two sets: a ``gold standard'' set of 8 questions chosen uniformly at random and an ``unresolved'' set comprising the 2 remaining questions. The set of all 12 unresolved questions are analogous to papers in the peer-review setting ($\numpapers=12$). We computed the similarity of any worker to any paper (question) as the fraction of questions that the worker answered correctly among the 8 gold standard questions for the region corresponding to that paper (question). Having computed the similarities, we selected $\numreviewers=40$ of the workers uniformly at random and created five assignments $\tpms, \fairassignment$, $\gargassignment$, $\hardassignment$ and $\rand$, with identity transformation function $\transformation(\similarity) = \similarity$,  where $\rand$ is a random feasible assignment. In each of these assignments, every question was answered by $\paperload = 3$ workers and every worker answered at most $\maxrevload = 2$ questions. Finally, for each assignment, we computed the answers for the remaining $\numpapers = 12$ questions by taking a majority vote of the responses from workers assigned to each question. Ties are also considered as mistakes.

At the end of all iterations, we computed the fraction of questions whose final answers are estimated incorrectly under the five assignments as well as the mean fairness $\assignmentquality{\assignment}$ and conventional sum of similarities $\greedyassignmentquality{\assignment}$. We summarize the results in Table~\ref{table:results}. We see that all non-trivial algorithms significantly outperform random assignment. However, $\tpms$ incurs about $8\%$ increased error as compared to $\fairassignment$.

Similar to Case~\ref{sin:sparce} of synthetic experiments, the optimally fair assignment $\hardassignment$ turns out to incur larger fraction of errors as compared to approximations $\fairassignment$ and $\gargassignment$. The reason is that the assignment $\hardassignment$ maximizes the quality of the assignment with respect to the most ``disadvantaged'' question, but in contrast to $\fairassignment$ and $\gargassignment$, does not care about the fate of remaining questions. 

We also see that $\fairassignment$ slightly outperforms $\gargassignment$ in terms of the fraction of errors while having slightly smaller average fairness. One reason for this is that in parallel with $\assignmentquality{\fairassignment}$ being close to optimal, \algo algorithm managed to achieve the high value of conventional sum of similarities, thus maintaining a balance between the fairness $\assignmentquality{\assignment}$ and the global objective $\greedyassignmentquality{\assignment}$.

We find these observations to be of notable interest for the actual conference peer-review scenarios. The task of identifying flags in the experiment involved a rather homogeneous set of similarities (in the sense that each worker  either knew many or only few flags) where optimizing~\eqref{eqn:unfair_criteria} or~\eqref{eqn:fairness_criteria} would yield similar results. In contrast, the significantly higher heterogeneity in peer-review, the presence of many non-mainstream papers as well as both very strong and very weak reviewers, is expected to further amplify the observed improvements offered by the \algo algorithm as compared to \algtpms and \alggargdot. 

\begin{figure}
    \centering
    \includegraphics[width=0.5\textwidth]{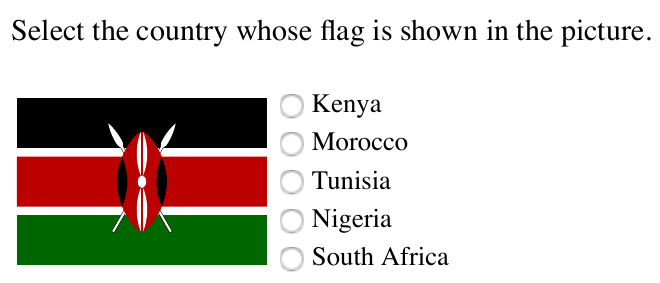}
    \caption{Question interface}
        \label{fig:interface}
\end{figure}

\begin{table}[t]
\vskip 0.15in
\begin{center}
\begin{small}
\begin{sc}
\begin{tabular}{lcccr}
\toprule
     Algorithm     & Error fraction & Error increase & Mean fairness $\assignmentquality{\assignment}$ & Mean sum of sim.  $\greedyassignmentquality{\assignment}$ \\
\midrule
$\rand$                & 0.394       & +275\%    &  6.4 & 171.1  \\
$\tpms$                & 0.113       & +8\%      & 20.8 & 274.6  \\
$\hardassignment$      & 0.110       & +5\%      & 21.9 & 269.8  \\
$\gargassignment$      & 0.108       & +3\%      & 21.7 & 270.4  \\
$\fairassignment$      & 0.105       & ---       & 21.6 & 272.9  \\
\bottomrule
\end{tabular}
\end{sc}
\end{small}
\end{center}
\vskip -0.1in
\caption{Results of the experiment on Amazon Mechanical Turk.}
\label{table:results}
\end{table}

\newpage
\section{Proofs}
\label{sec:proofs}

We now present the proofs of our main results.

\subsection{Proof of Theorem~\ref{thm:deterministic}}

We prove the result in three steps. First, we establish a lower bound on the fairness of the \algo algorithm. Then we establish an upper bound on the fairness of the optimal assignment. Finally, we combine these bounds to obtain the result~\eqref{eqn:deterministic}.

\bigskip

\noindent
\textbf{Lower bound for the PeerReview4All algorithm.}

\noindent
We show a lower bound for the intermediate assignment $\tmpassignment$ at Step~\ref{Algostep:select_candidate} during the first iteration of Steps~\ref{Algostep:loopkappa} to~\ref{Algostep:enditer}. We denote this particular assignment as $\tmpassignment_1$. Note that in Step~\ref{Algostep:assignweakest} we fix the assignment for $\tmpassignment_1$'s worst-off papers into the final output, and hence we have $\assignmentquality[\transformation]{\tmpassignment_1} \ge \assignmentquality[\transformation]{\fairassignment_{\transformation}}$. On the other hand, by keeping track of $\assignment_0$ (Step~\ref{Algostep:enditer}), we ensure that in all of the subsequent iterations of Steps~\ref{Algostep:loopkappa} to~\ref{Algostep:enditer}, the temporary assignment $\tmpassignment$ will be at least as fair as $\tmpassignment_1$, which implies $\assignmentquality[\transformation]{\tmpassignment_1} = \assignmentquality[\transformation]{\fairassignment_{\transformation}}$. 

Getting back to the first iteration of Steps~\ref{Algostep:loopkappa} to~\ref{Algostep:enditer}, we note that when Step~\ref{Algostep:loopkappa} is completed, we have $\paperload$ assignments $\assignment_1, \ldots, \assignment_{\paperload}$ as candidates. Notice that for every $\tmpcapacity \in [\paperload]$, assignment $\assignment_{\tmpcapacity}$ is constructed with a two-step procedure by joining the outputs $\assignment_{\tmpcapacity}^1$ and $\assignment_{\tmpcapacity}^2$ of Subroutine~\ref{alg:fair_subroutine}. Recalling the definition~\eqref{eqn:criticalsim} of $\criticalsim_{\tmpcapacity}$, we now show that for every value of $\tmpcapacity \in [\paperload]$, the assignment $\assignment_{\tmpcapacity}^1$ satisfies:
\begin{align*}
	\min\limits_{j \in [\numpapers]} \min\limits_{i \in \reviewerset{\assignment_{\tmpcapacity}^{1}}(j)} \similarity_{ij} = \criticalsim_{\tmpcapacity}.	
\end{align*}

Consider any value of $\tmpcapacity \in [\paperload]$. The definition of $\criticalsim_{\tmpcapacity}$ ensures that there exist an assignment, say $\assignment^{\ast}$, which assigns $\tmpcapacity$ reviewers to each paper in a way that minimum similarity in this assignment equals $\criticalsim_{\tmpcapacity}$. Now note that Subroutine~\ref{alg:fair_subroutine},  called in Step~\ref{Algostep:subroutine_call} of the algorithm, adds edges to the flow network in order of decreasing similarities. Thus, at the time all edges with similarity higher or equal to $\criticalsim_{\tmpcapacity}$ are added, we have that no edges with similarity smaller that $\criticalsim_{\tmpcapacity}$ are added, and that all edges which correspond to the assignment $\assignment^{\ast}$ are also added to the network. Thus, a maximum flow of size $\numpapers \tmpcapacity$ is achieved and hence each assigned (reviewer, paper) pair has similarity at least $\criticalsim_{\tmpcapacity}$.  

Recalling that $\criticalsim_{\infty}$ is the lowest similarity in similarity matrix $\simmatrix$, one can deduce that $\assignmentquality[\transformation]{\assignment_{\tmpcapacity}} \ge \tmpcapacity \transformation(\criticalsim_{\tmpcapacity}) + \left(\paperload - \tmpcapacity \right) \transformation(\criticalsim_{\infty})$ due to the monotonicity of $\transformation$. Consequently, we have
\begin{align}
\label{eqn:pr4a_lower}
    \assignmentquality[\transformation]{\fairassignment_{\transformation}} \ge \assignmentquality[\transformation]{\assignment_{\tmpcapacity}} \ge \tmpcapacity \transformation(\criticalsim_{\tmpcapacity}) + \left(\paperload - \tmpcapacity \right) \transformation(\criticalsim_{\infty}),
\end{align} for all $\tmpcapacity \in [\paperload]$. Taking a maximum over all values of $\tmpcapacity \in [\paperload]$ concludes the proof.

\bigskip
\noindent
{\bf Upper bound for the optimal assignment $\hardassignment_{\transformation}$.}

\noindent
		Consider any value of $\tmpcapacity \in [\paperload]$.
		By definition~\eqref{eqn:criticalsim} of $\criticalsim_{\tmpcapacity}$, for any feasible assignment $\assignment \in \assignmentfamily$, there exists some paper $j_{\tmpcapacity}^{\ast} \in [\numpapers]$ for which at most $(\tmpcapacity - 1)$ reviewers have similarity strictly greater than $\criticalsim_{\tmpcapacity}$.
	 Let us now consider assignment $\hardassignment_{\transformation}$ and corresponding paper $j_{\tmpcapacity}^{\ast}$. This paper is assigned to at most $(\tmpcapacity - 1)$ reviewers with similarity greater than $\criticalsim_{\tmpcapacity}$ and to at least $(\paperload - \tmpcapacity + 1)$ reviewers with similarity smaller or equal to $\criticalsim_{\tmpcapacity}$. Recalling that $\criticalsim_0$ is the largest possible similarity, we conclude that due to monotonicity of $\transformation$, the following upper bound holds:
	\begin{align}
	\label{eqn:hard_upper}
			\assignmentquality[\transformation]{\hardassignment_{\transformation}} & = \min\limits_{j \in [\numpapers]} \slim_{i \in \reviewerset{\hardassignment_{\transformation}}(j)} \transformation{(\similarity_{ij}}) \le  \slim_{i \in \reviewerset{\hardassignment_{\transformation}}(j_{\tmpcapacity}^{\ast})} \transformation{(\similarity_{ij_{\tmpcapacity}^{\ast}}}) \le \left(\tmpcapacity - 1 \right)\transformation(\criticalsim_{0}) + \left(\paperload - \tmpcapacity + 1 \right)\transformation(\criticalsim_{\tmpcapacity}).
		\end{align}
	Taking a minimum over all values of $\tmpcapacity \in [\paperload]$, then yields an upper bound on the fairness of $\hardassignment_{\transformation}$.
	
\bigskip
\noindent
{\bf Putting it together.} \\
	To conclude the argument, it remains to plug in the obtained bounds~\eqref{eqn:pr4a_lower} and~\eqref{eqn:hard_upper} into ratio $\frac{\assignmentquality[\transformation]{\fairassignment_{\transformation}}}{\assignmentquality[\transformation]{\hardassignment_{\transformation}}}$:
	\begin{align*}
		\frac{\assignmentquality[\transformation]{\fairassignment_{\transformation}}}{\assignmentquality[\transformation]{\hardassignment_{\transformation}}} \ge \frac{\max\limits_{\tmpcapacity \in [\paperload]} \Big( \tmpcapacity \transformation(\criticalsim_{\tmpcapacity}) + \left(\paperload - \tmpcapacity \right) \transformation(\criticalsim_{\infty})\Big) }{\min\limits_{\tmpcapacity \in [\paperload]} \Big((\tmpcapacity - 1)\transformation(\criticalsim_0) + \left(\paperload - \tmpcapacity + 1 \right) \transformation(\criticalsim_{\tmpcapacity}) \Big) }.
	\end{align*}

	Setting $\tmpcapacity = 1$ in both numerator and denominator and recalling that $\transformation(\similarity) \ge 0 \ \forall \similarity \in [0, 1]$, we obtain a worst-case approximation in terms of required paper load: $\frac{\assignmentquality{\fairassignment}}{\assignmentquality{\hardassignment}} \ge \frac{1}{\paperload}$.


\subsection{Proof of Corollary~\ref{corr:seq_deterministic}}
Let us pause the \algo algorithm at the beginning of the $\currentiter^{\text{th}}$ iteration of Steps~\ref{Algostep:loopkappa} to~\ref{Algostep:enditer} and inspect its state. 

\begin{itemize}[noitemsep]
    \item The set $\tobeassigned$ consists of papers that are not yet assigned: 
    \begin{align*}
        \tobeassigned = [\numpapers] \backslash \left(\bigcup_{l = 1}^{\currentiter-1} \paperset_l \right).
    \end{align*}
    
    \item The vector of reviewers' loads $\vectmaxrevload$ is adjusted with respect to assigned papers. For every reviewer $i \in [\numreviewers]$, we have: 
    \begin{align*}
        \vectmaxrevload_i = \maxrevload - \card \left( \left\{j \in \bigcup_{l = 1}^{\currentiter-1} \paperset_l  \Big | i \in \reviewerset{\fairassignment_{\transformation}} (j) \right\} \right).
    \end{align*}
    
    \item The similarity matrix $\simmatrix_{\currentiter}$ consists of columns of the initial similarity matrix $\simmatrix$ which correspond to papers in $\tobeassigned$.
\end{itemize}

The only thing that connects the algorithm with the previous iterations is the assignment $\assignment_0$, computed in Step~\ref{Algostep:enditer} of the previous iteration. However, we note that the sum similarity for the worst-off papers, determined in Step~\ref{Algostep:assignweakest} of the current iteration (in other words, fairness of $\tmpassignment_{\currentiter}$ ), is lower-bounded by the largest fairness of the candidate assignments
$\assignment_1, \ldots, \assignment_{\paperload}$, which are computed in Step~\ref{Algostep:loopkappa}.

We now repeat the proof of Theorem~\ref{thm:deterministic} with the following changes. Instead of the similarity matrix $\simmatrix$, we use the updated matrix $\simmatrix_{\currentiter}$; instead of considering all papers $\numpapers$ we consider only papers from $\tobeassigned$; instead of assuming that each reviewer $i \in [\numreviewers]$ can review at most $\maxrevload$ papers, we allow reviewer $i \in [\numreviewers]$ to review at most $\vectmaxrevload_i$ papers. Hence, we arrive to the bound~\eqref{eqn:deterministic} on the fairness of $\tmpassignment_{\currentiter}$, where $\hardassignment$ should be read as $\hardassignment\left(\tobeassigned \right) = \hardassignment\left(\paperset_{\{\currentiter:\totaliter\}} \right)$ and values $\criticalsim_{\tmpcapacity}, \tmpcapacity \in \{0, \ldots, \paperload \} \cup \{\infty\}$ are computed for similarity matrix $\simmatrix_r$ and constraints on reviewers' loads $\vectmaxrevload$. Thus, we obtain~\eqref{eqn:seq_deterministic} and conclude the proof of the corollary.


\subsection{Proof of Theorem~\ref{thm:main_theorem}}

Before we prove the theorem, let us formulate an auxiliary lemma which will help us show the claimed upper bound. We give the proof of this lemma subsequently in Section~\ref{sec:prof_of_lemma}.

\begin{lemma}
\label{lemma:upper_bound}
Consider any valid assignment $\assignment \in \assignmentfamily$ and any estimator $\estimator \in \left\{ \mle, \avgest \right\} $. Then for every $\distance > 0$, the error incurred by $\estimator$ is upper bounded as 
\begin{align*}
	\sup\limits_{\left(\truepaperquality_1, \ldots, \truepaperquality_{\numpapers} \right) \in \problemfamily(\distance)}& \prob{\accepted_{\sizeofconf}\left( \assignment, \estimator \right) \ne \truebest}  \le \sizeofconf (\numpapers -  \sizeofconf)   \exponent{-\left(\frac{\distance}{2 \criticalstd(\assignment, \estimator) } \right)^2},	
\end{align*}
where
\begin{align*}
	\criticalstd^2(\assignment, \estimator) = 	
	\begin{cases}
		\max\limits_{j \in [\numpapers]}  \left(\slim_{i \in \reviewerset{\assignment}(j)} \frac{1}{\std_{ij}^2}\right)^{-1} & \text{ if } \estimator = \mle \\
		\max\limits_{j \in [\numpapers]}  \left( \frac{1}{\paperload^2} \slim_{i \in \reviewerset{\assignment}(j)} \std_{ij}^2 \right)  & \text{ if } \estimator = \avgest. \\ 
	\end{cases}
\end{align*}
\end{lemma}

\subsubsection{Proof of upper bound}
\label{sec:main_thm_upper}
\newcommand{\mleassignment}{\assignment_{\text{MLE}}}
\newcommand{\avgassignment}{\assignment_{\text{MEAN}}}

First, recall from~\eqref{eqn:avg_score} the distribution of $\avgest_j, j\in [\numpapers]$. Then the \algo algorithm called with $\transformation = \transfavg$  simultaneously tries to maximize the fairness of the assignment with respect to $\transformation$ and minimize the maximum variance of the estimated scores $\avgest_{j}, j \in [\numpapers]$. Similarly, the choice of $\transformation = \transfmle$ ensures that together with optimizing the corresponding fairness, the algorithm also minimizes the maximum variance of $\mle_{j}, j \in [\numpapers]$, defined in~\eqref{eqn:mle_score}. Thus, the choice of the estimator defines the choice of the transformation function $\transformation$ which minimizes the maximum variance of the estimated scores. To maintain brevity, we denote  $\avgassignment = \fairassignment_{\transfavg}$, $\mleassignment = \fairassignment_{\transfmle}$, $\avgassignment(j) = \reviewerset{\avgassignment}(j)$ and $\mleassignment(j) = \reviewerset{\mleassignment}(j)$. 

Let now $\simmatrix \in \goodavgassignment(\highqualitycrowd)$. We begin with the pair of assignment and estimator $ \left( \avgassignment, \avgest \right)$. Notice that for arbitrary feasible assignment $\assignment \in \assignmentfamily$ and estimator $\avgest$,
\begin{align*}
	 \criticalstd^2(\assignment, \avgest) & =  \max\limits_{j \in [\numpapers]} \left(\frac{1}{\paperload^2} \slim_{i \in \reviewerset{\assignment}(j)} \std_{ij}^2 \right) = \frac{1}{\paperload^2} \max\limits_{j \in [\numpapers]} \left({\slim_{i \in \reviewerset{\assignment}(j)}  1 - \left(1 -\genvariance(\similarity_{ij}) \right)}\right) \\ & =  \frac{1}{\paperload^2} \left(\paperload - \min\limits_{j \in [\numpapers]} \slim_{i \in \reviewerset{\assignment}(j)} \left(\transfavg(\similarity_{ij})\right) \right)	= \frac{1}{\paperload^2} \left(\paperload - \assignmentquality[\transfavg]{\assignment} \right).	
\end{align*}
Now we can write
\begin{align*}
	\sup\limits_{\simmatrix \in \goodavgassignment(\highqualitycrowd)} \criticalstd^2(\avgassignment,  \avgest) & =  \frac{1}{\paperload^2} \left(\paperload - \highqualitycrowd \inf\limits_{\simmatrix \in \goodavgassignment(\highqualitycrowd)}\frac{\assignmentquality[\transfavg]{\avgassignment}}{\highqualitycrowd} \right) \\ & \le \frac{1}{\paperload^2} \left(\paperload - \highqualitycrowd \inf\limits_{\simmatrix \in \goodavgassignment(\highqualitycrowd)}\frac{\assignmentquality[\transfavg]{\avgassignment}}{\assignmentquality[\transfavg]{\hardassignment_{\transfavg}}} \right) \\& = \frac{\paperload - \highqualitycrowd \approximation_{\highqualitycrowd}}{\paperload^2}.	
\end{align*}
Using Lemma~\ref{lemma:upper_bound}, we conclude the proof for the mean score estimator:
\begin{align}
	 \sup\limits_{\substack{ \left(\truepaperquality_1, \ldots, \truepaperquality_{\numpapers} \right) \in \problemfamily(\distance) \\ \simmatrix \in \goodavgassignment(\highqualitycrowd) }} \prob{\accepted_{\sizeofconf}\left( \avgassignment, \avgest \right) \ne \truebest}  & \le \sizeofconf (\numpapers -  \sizeofconf)   \exponent{-\left(\frac{\distance}{2 \sup\limits_{\simmatrix \in \goodavgassignment(\highqualitycrowd)} \criticalstd(\avgassignment,  \avgest) } \right)^2} \label{eqn:chain1} \\ & \le  \numpapers^2  \exponent{-\frac{\paperload^2 \distance^2}{4 \left(\paperload - \highqualitycrowd \approximation_{\highqualitycrowd} \right)}} \le \numpapers^2 \exponent{- \ln \frac{\numpapers^2}{\errorrate}} \le \errorrate. \label{eqn:chain2}
\end{align}

Let us now consider the pair $\left(\mleassignment, \mle \right)$. It suffices to show that
\begin{align}
\label{eqn:sufficient}
	\sup\limits_{\simmatrix \in \goodavgassignment(\highqualitycrowd)} \criticalstd^2(\mleassignment, \mle) \le \sup\limits_{\simmatrix \in \goodavgassignment(\highqualitycrowd)} \criticalstd^2({\avgassignment,  \avgest)}.
\end{align}
Let us consider $\simmatrix \in \goodavgassignment(\highqualitycrowd)$. Recall from the proof of Theorem~\ref{thm:deterministic} that the fairness of the resulting assignment is determined in the first iteration of Steps~\ref{Algostep:loopkappa} to~\ref{Algostep:enditer}. After completion of Step~\ref{Algostep:loopkappa}, we have $\paperload$ candidate assignments $\assignment_1, \ldots, \assignment_{\paperload}$. Observe that Subroutine~\ref{alg:fair_subroutine} in Step~\ref{Sstep:pick_flow} uses the same heuristic for both $\avgassignment$ and $\mleassignment$. Hence, the $\paperload$ candidate assignments yielded when \algo constructs $\avgassignment$ coincide with the candidate assignments yielded when \algo constructs $\mleassignment$. Depending on the choice of $\transformation$, in Step~\ref{Algostep:select_candidate} the algorithm picks one assignment that maximizes fairness~\eqref{eqn:gen_fair_criteria} with respect to $\transformation$. Thus, 
\begin{align}
\label{eqn:fixing_fairness}
	\assignmentquality[\transfavg]{\avgassignment} = \max\limits_{\tmpcapacity \in [\paperload]} \assignmentquality[\transfavg]{\assignment_{\tmpcapacity}} \quad \text{and} \quad \assignmentquality[\transfmle]{\mleassignment} = \max\limits_{\tmpcapacity \in [\paperload]} \assignmentquality[\transfmle]{\assignment_{\tmpcapacity}}.
\end{align}
Hence, we have
\begin{align*}
	\criticalstd^2(\mleassignment, \mle) & = \max\limits_{j \in [\numpapers]} \left( \slim_{i \in \mleassignment(j)} \frac{1}{\std_{ij}^2} \right)^{-1} =  \max\limits_{j \in [\numpapers]} \left( \frac{1}{\slim_{i \in \mleassignment(j)} \frac{1}{\genvariance(\similarity_{ij})}}\right) \\ &  = \frac{1}{\assignmentquality[\transfmle]{\mleassignment}} \le \frac{1}{\assignmentquality[\transfmle]{\avgassignment}}.
\end{align*}
where the last ineqaulity is due to~\eqref{eqn:fixing_fairness}. 
Recalling the definition of the fairness~\eqref{eqn:gen_fair_criteria} and using Jensen's inequality, we continue:
\begin{align*}
     	\criticalstd^2(\mleassignment, \mle) &\le \max\limits_{j \in [\numpapers]} \left( \frac{1}{\paperload^2} \slim_{i \in \avgassignment(j)} \genvariance(\similarity_{ij}) \right)  =  \max\limits_{j \in [\numpapers]} \left( \frac{\paperload - \slim_{i \in \avgassignment(j)} \left( 1 - \genvariance(\similarity_{ij}) \right)}{\paperload^2} \right) \\ & =  \frac{\paperload - \assignmentquality[\transfavg]{\avgassignment}}{\paperload^2} = \criticalstd^2(\avgassignment,  \avgest).
\end{align*}
Taking a supremum over all $\simmatrix \in \goodavgassignment(\highqualitycrowd)$, we obtain~\eqref{eqn:sufficient} which together with Lemma~\ref{lemma:upper_bound} and the first part of the statement concludes the proof.

\subsubsection{Proof of lower bound}
\label{sect:main_thm_lower}

Proof of our lower bound is based on Fano's inequality~\citep{cover05elements} which provides a lower bound for probability of error in $\numhyp$-ary hypothesis testing problems.

Without loss of generality we assume that $\sizeofconf \le \frac{1}{2} \numpapers$. Otherwise, the result will hold by symmetry of the problems. 

We first claim that there exists a value $\similarity \in [0, 1]$ such that $\genvariance(\similarity) = 1 - \frac{\highqualitycrowd}{\paperload}$. Indeed, by assumptions of the theorem, $\genvariance$ is continuous strictly monotonically decreasing function and $\frac{q}{\paperload} \ge 1 - \genvariance(0)$. Thus, $\genvariance(0) \ge 1 - \frac{\highqualitycrowd}{\paperload}$. On the other hand, if $\genvariance(1) > 1 - \frac{\highqualitycrowd}{\paperload}$, then for every similarity matrix $\simmatrix$ we have
\begin{align*}
	\assignmentquality[\transfavg]{\assignment} \le \paperload \left(1 - \genvariance(1) \right) < \highqualitycrowd. 	
\end{align*}
The last inequality contradicts with the definition~\eqref{eqn:good_assignment} of $\goodavgassignment(\highqualitycrowd)$, verifying that
\begin{align*}
	\genvariance(0) \ge 1 - \frac{\highqualitycrowd}{\paperload} \ge \genvariance(1).  	
\end{align*}
Given that $\genvariance$ is continuous strictly monotonically decreasing function, we conclude that these exists $\similarity = \genvariance^{-1}\left(1 - \frac{\highqualitycrowd}{\paperload}\right) \in [0, 1]$.

Consider the similarity matrix $\tmpsimmatrix = \left\{\genvariance^{-1}\left(1 - \frac{\highqualitycrowd}{\paperload}\right) \right\}^{\numreviewers \times \numpapers}$. Observe that $\tmpsimmatrix \in \goodavgassignment(\highqualitycrowd)$, since every feasible assignment $\assignment \in \assignmentfamily$ has fairness 
\begin{align*}
\assignmentqualitymanual[\transfavg]{\tmpsimmatrix}{\assignment} = \min\limits_{j \in [\numpapers]} \slim_{i \in \reviewerset{\assignment}(j)} \left( \transfavg(\similarity_{ij}) \right)	 =  \min\limits_{j \in [\numpapers]} \slim_{i \in \reviewerset{\assignment}(j)}  \left\{ 1 - \genvariance\left(\genvariance^{-1}\left( 1 - \frac{\highqualitycrowd}{\paperload}\right)\right) \right\} = \highqualitycrowd.
\end{align*}
Thus, in any feasible assignment each paper $j \in [\numpapers]$ receives $\paperload$ reviewers with similarity exactly $\genvariance^{-1} \left(1 - \frac{\highqualitycrowd}{\paperload}\right)$.

To apply Fano's inequality, we need to reduce our problem to a hypothesis testing problem. To do so, let us introduce the set $\hypset$ of $(\numpapers - \sizeofconf + 1)$ instances of the paper accepting/rejecting problem: every problem instance in this set has the same similarity matrix $\tmpsimmatrix$, but differs in the set of top $\sizeofconf$ papers $\truebest$. We now consider the problem of distinguishing between these problem instances, which is equivalent to the problem of correctly recovering the top $\sizeofconf$ papers. More concretely, we denote the $(\numpapers - \sizeofconf + 1)$ problem instances as, $\hypset = \left\{1, 2, \ldots, {\numpapers - \sizeofconf + 1} \right\}$, where for any problem ${\ell} \in \hypset$ the set of top $\sizeofconf$ papers is denoted as $\truebest({\ell})$ and set as $\left\{1, 2, \ldots, \sizeofconf - 1 \right\} \cup \left\{ \sizeofconf - 1 + \ell \right\}$. The true quality of any paper $j \in [\numpapers]$ in any problem instance $\ell \in \hypset$ is
\begin{align*}
	\truepaperquality_j({\ell}) = \begin{cases}
 											 \delta & \text{if} \quad j \in \truebest({\ell}) \\
 											 0 &  \text{otherwise},
 										 \end{cases}
\end{align*}
thereby ensuring that $\left(\truepaperquality_1(\ell), \ldots, \truepaperquality_{\numpapers}(\ell) \right) \in \problemfamily(\distance)$, for every instance $\ell \in \hypset$.

Let $\problemvar$ denote a random variable which is uniformly distributed over elements of $\hypset$. Then given $\problemvar = {\ell}$, we denote a random matrix of reviewers' scores as $\scorematrix^{(\ell)} \in \reals^{\paperload \times \numpapers}$ whose $(r, j)^{\text{th}}$ entry is a score given by reviewer $i_r, r \in [\paperload]$, assigned to paper $j$ and 
\begin{align}
\label{eqn:score_matrix_dist}
	\scorematrix_{r j}^{(\ell)} \sim \begin{cases}
 										\gaussian\left( \distance, 1 - \frac{\highqualitycrowd}{\paperload} \right) & \text{if} \quad j \in \truebest({\ell}) \\
 										\gaussian\left(0, 1 - \frac{\highqualitycrowd}{\paperload} \right) & \text{otherwise}.
 								   \end{cases}
\end{align}
We denote the distribution of random matrix $\scorematrix^{(\ell)}$ as $\scoredist^{(\ell)}$. Note that $\scorematrix^{(\ell)}$ does not depend on the selected assignment $\assignment \in \assignmentfamily$. Indeed, recall from~\eqref{eqn:reviewer_model}, that assignment $\assignment$ affects only variances of observed scores. On the other hand, for any reviewer $i \in [\numreviewers]$  and for any paper $j \in [\numpapers]$, the score $\scoregiven_{i j}$ has variance $1 - \frac{\highqualitycrowd}{\paperload}$. Thus, for any feasible assignment $\assignment$ and any $\ell \in \hypset$, the distribution of random matrix $\scorematrix^{\ell}$  has the form~\eqref{eqn:score_matrix_dist}.

Now let us consider the problem of determining the index $\problemvar =  {\ell} \in \hypset$, given the observation $\scorematrix^{(\ell)}$ following the distribution $\scoredist^{(\ell)}$. Fano's inequality provides a lower bound for probability of error of every estimator $\hypest: \reals^{\paperload \times \numpapers} \to \hypset$ in terms of Kullback-Leibler divergence between distributions $\scoredist^{(\ell_1)}$ and $\scoredist^{(\ell_2)} \ \left( \ell_1 \ne \ell_2, \ \ell_1, \ell_2 \in [\numpapers - \sizeofconf + 1] \right)$:
\begin{align}
\label{eqn:fano}
	\prob{\hypest(\scorematrix) \ne {\problemvar}} \ge 1 - \frac{\max\limits_{{\ell_1} \ne {\ell_2} \in \hypset} \KL{\scoredist^{(\ell_1)}}{\scoredist^{(\ell_2)}} + \log 2}{\log \left(\card(\hypset)\right)},	
\end{align}
where $\card(\hypset)$ denotes the cardinality of $\hypset$ and equals $(\numpapers - \sizeofconf + 1)$ for our construction.

Let us now derive an upper bound on the quantity
\begin{align}
\label{eqn:KL_thm_main}
	\max\limits_{{\ell_1} \ne {\ell_2} \in \hypset} \KL{\scoredist^{(\ell_1)}}{\scoredist^{(\ell_2)}}.	
\end{align}
First, note that for each $\ell \in [\numpapers - \tmpcapacity + 1]$, entries of $\scorematrix^{(\ell)}$ are independent. Second, for arbitrary $\ell_1 \ne \ell_2$, the distributions of $\scorematrix^{(\ell_1)}$ and $\scorematrix^{(\ell_2)}$ differ only in two columns. Thus,
\begin{align*}
		 \KL{\scoredist^{(\ell_1)}}{\scoredist^{(\ell_2)}} =  \paperload \left\{ \KL{\gaussian\left(\distance, 1 - \frac{\highqualitycrowd}{\paperload}\right)}{\gaussian\left(0, 1 - \frac{\highqualitycrowd}{\paperload} \right)} + \KL{\gaussian\left(0, 1 - \frac{\highqualitycrowd}{\paperload}\right)}{\gaussian\left(\distance, 1 - \frac{\highqualitycrowd}{\paperload} \right)}\right\}.
\end{align*}
Some simple algebraic manipulations yield:
\begin{align}
\label{eqn:KL_gaussian}
	\KL{\gaussian\left(\distance, 1 - \frac{\highqualitycrowd}{\paperload}\right)}{\gaussian\left(0, 1 - \frac{\highqualitycrowd}{\paperload} \right)} = \KL{\gaussian\left(0, 1 - \frac{\highqualitycrowd}{\paperload}\right)}{\gaussian\left(\distance, 1 - \frac{\highqualitycrowd}{\paperload} \right)} = \frac{\distance^2}{2 \left( 1 - \frac{\highqualitycrowd}{\paperload}\right)}.
\end{align}
Finally, substituting~\eqref{eqn:KL_gaussian} in~\eqref{eqn:fano}, for $\numpapers > 6$ and for a sufficiently small constant $\const$, we have
\begin{align*}
	\prob{\hypest(\scorematrix) \ne {\problemvar}} \ge 1 - \frac{ \frac{\paperload^2 \distance^2}{\paperload - \highqualitycrowd}+ \log 2}{\log \left(\numpapers - \sizeofconf + 1\right)} \ge  1 - \frac{ \const^2 \ln \numpapers + 1}{\log \left(\frac{\numpapers}{2} + 1\right)}  \ge \frac{1}{2}.	
\end{align*} 
This lower bound implies 
\begin{align*}
	\sup\limits_{\simmatrix \in \goodavgassignment(\highqualitycrowd)}  \inf\limits_{\left(\estimator, \assignment \in \assignmentfamily \right)} \sup\limits_{ \left(\truepaperquality_1, \ldots, \truepaperquality_{\numpapers} \right) \in \problemfamily(\distance) } \prob{\accepted_{\sizeofconf}\left( \assignment, \estimator \right) \ne \truebest}  \ge \frac{1}{2}.		
\end{align*}


\subsubsection{Proof of Lemma~\ref{lemma:upper_bound}}
\label{sec:prof_of_lemma}

	First, let $\estimator = \avgest$. Then given a valid assignment $\assignment$, the estimates $\estpaperqualityavg_j, j \in [\numpapers]$, are distributed as
	\begin{align*}
		\estpaperqualityavg_j \sim \gaussian\left(\truepaperquality_j, \frac{1}{\paperload^2} \slim_{i \in \reviewerset{\assignment}(j)} \std_{ij}^2 \right) = \gaussian\left(\truepaperquality_j, \cumvar_j^2 \right), 	
	\end{align*}
	where we have defined $\cumvar_j^2 = \frac{1}{\paperload^2} \slim_{i \in \reviewerset{\assignment}(j)} \std_{ij}^2$. Now let us consider two papers $j_1, j_2$ such that $j_1$ belongs to the top $\sizeofconf$ papers $\truebest$ and $j_2 \notin \truebest$. The probability that paper $j_2$ receives higher score than paper $j_1$ is upper bounded as
	\begin{align*}
		 \prob{\estpaperqualityavg_{j_1} \le \estpaperqualityavg_{j_2}} &= \prob{\left(\estpaperqualityavg_{j_1} - \estpaperqualityavg_{j_2} \right) - \expectation{\estpaperqualityavg_{j_1} - \estpaperqualityavg_{j_2}} \le -\expectation{\estpaperqualityavg_{j_1} - \estpaperqualityavg_{j_2}}} \\ & \overset{(i)}{\leq} \exponent{-\frac{\left(\expectation{\estpaperqualityavg_{j_1} - \estpaperqualityavg_{j_2}}\right)^2}{2 \left( \cumvar_{j_1}^2 + \cumvar_{j_2}^2 \right)}} \overset{(ii)}{\leq} \exponent{-\left(\frac{\distance}{2 \criticalstd(\assignment, \avgest)} \right)^2},
	\end{align*}
	where inequality $(i)$ is due to Hoeffding's inequality, and inequality $(ii)$ holds because $\expectation{\estpaperqualityavg_{j_1} - \estpaperqualityavg_{j_2}} = \truepaperquality_{j_1} - \truepaperquality_{j_2} \ge \distance$ and $\criticalstd^2(\assignment, \avgest) = \max\limits_{j \in [\numpapers]} \cumvar_j^2$. The estimator makes a mistake if and only if at least one paper from $\truebest$ receives lower score than at least one paper from $[\numpapers] \backslash \truebest$. A union bound across every paper from $\truebest$, paired with $(\numpapers - \sizeofconf)$ papers from $[\numpapers]\backslash \truebest$, yields our claimed result.
		
	Let us now consider $\estimator = \mle$. Then it is not hard to see that
	\begin{align*}
		\estpaperqualityavg_j \sim \gaussian\left(\truepaperquality_j, \left( \slim_{i \in \reviewerset{\assignment}(j)} \frac{1}{\std_{ij}^2}  \right)^{-1} \right) = \gaussian\left(\truepaperquality_j, \cumvar_j^2 \right),	
	\end{align*}
	where we denoted $\cumvar_j^2 = \left( \slim_{i \in \reviewerset{\assignment}(j)} \frac{1}{\std_{ij}^2}  \right)^{-1}$.
	Proceeding in a manner similar to the proof for the averaging estimator yields the claimed result.


\subsection{Proof of Corollary~\ref{thm:main_mle}}

The proof of Corollary~\ref{thm:main_mle} follows along similar lines as the proof of Theorem~\ref{thm:main_theorem}.

\subsubsection{Proof of upper bound}

Let us consider some $\tmpcapacity \in [\paperload]$ and $\simmatrix \in \onegoodforeveryone(\singlehighqual)$. We apply Lemma~\ref{lemma:upper_bound} to proof the upper bound and in order to do so, we need to derive an upper bound on $\criticalstd(\fairassignment_{\transfmle}, \mle)$.  
\begin{align*}
	\criticalstd^2({\fairassignment_{\transfmle}, \mle}) & =  \max\limits_{j \in [\numpapers]} \left(\slim_{i \in \reviewerset{\fairassignment_{\transfmle}}(j)} \frac{1}{\std_{ij}^2}\right)^{-1} = \left( \min\limits_{j \in [\numpapers]} \slim_{i \in \reviewerset{\fairassignment_{\transfmle}}(j)} \transfmle (\similarity_{ij})\right)^{-1} \\& \le  \frac{1}{\frac{\tmpcapacity}{\genvariance(\singlehighqual)} + \frac{\paperload - \tmpcapacity}{\genvariance(0)}} = \frac{\genvariance(\singlehighqual) \genvariance(0)}{\tmpcapacity \genvariance(0) + ({\paperload - \tmpcapacity}) \genvariance(\singlehighqual)}.
\end{align*}
Thus,
\begin{align}
    \label{eqn:supMLEpr4a}
	\sup\limits_{\simmatrix \in \onegoodforeveryone(\singlehighqual)} \criticalstd^2({\fairassignment_{\transfmle}, \mle}) \le 	\frac{\genvariance(\singlehighqual) \genvariance(0)}{\tmpcapacity \genvariance(0) + ({\paperload - \tmpcapacity}) \genvariance(\singlehighqual)}.
\end{align}
It remains to apply Lemma~\ref{lemma:upper_bound} to complete our proof, and we do so by applying the chain of arguments~\eqref{eqn:chain1} and~\eqref{eqn:chain2} to the bound~\eqref{eqn:supMLEpr4a}, where the pair $(\fairassignment_{\transfavg}, \avgest)$ in ~\eqref{eqn:chain1} and~\eqref{eqn:chain2} is substituted with the pair $(\fairassignment_{\transfmle}, \mle)$.

\subsubsection{Proof of lower bound}

To prove the lower bound, we use the Fano's ineqaulity in the same way as we did when proved Theorem~\ref{thm:main_theorem}(b). However, we now need to be more careful with construction of working similarity matrix $\tmpsimmatrix \in \onegoodforeveryone(\singlehighqual)$.

As in the proof of Theorem~\ref{thm:main_theorem}(b), we assume $\sizeofconf \le \frac{\numpapers}{2}$. If the converse holds, than the result holds by symmetry of the problem. Next, consider arbitrary feasible assignment $\tmpassignment \in \assignmentfamily_{\tmpcapacity}$. Recall, that $\assignmentfamily_{\tmpcapacity}$ consists of assignments which assign each paper $j \in [\numpapers]$ to $\tmpcapacity$ instead of $\paperload$ reviewers such that each reviewer reviews at most $\maxrevload$ papers.

Now we define a similarity matrix $\tmpsimmatrix$ as follows:
\begin{align}
\label{eqn:simmatrix_tricky}
	\similarity_{ij} = \begin{cases}
 							\singlehighqual & \text{if} \quad i \in \reviewerset{\tmpassignment}(j) \\
 							0 & \text{otherwise}.
 					   \end{cases}
\end{align}
Thus, for each paper $j \in [\numpapers]$ there exist exactly $\tmpcapacity$ reviewers with non-zero similarity $\singlehighqual$ and in every feasible assignment $\assignment \in \assignmentfamily$ each paper $j \in [\numpapers]$ is assigned to at most $\tmpcapacity$ reviewers with non-zero similarity. Note that $\tmpsimmatrix \in \onegoodforeveryone(\singlehighqual)$.

Now let us consider the set of $(\numpapers - \sizeofconf + 1)$ problem instances $\hypset$ defined in Section~\ref{sect:main_thm_lower}. For every feasible assignment $\assignment \in \assignmentfamily$, if $\scorematrix^{(\assignment, \ell)}$ is a matrix of observed reviewers' scores for instance $\ell \in \hypset$, then $(r, j)^{\text{th}}$ entry of $\scorematrix^{(\assignment, \ell)}$ follows the distribution
\begin{align}
\label{eqn:score_distr_tricky}
	\scorematrix^{(\assignment, \ell)}_{rj}	= \begin{cases}
 										\gaussian\big(\distance \times \indicator{ j \in \truebest({\ell}) }, \genvariance(\singlehighqual) \big) & \text{if} \quad \tmpassignment_{i_r j} = 1 \\
 										\gaussian\big(\distance \times \indicator{ j \in \truebest({\ell})},  \genvariance(0) \big) & \text{if} \quad \tmpassignment_{i_r j} = 0, \\
 								  \end{cases}
\end{align}
where $i_r, r \in [\paperload]$ is reviewer assigned to paper $j$ in assignment $\assignment$.

We denote the distribution of random matrix $\scorematrix^{(\assignment, \ell)}$ as $\scoredist^{({\assignment}, \ell)}$. Note that in contrast to the proof of Theorem~\ref{thm:main_theorem}, here $\scorematrix^{(\assignment, \ell)}$ does depend on the selected assignment $\assignment \in \assignmentfamily$. Thus, instead of~\eqref{eqn:KL_thm_main}, we need to derive an upper bound on the quantity
\begin{align*}
	\sup\limits_{\assignment \in \assignmentfamily} \max\limits_{{\ell_1} \ne {\ell_2} \in \hypset} \KL{\scoredist^{({\assignment}, \ell_1)}}{\scoredist^{({\assignment}, \ell_2)}}. 	
\end{align*}

First, note that for each $\ell \in [\numpapers - \sizeofconf + 1]$ and for each feasible assignment $\assignment \in \assignmentfamily$, the entries of $\scorematrix^{({\assignment}, \ell)}$ are independent. Second, for arbitrary $\ell_1 \ne \ell_2$, the distributions of $\scorematrix^{({\assignment}, \ell_1)}$ and $\scorematrix^{({\assignment}, \ell_2)}$ differ only in two columns. Thus, for any feasible assignment $\assignment \in \assignmentfamily$, we have
\begin{align}
	 \KL{\scoredist^{({\assignment}, \ell_1)}}{\scoredist^{({\assignment}, \ell_2)}} & \le   \goodrevs_{\ell_1} \KL{\gaussian\big( \distance,  \genvariance(\singlehighqual) \big)}{\gaussian\big( 0,  \genvariance(\singlehighqual) \big)} + \left(\paperload - \goodrevs_{\ell_1} \right) \KL{\gaussian\big( \distance,  \genvariance({0}) \big)}{\gaussian\big( 0,  \genvariance({0}) \big)} \nonumber \\ & + \goodrevs_{\ell_2} \KL{\gaussian\big( 0,  \genvariance(\singlehighqual) \big)}{\gaussian\big( \distance,  \genvariance(\singlehighqual) \big)} + \left(\paperload - \goodrevs_{\ell_2} \right) \KL{\gaussian\big( 0,  \genvariance(0) \big)}{\gaussian\big( \distance,  \genvariance(0) \big)} \label{eqn:tricky_KL} \\ & = \left( \goodrevs_{\ell_1} + \goodrevs_{\ell_2} \right) \frac{\distance^2}{2 \genvariance(\singlehighqual)} + \left(2 \paperload - \goodrevs_{\ell_1} - \goodrevs_{\ell_2} \right) \frac{\distance^2}{2 \genvariance(0)} \label{eqn:tricky_KL2},
\end{align}
where $\goodrevs_{\ell_1}$ is the number of reviewers with similarity $\singlehighqual$ assigned to paper $\left(\sizeofconf - 1 + \ell_1 \right)$ in $\assignment$ and $\goodrevs_{\ell_2}$ is the number of reviewers with similarity $\singlehighqual$ assigned to paper $\left(\sizeofconf - 1 + \ell_2 \right)$. By construction of similarity matrix $\tmpsimmatrix$, for each $\ell \in [\numpapers - \sizeofconf + 1]$ and for each $\assignment \in \assignmentfamily$,  we have $\goodrevs_{\ell} \le \tmpcapacity$. Note that two  summands in~\eqref{eqn:tricky_KL2} are proportional to a convex combination of $\frac{\distance^2}{2 \genvariance(\singlehighqual)}$ and $\frac{\distance^2}{2 \genvariance(0)}$. Moreover, by monotonicity of $\genvariance$, we have $\frac{\distance^2}{2 \genvariance(\singlehighqual)} \ge \frac{\distance^2}{2 \genvariance(0)}$, and hence
\begin{align*}
	\sup\limits_{\assignment \in \assignmentfamily} \max\limits_{{\ell_1} \ne {\ell_2} \in \hypset} \KL{\scoredist^{(\assignment, \ell_1)}}{\scoredist^{(\assignment, \ell_2)}} \le \frac{\tmpcapacity \distance^2}{\genvariance(\singlehighqual)} + \frac{\left(\paperload - \tmpcapacity\right) \distance^2}{\genvariance(0)} = \distance^2	\left( \frac{\tmpcapacity \genvariance(0) + \left(\paperload - \tmpcapacity\right) \genvariance(\singlehighqual)}{\genvariance(\singlehighqual) \genvariance(0)}\right).
\end{align*}

Applying Fano's ineqaulity~\eqref{eqn:fano}, we conclude that for all feasible assignments $\assignment \in \assignmentfamily$, if $\numpapers > 6$ and universal constant $\const$ is sufficiently small, then 
\begin{align*}
	\prob{\hypest(\scorematrix) \ne {\problemvar}} \ge 1 - \frac{ \distance^2	\left( \frac{\tmpcapacity \genvariance(0) + \left(\paperload - \tmpcapacity\right) \genvariance(\singlehighqual)}{\genvariance(\singlehighqual) \genvariance(0)}\right) + \log 2}{\log \left(\numpapers - \sizeofconf + 1\right)} \ge  1 - \frac{ \const^2 \ln \numpapers + 1}{\log \left(\frac{\numpapers}{2} + 1\right)}  \ge \frac{1}{2}.	
\end{align*} 	
This bound thus implies 
\begin{align*}
	\sup\limits_{\simmatrix \in \onegoodforeveryone(\singlehighqual)}  \inf\limits_{\left(\estimator, \assignment \in \assignmentfamily \right)} \sup\limits_{ \left(\truepaperquality_1, \ldots, \truepaperquality_{\numpapers} \right) \in \problemfamily(\distance) } \prob{\accepted_{\sizeofconf}\left( \assignment, \estimator \right) \ne \truebest}  \ge \frac{1}{2}.		
\end{align*}


\subsection{Proof of Theorem~\ref{thm:main_theorem_ham}}

Before we prove the theorem, we state an auxiliary proposition which will help us to prove a lower bound.

\begin{lemma}[\citealp{shah15pairwise}]
\label{lemma:lev_bound}
	Let $\tol > 0$ be an integer such that $2 \tol \le \frac{1}{1 + \constnub}\min\left\{\numpapers^{1 - \constnua}, \sizeofconf, \numpapers - \sizeofconf \right\}$ for some constants $\constnua, \constnub \in (0; 1)$ 	and $\numpapers$ is larger than some $\left(\constnua, \constnub \right)$-dependent constant. Then there exist a set of binary strings $\left\{\str^1, \str^2, \ldots, \str^{\levbound} \right\} \subseteq \left\{0, 1 \right\}^{\numpapers/2}$ with cardinality $\levbound > \exponent{\frac{9}{10} \constnua \constnub \tol \log \numpapers}$ such that
	\begin{align*}
		\hamming{\str^{\ell_1}}{\mathbf{0}_{\numpapers/2}} = 2(1 + \constnub) \tol \quad \text{and} \quad \hamming{\str^{\ell_1}}{\str^{\ell_2}}	 > 4 \tol \quad \forall \ell_1 \ne \ell_2 \in [\levbound]
	\end{align*}
\end{lemma}

The proof of Lemma~\ref{lemma:lev_bound} relies on a coding-theoretic result due to~\cite{lev71codes} which gives a lower bound on the number of codewords of fixed length $\numpapers$ and Hamming weights $\const_1$ with Hamming distance between each pair of codewords higher than $\const_2$.
	
\subsubsection{Proof of upper bound}
\label{sec:upper_bound_ham_avg}

Without loss of generality we assume that the true underlying ranking of the papers is $1, 2, \ldots, \sizeofconf, \ldots, \numpapers$. We prove the claim for pair $\left(\fairassignment_{\transfavg}, \avgest \right)$ below, and proof for $\left(\fairassignment_{\transfmle}, \mle \right)$ follows from the proof of the corresponding part of Theorem~\ref{thm:main_theorem}(a).

 From the proof of Lemma~\ref{lemma:upper_bound} and Section~\ref{sec:main_thm_upper}, we know that 
    under conditions of the theorem, for every paper $j_1 \le \sizeofconf - \tol$ and for every paper $j_2 \ge \sizeofconf + \tol + 1$,
\begin{align}
\label{eqn:ham_small_error}
	\sup\limits_{\simmatrix \in \goodavgassignment(\highqualitycrowd)} \prob{\estpaperqualityavg_{j_1} - \estpaperqualityavg_{j_2} \le 0} \le \exponent{- \left(\frac{\distance}{2 \sup\limits_{\simmatrix \in \goodavgassignment(\highqualitycrowd)} \criticalstd({\fairassignment_{\transfavg}, \avgest})} \right)^2}
\end{align}
where
\begin{align}
\label{eqn:ham_small_var}
	\sup\limits_{\simmatrix \in \goodavgassignment(\highqualitycrowd)} \criticalstd^2(\fairassignment_{\transfavg}, \avgest) \le \frac{\paperload - \approximation_{\highqualitycrowd} \highqualitycrowd}{\paperload^2}.	
\end{align}
Taking a union bound across every paper from the top $(\sizeofconf - \tol)$ papers, paired with the bottom $(\numpapers - \sizeofconf - \tol)$ papers, we obtain
\begin{align*}
	\sup\limits_{\simmatrix \in \goodavgassignment(\highqualitycrowd)} \prob{\exists j_1 \le \sizeofconf - \tol, j_2 \ge \sizeofconf + \tol + 1 \quad \text{such that } \estpaperqualityavg_{j_1} \le \estpaperqualityavg_{j_2}} \le \numpapers^2 \exponent{-\frac{\paperload^2 \distance^2}{4(\paperload - \approximation_{\highqualitycrowd} \highqualitycrowd)}} \le \errorrate.
\end{align*}
In other words, for every similarity matrix $\simmatrix \in \goodavgassignment(\highqualitycrowd)$, with probability at least $(1 - \errorrate)$, the top $(\sizeofconf - \tol)$ papers will receive higher score than bottom $(\numpapers - \sizeofconf - \tol)$ papers. Thus, among accepted papers $\accepted_{\sizeofconf}\left(\fairassignment_{\transfavg}, \avgest \right)$, at most $\tol$ papers will not belong to $\truebest$, thereby ensuring that
\begin{align*}
	\hamming{\accepted_{\sizeofconf}\left(\fairassignment_{\transfavg}, \avgest \right)}{\truebest} \le 2 \tol	
\end{align*}
with probability at least $1 - \errorrate$.

\subsubsection{Proof of lower bound}
\label{sec:lower_bound_ham_avg}

To prove the lower bound, we follow similar path as we used when we derived a lower bound in Theorem~\ref{thm:main_theorem}.  However, we now need more advanced technique to construct necessary set of instances.

As in the proof of Theorem~\ref{thm:main_theorem}(b), we assume that $\sizeofconf \le \frac{\numpapers}{2}$. If the converse holds, than the result holds by the symmetry of the problem. Next, consider similarity matrix $\tmpsimmatrix = \left\{\genvariance^{-1}\left(1 - \frac{\highqualitycrowd}{\paperload}\right) \right\}^{\numreviewers \times \numpapers} \in \goodavgassignment(\highqualitycrowd)$. To apply Fano's inequality, it remains to construct a set $\hypset = \left\{1, 2, \ldots, \levbound \right\}$ of suitable instances of paper accepting/rejecting problem: every problem instance in this set has the same similarity matrix $\tmpsimmatrix$, but differs in the set of top $\sizeofconf$ papers $\truebest$. We note that in contrast to the proof of Theorem~\ref{thm:main_theorem}(b), it is not enough to create $(\numpapers - \sizeofconf + 1)$ instances where the sets of top $\sizeofconf$ papers differ only in a single paper. As we will see below, it suffices to construct instances such that for every $\ell_1, \ell_2 \in \hypset$, the sets of top $\sizeofconf$ papers satisfy $\hamming{\truebest(\ell_1)}{\truebest(\ell_2)} > 4 \tol$.

Note that requirements of Lemma~\ref{lemma:lev_bound} are satisfied by the conditions of Theorem~\ref{thm:main_theorem_ham}. Let $\left\{\str^1, \str^2, \ldots, \str^{\levbound} \right\}$ be the corresponding binary strings. For every problem $\ell \in \hypset$, consider the following binary string:
\begin{align}
\label{eqn:needed_string}
	\widetilde{b}^{\ell} = \overbrace{\underbrace{1, 1, \ldots, 1}_{\sizeofconf - 2 (1 + \constnub) \tol},	{0, 0, \ldots, 0}}^{\numpapers/2}, {\str^{\ell}_1, \str^{\ell}_2, \ldots, \str^{\ell}_{\numpapers/2}}.
\end{align}

First, note that  $2\tol \le \frac{1}{1 + \constnub}\sizeofconf$, and hence $\sizeofconf - 2 (1 + \constnub) \tol \ge 0$, thereby ensuring that the construction~\eqref{eqn:needed_string} is not vacuous. Now let $\truebest(\ell)$ be the set of indices such that their corresponding elements in string $\widetilde{b}^{\ell}$ equal $1$. By construction, the cardinality of $\truebest(\ell)$ is $\sizeofconf$ so it is a valid set of top $\sizeofconf$ papers. Finally, we need to set the scores of papers. Let for every paper $j \in [\numpapers]$:
\begin{align*}
	\truepaperquality_j(\ell) = \begin{cases}
 									\distance &   \text{if} \quad \widetilde{b}^{\ell}_j = 1 \\
 									0 & \text{if} \quad \widetilde{b}^{\ell}_j = 0, \\
 							    \end{cases}
\end{align*}
which ensures that for every $\ell \in \hypset$, $(\truepaperquality_1(\ell), \truepaperquality_2(\ell), \ldots, \truepaperquality_{\numpapers}(\ell)) \in \problemfamily \subset \problemfamilyham$.

The strategy for the remaining part of the proof is the following. We first show that the problem instances defined above are well-separated in a sense that for any two of them, the corresponding sets of the top $\sizeofconf$ papers differ in sufficiently many elements. We then assume that there exists an (assignment algorithm, estimator) pair which for every similarity matrix $\simmatrix \in \goodavgassignment(\highqualitycrowd)$ recovers the set of top $\sizeofconf$ papers with at most $\tol$ errors with high probability. Then this pair must be able to determine with high probability the problem instance $\ell$, sampled uniformly at random from $\hypset$, by observing corresponding reviewers' scores. We then apply Fano's inequality to show the impossibility of the last implication. 

Following the plan described above, we note that for every two distinct instances $\ell_1, \ell_2 \in \hypset$, we have
\begin{align*}
	\hamming{\truebest(\ell_1)}{\truebest(\ell_2)} > 4 \tol. 	
\end{align*}
Consequently, for every set $\truebest$ of $\sizeofconf$ papers, $\hamming{\truebest}{\truebest(\ell)} \le 2 \tol$ for at most one instance $\ell \in \hypset$. Now assume for the sake of contradiction that for every similarity matrix $\simmatrix \in \goodavgassignment(\highqualitycrowd)$, there exists an assignment $\tmpassignment = \tmpassignment\left(\simmatrix\right)$ and estimator $\estimator = \estimator\left(\simmatrix \right)$ such that for arbitrarily large value of $\numpapers$
\begin{align}
\label{eqn:assumption}
	\sup\limits_{ \left(\truepaperquality_1, \ldots, \truepaperquality_{\numpapers} \right) \in \problemfamily(\distance) } \prob{\hamming{\accepted_{\sizeofconf}\left( \tmpassignment, \estimator \right)}{\truebest} > 2 \tol}  < \frac{1}{2}.	
\end{align}
This assumption implies that estimator $\estimator(\tmpsimmatrix)$  might be used to determine the problem $\problemvar = \ell$ sampled uniformly at random from $\hypset$ correctly with probability greater than $1/2$. Indeed, notice that similarity matrix $\tmpsimmatrix$ was constructed in a way that $\accepted_{\sizeofconf}\left( \tmpassignment, \estimator \right)$ does not depend on assignment $\tmpassignment$.

Given $\problemvar = \ell$, let $\scorematrix^{(\ell)}$ be the random matrix of reviewers' scores. The distribution $\scoredist^{(\ell)}$ of components of $\scorematrix^{(\ell)}$ is defined in~\eqref{eqn:score_matrix_dist}. To apply Fano's inequality~\eqref{eqn:fano}, it remains to derive an upper bound on the quantity $\max\limits_{\ell_1 \ne \ell_2 \in \hypset} \KL{\scoredist^{(\ell_1)}}{\scoredist^{(\ell_2)}}$. 

First, note that entries of $\scorematrix^{(\ell)}$ are independent. Second, note that for every pair $\ell_1 \ne \ell_2 \in \hypset$ and for every $j \in [\numpapers/2]$, the distribution of the $j^{\text{th}}$ column of $\scorematrix^{({\assignment}, \ell_1)}$ is identical to the distribution of the $j^{\text{th}}$ column of  $\scorematrix^{({\assignment}, \ell_2)}$. Among the last $\numpapers/2$ columns, the distributions of at most $4(1 + \constnub) \tol$ columns of $\scorematrix^{({\assignment}, \ell_1)}$ differ from the distributions of the corresponding columns in $\scorematrix^{({\assignment}, \ell_2)}$. Thus, for arbitrary $\ell_1 \ne \ell_2 \in \hypset$
\begin{align*}
	\KL{\scoredist^{(\ell_1)}}{\scoredist^{(\ell_2)}} \le 2(1 + \constnub) \tol \paperload \left\{ \KL{\gaussian\left(\distance, 1 - \frac{\highqualitycrowd}{\paperload}\right)}{\gaussian\left(0, 1 - \frac{\highqualitycrowd}{\paperload} \right)} + \KL{\gaussian\left(0, 1 - \frac{\highqualitycrowd}{\paperload}\right)}{\gaussian\left(\distance, 1 - \frac{\highqualitycrowd}{\paperload} \right)}\right\}.		
\end{align*}

Recalling~\eqref{eqn:KL_gaussian}, we deduce that
\begin{align*}
	\max\limits_{\ell_1 \ne \ell_2 \in \hypset} 	\KL{\scoredist^{(\ell_1)}}{\scoredist^{(\ell_2)}} \le 4(1 + \constnub) \tol \paperload \frac{\paperload \distance^2}{2(\paperload - \highqualitycrowd)} = 2(1 + \constnub) \tol \frac{\paperload^2 \distance^2}{\paperload - \highqualitycrowd} \le 4 \const^2  \constnua \constnub \tol \ln \numpapers.
\end{align*}
Finally, Fano's inequality together with Lemma~\ref{lemma:lev_bound} ensures that for every estimator $\hypest: \scorematrix \to \hypset$
\begin{align*}
	\prob{\hypest(\scorematrix) \ne P} \ge 1 - \frac{4 \const^2  \constnua \constnub \tol \ln \numpapers + \log 2}{\frac{9}{10} \constnua \constnub \tol \log \numpapers} \ge 1 - \frac{40}{9} \const^2 \frac{\ln \numpapers}{\log \numpapers} - \frac{1}{\frac{9}{10} \constnua \constnub \tol \log \numpapers} \ge \frac{1}{2}
\end{align*}
for  $\numpapers$ larger than some $(\constnua, \constnub)$-dependent constant and small enough universal constant $\const$. This leads to a contradiction with~\eqref{eqn:assumption}, thus proving the theorem.


\subsection{Proof of Corollary~\ref{thm:main_theorem_ham_mle}} 

The proof of the Corollary~\ref{thm:main_theorem_ham_mle} is based on the ideas of the proofs of Theorem~\ref{thm:main_theorem_ham} and Corollary~\ref{thm:main_mle} and repeats them with minor changes.


\subsubsection{Proof of upper bound}

To show the required upper bound, we repeat the proof of Theorem~\ref{thm:main_theorem_ham_mle}(a) from Section~\ref{sec:upper_bound_ham_avg} with the following changes. Equation~\eqref{eqn:ham_small_error} should be substituted with:
\begin{align*}
	\sup\limits_{\simmatrix \in \onegoodforeveryone(\singlehighqual)} \prob{\estpaperqualitymle_{j_1} - \estpaperqualitymle_{j_2} \le 0} \le \exponent{- \left(\frac{\distance}{2 \sup\limits_{\simmatrix \in \onegoodforeveryone(\singlehighqual)} \criticalstd({\fairassignment_{\transfmle}, \mle})} \right)^2}.
\end{align*}
Equation~\eqref{eqn:ham_small_var} should be substituted with:
\begin{align*}
	\sup\limits_{\simmatrix \in \onegoodforeveryone(\singlehighqual)} \criticalstd^2({\fairassignment, \mle}) \le 	\frac{\genvariance(\singlehighqual) \genvariance(0)}{\tmpcapacity \genvariance(0) + ({\paperload - \tmpcapacity}) \genvariance(\singlehighqual)}.
\end{align*}
In the remaining part of the proof, pair $(\fairassignment_{\transfavg}, \avgest)$ should be substituted with the pair $(\fairassignment_{\transfmle}, \mle)$.


\subsubsection{Proof of lower bound}

To prove the lower bound, we use the set of problems $\hypset$ constructed in Section~\ref{sec:lower_bound_ham_avg} and the similarity matrix $\tmpsimmatrix$ as defined in~\eqref{eqn:simmatrix_tricky}. 

Given $\problemvar = \ell$ and any feasible assignment $\assignment \in \assignmentfamily$, let $\scorematrix^{({\assignment}, \ell)}$ be the random matrix of reviewers' scores. The distribution $\scoredist^{({\assignment}, \ell)}$ of components of $\scorematrix^{({\assignment}, \ell)}$ is defined in~\eqref{eqn:score_distr_tricky}. Since the distribution of reviewers' scores now depends on the assignment, to apply Fano's inequality~\eqref{eqn:fano}, we need to derive an upper bound on the quantity $\sup\limits_{\assignment \in \assignmentfamily} \max\limits_{\ell_1 \ne \ell_2 \in \hypset} \KL{\scoredist^{({\assignment}, \ell_1)}}{\scoredist^{({\assignment}, \ell_2)}}$.

First, note that entries of $\scorematrix^{({\assignment}, \ell)}$ are mutually independent. Second, note that for every pair $\ell_1 \ne \ell_2 \in \hypset$ and for every $j \in [\numpapers/2]$, the distribution of the $j^{\text{th}}$ column of $\scorematrix^{({\assignment}, \ell_1)}$ is identical to the distribution of the $j^{\text{th}}$ column of  $\scorematrix^{({\assignment}, \ell_2)}$. Among the last $\numpapers/2$ columns, the distributions of at most $4(1 + \constnub) \tol$ columns of $\scorematrix^{({\assignment}, \ell_1)}$ differ from the distributions of the corresponding columns in $\scorematrix^{({\assignment}, \ell_2)}$. Next, consider arbitrary feasible assignment $\assignment \in \assignmentfamily$. Let $\goodrevs_{\ell_{1}}^{(r)}, r \in [2(1 + \constnub)\tol],$ denote the number of strong reviewers (with similarity $\singlehighqual$) assigned in $\assignment$ to paper $j_1^{(r)} \in \truebest(\ell_1)$, where paper $j_1^{(r)}$ corresponds to the the second part of the string $\widetilde{b}^{\ell_1}$ defined in~\eqref{eqn:needed_string}. Recall now that there are at most $4(1 + \constnub) \tol$ papers that belong to exactly one of the sets $\truebest(\ell_1)$ and $\truebest(\ell_2)$. Hence, the equation for upper bound of the Kullback-Leibler divergence between $\scoredist^{({\assignment}, \ell_1)}$ and $\scoredist^{({\assignment}, \ell_2)}$ is obtained by assuming that all the papers that belong to the $\truebest(\ell_1)$ and correspond to the second half of the string $\widetilde{b}^{\ell}$ do not belong to $\truebest(\ell_2)$ and vice versa. Thus, similar to~\eqref{eqn:tricky_KL}-\eqref{eqn:tricky_KL2},  for arbitrary $\ell_1 \ne \ell_2 \in \hypset$ and for arbitrary feasible assignment $\assignment \in \assignmentfamily$, we have
\begin{align*}
	\KL{\scoredist^{({\assignment}, \ell_1)}}{\scoredist^{({\assignment}, \ell_2)}} &\le \slim_{r=1}^{2(1 + \constnub) \tol}  \left\{ \goodrevs_{\ell_{1}}^{(r)} \KL{\gaussian\big( \distance,  \genvariance(\singlehighqual) \big)}{\gaussian\big( 0,  \genvariance(\singlehighqual) \big)} + \left(\paperload - \goodrevs_{\ell_1}^{(r)} \right) \KL{\gaussian\big( \distance,  \genvariance({0}) \big)}{\gaussian\big( 0,  \genvariance({0}) \big)} \right\} \\ &+ \slim_{r=1}^{2(1 + \constnub) \tol}  \left\{ \goodrevs_{\ell_{2}}^{(r)} \KL{\gaussian\big( 0,  \genvariance(\singlehighqual) \big)}{\gaussian\big( \distance,  \genvariance(\singlehighqual) \big)} + \left(\paperload - \goodrevs_{\ell_2}^{(r)} \right) \KL{\gaussian\big( 0,  \genvariance({0}) \big)}{\gaussian\big( \distance,  \genvariance({0}) \big)} \right\} \\ & = \left( \slim_{r=1}^{2(1 + \constnub) \tol} \left( \goodrevs_{\ell_1}^{(r)} + \goodrevs_{\ell_2}^{(r)} \right) \right) \frac{\distance^2}{2 \genvariance(\singlehighqual)} + \left(4(1 + \constnub) \tol \paperload - \slim_{r=1}^{2(1 + \constnub) \tol}\left( \goodrevs_{\ell_1}^{(r)} + \goodrevs_{\ell_2}^{(r)} \right) \right) \frac{\distance^2}{2 \genvariance(0)}. 
\end{align*}

Noting that $\frac{\distance^2}{2 \genvariance(\singlehighqual)} \ge \frac{\distance^2}{2 \genvariance(0)}$, we obtain
\begin{align*}
	\sup\limits_{\assignment \in \assignmentfamily} \max\limits_{{\ell_1} \ne {\ell_2} \in \hypset} \KL{\scoredist^{(\assignment, \ell_1)}}{\scoredist^{(\assignment, \ell_2)}} & \le 2(1 + \constnub) \tol  \left( \frac{\tmpcapacity \distance^2}{\genvariance(\singlehighqual)} + \frac{\left(\paperload - \tmpcapacity\right) \distance^2}{\genvariance(0)}\right) \\ & = 2(1 + \constnub) \tol \distance^2 \left( \frac{\tmpcapacity \genvariance(0) + \left(\paperload - \tmpcapacity\right) \genvariance(\singlehighqual)}{\genvariance(\singlehighqual) \genvariance(0)}\right) \\ & \le 4 \const^2  \constnua \constnub \tol \ln \numpapers. 
\end{align*}
Applying Fano's inequality~\eqref{eqn:fano}, we obtain the desired lower bound.


\subsection{Proof of Theorem~\ref{thm:main_theorem_subj}} 

Note that Theorem~\ref{thm:main_theorem_subj} is similar in nature with Theorem~\ref{thm:main_theorem}, the only difference is that now we are trying to recover a ranking which is induced by the assignment.


\subsubsection{Proof of upper bound}

Given any feasible assignment $\assignment$, the ``ground truth'' ranking that we try to recover is given by
\begin{align}
\label{eqn:induced_ranking}
	\avgpaperquality_j(\assignment) = \frac{1}{\paperload} \slim_{i \in \reviewerset{\assignment}(j)} \subjectivescore_{ij}.	
\end{align}
Then the estimates $\estpaperqualityavg_j, j \in [\numpapers]$, are distributed as
	\begin{align}
	\label{eqn:distr_ham_subj}
		\estpaperqualityavg_j \sim \gaussian\left(\frac{1}{\paperload} \slim_{i \in \reviewerset{\assignment}(j)} \subjectivescore_{ij}	, \frac{1}{\paperload^2} \slim_{i \in \reviewerset{\assignment}(j)} \std_{ij}^2 \right) = \gaussian\left(\avgpaperquality_j(\assignment), \cumvar_j^2 \right), 	
	\end{align}
where $\cumvar_j^2 = \frac{1}{\paperload^2} \slim_{i \in \reviewerset{\assignment}(j)} \std_{ij}^2$. Now observe that Lemma~\ref{lemma:upper_bound}, with $\truebestsubj\left(\assignment, \avgpaperquality(\assignment) \right)$ substituted for $\truebest$, also holds for the subjective score model and the averaging estimator $\avgest$. Thus, repeating the proof of the upper bound for averaging estimator in Theorem~\ref{thm:main_theorem}(a) and substituting $\truebest$ with $\truebestsubj\left(\fairassignment, \avgpaperquality(\fairassignment) \right)$ in~\eqref{eqn:chain1}, yields the claimed result.


\subsubsection{Proof of lower bound} 
\label{sec:subj_main_lower}

The lower bound directly follows from Theorem~\ref{thm:main_theorem}(b). To see this, consider the following matrix of reviewers' subjective scores:
$ \subjscorematrix = \left\{\subjectivescore_{ij} \right\}_{i \in [\numreviewers], j \in [\numpapers]}$, where $\subjectivescore_{ij} = \truepaperquality_{j}$.	 Under this assumption, the total ranking induced by assignment $\assignment$ does not depend on the assignment: $\avgpaperquality_j(\assignment) = \truepaperquality_j$.  Now we can conclude that such choice of $\subjscorematrix$ brings us to the objective model setup in which true underlying ranking exists and does not depend on the assignment. Thus, the lower bound of Theorem~\ref{thm:main_theorem}(b) transfers to the subjective score model.


\subsection{Proof of Theorem~\ref{thm:main_theorem_ham_subj}} 

The proof of the Theorem~\ref{thm:main_theorem_ham_subj} is based on the ideas of the proofs of Theorem~\ref{thm:main_theorem_ham} and Theorem~\ref{thm:main_theorem_subj} and repeats them with minor changes.


\subsubsection{Proof of upper bound}

Having equations~\eqref{eqn:induced_ranking} and~\eqref{eqn:distr_ham_subj}, we note that the goal now mimics the goal we achieved when proved an upper bound for averaging estimator in Theorem~\ref{thm:main_theorem_ham}.


\subsubsection{Proof of lower bound}

The argument from Section~\ref{sec:subj_main_lower} ensures that the lower bound established in Theorem~\ref{thm:main_theorem_ham} directly transfers to the to the subjective score model.


\section{Discussion}
\label{sec:discussion}

Researchers submit papers to conferences expecting a fair outcome from the peer-review process. This expectation is often not met, as is illustrated by the difficulties that non-mainstream or inter-disciplinary research faces in present peer-review systems.  We design a reviewer-assignment algorithm \algo to address the crucial issues of fairness and accuracy. Our guarantees impart promise for deploying the algorithm in conference peer-reviews.  

There are number of open problems suggested by our work. The first  direction is associated with approximation algorithms and corresponding guarantees established in this work. One goal is to determine whether there exists a polynomial-time algorithm with worst case approximation guarantees better than $1/\paperload$ established in this paper~\eqref{eqn:deterministic_simplified}. It would also be useful to obtain a deeper understanding of the adaptive behavior of our algorithm with bounds more nuanced than~\eqref{eqn:deterministic_complicated}. Finally, we leave the task of improving the computational efficiency of our \algo algorithm out of the scope of this work. However, we suggest that optimal implementation of Subroutine~\ref{alg:fair_subroutine} should not be based on the general max-flow algorithm and instead should rely on algorithms specifically designed to work fast on layered graphs.

The second direction is related to the statistical part of our work. In this paper we provide a minimax characterization of the simplified version of the paper acceptance problem. This simplified procedure may be considered as an initial estimate that can be used as a guideline for the final decisions.  However, there remain a number of other factors, such as self-reported confidence of reviewers or inter-reviewer discussions, that may additionally be included in the model.

Finally, an important related problem is to improve the assessment of similarities between reviewers and papers. It will be interesting to see whether the problems of assessing similarities and assigning reviewers can be addressed jointly in an active manner possibly incorporating feedback from the previous iterations of the conference


\section*{Acknowledgments}
This work was supported in parts by NSF grants CRII: CIF: 1755656, CIF: 1563918, and CIF: 1763734.

\bibliography{paper}
\bibliographystyle{apalike}


\vspace{30pt}

\newcommand{\specsim}{\widetilde{s}}

\newpage

\appendix
\noindent{\Large \bf Appendix}
\bigskip

\noindent
We provide supplementary materials and additional discussion. 


\section{Discussion of approximation results}
	\label{appendix:approximation_discussion}
		
In this section we discuss the approximation-related results. In what follows we consider function $f(\similarity) = \similarity$ and for any value $\const \in \reals$, we denote the matrix all of whose entries are $\const$ as $\mathbf{c}$.


\subsection{Example for \alggarg algorithm.}
\label{appendix:garg_discussion}

We begin by construction a series of similarity matrices for various $\paperload$ such that $\assignmentquality{\gargassignment} = 0$ while assignments $\fairassignment$ and $\hardassignment$ have non-trivial fairness. 

\begin{proposition}
	For every positive integer $\paperload$, there exists a similarity matrix $\simmatrix$ such that $\assignmentquality{\gargassignment} = 0$ and $\assignmentquality{\fairassignment} \ge \frac{1}{\paperload} \assignmentquality{\hardassignment} > 0$.
\end{proposition}
\begin{proof}
	Given any positive integer $\paperload \in \naturals$, consider an instance of reviewer assignment problem with $\numpapers = \numreviewers$, $\maxrevload = \paperload$ and similarities given by the block matrix 
		\begin{align}
		\label{eqn:simmatrix_appendix}
        \simmatrix = \begin{array}{c@{\!\!\!}l}
            \left[ \begin{array}{c|c|c}
							\mathbf{1} & \mathbf{1} &\mathbf{0}\\
							\hline
							\mathbf{0} &\mathbf{0} & ({\specsim} - {\varepsilon}) \cdot \mathbf{1} \\
							\hline
							\smash{\underbrace{ \addstackgap[2pt] {$(\specsim - \varepsilon) \cdot \mathbf{1}$}}_{\numpapers_1}}  & \smash{\underbrace{ \addstackgap[2pt] {${(\specsim - \varepsilon)} \cdot \mathbf{1}$}}_{\numpapers_1}} & \smash{\underbrace{ \addstackgap[2pt] {$\hspace{12pt} \specsim \cdot \mathbf{1} \hspace{12pt}$}}_{\numpapers_1}}  \\
		   \end{array} \right]
        &
         \begin{array}[c]{@{}l@{\,}l}
           \left. \,\,\, \right\} & \text{$\numreviewers_1$} \\
           \left.  \,\,\, \right\} & \text{$\numreviewers_2$} \\
           \left.  \,\,\, \right\} & \text{$\numreviewers_3$} \\
           \end{array}
           \end{array}
        \end{align}
	\\~\\~
	Here $\specsim = \frac{\numreviewers_1}{\numreviewers_1 + \numreviewers_2}$, the value $\varepsilon > 0$ is some small constant strictly smaller than $\specsim$, and $\numreviewers_r = \numpapers_r > 0$ for every $r \in 
	\{1,2,3\}$. We also require $\numreviewers_3 > \paperload$ and
	\begin{align}
		\label{eqn:requirements}
		\numreviewers_2 = \left(\paperload - 1 \right) \numreviewers_1 + 1.	
	\end{align}
	We refer to the first $\numpapers_1$ papers and $\numreviewers_1$ reviewers as belonging to the first group, the second $\numpapers_2$ papers and $\numreviewers_2$ reviewers as belonging to the second group, and so on. 
	
	The \alggarg algorithm involves two steps. The first step consists of solving a linear programming relaxation and finding the most fair fractional assignment. The second step then performs a rounding procedure in order to obtain integer assignments. Let us first see the output of the first step of the \alggarg algorithm --- the fractional assignment with the highest fairness --- on the similarity matrix~\eqref{eqn:simmatrix_appendix}. Observe that for each of the $\numpapers_3$ papers in the third group, the sum of the similarities of any $\paperload$ reviewers is at most  $\paperload \specsim$, and furthermore, that this value is achieved with equality if and only if they are reviewed by  $\paperload$ reviewers from the third group. Next, the $\numreviewers_1$ reviewers from the first group can together review $\paperload \numreviewers_1$ papers. Dividing this amount equally over the $\numpapers_1 + \numpapers_2$ papers in the first two groups (in any arbitrary manner) and complementing the assignment with reviewers from the second group, we see that each paper from the first and the second groups receives a sum similarity $\paperload \frac{\numreviewers_1}{\numpapers_1 + \numpapers_2} = \paperload \specsim$. It is not hard to see that any deviation from the assignment introduced above will lead to a strict decrease of the fairness. 
	
The second step of the \alggarg algorithm is a rounding procedure that constructs a feasible assignment from the fractional assignment (solution of linear programming relaxation) obtained in the previous step. The rounding procedure is guaranteed to assign $\paperload$ reviewers to each paper, respecting the following condition: any reviewer assigned to any paper $j\in [\numpapers]$ in the resulting feasible assignment must have a non-zero fraction allocated to that paper in the fractional assignment. 
	
Now notice that aforementioned condition ensures that all papers from the third group must be assigned to reviewers from the third group. Next, recall that on one hand, reviewers from the first group can together review at most $\paperload \numreviewers_1$ different papers. On the other hand, in each optimally fair fractional assignment, the first $\numpapers_1 + \numpapers_2$ papers are assigned to reviewers from the first two groups. Thus, in the resulting integral assignment these papers also must be assigned to reviewers from the first two groups. These two facts together with the inequality $\paperload \numreviewers_1 < \numpapers_1 + \numpapers_2$ that we obtain from~\eqref{eqn:requirements} ensure that at least one paper in the resulting integral assignment will be reviewed by $\paperload$ reviewers with zero similarity. Hence, the assignment computed by the \alggarg algorithm has zero fairness $\assignmentquality{\gargassignment} = 0$.

    On the other hand, it is not hard to see that $\assignmentquality{\hardassignment} \ge \specsim - \varepsilon$. Indeed, let us assign one reviewer to each paper by the following procedure: the $\numpapers_1$ papers from the first group and some $\numpapers_2-1$ papers from the second group are all assigned one arbitrary reviewer each from the first group of reviewers. Such an assignment is possible since $\paperload \numreviewers_1 = \numpapers_1 + \numpapers_2 - 1$ due to~\eqref{eqn:requirements}. The remaining paper from the second group is assigned one arbitrary reviewer from the third group. At this point, there are $\numpapers_3$ papers (in the third group) which are not yet assigned to any reviewer, and $\numreviewers_3 + \numreviewers_2 - 1 \ge \numpapers_3$ reviewers who have not been assigned any paper and have similarity higher than $\specsim - \varepsilon$ with these $\numpapers_3$ papers in the third group. Assigning one reviewer each  from this set to each of these $\numpapers_3$ papers, we obtain an assignment in which each paper is allocated to one reviewer with similarity at least $\specsim - \varepsilon$. Completing the remaining assignments in an arbitrary fashion, we conclude that $\assignmentquality{\fairassignment} \ge \frac{1}{\paperload} \assignmentquality{\hardassignment} \ge \specsim - \varepsilon > 0$ where first inequality is due to Theorem~\ref{thm:deterministic}.  
\end{proof}

The results of simulations for $\paperload \in \{1,2,3,4 \}$, parameters $\numreviewers_1 = 1, \numreviewers_2 = \paperload, \numreviewers_3 = \paperload + 1, \varepsilon = 0.01$ and similarity matrices $\tmpsimmatrix$ defined in~\eqref{eqn:simmatrix_appendix} are depicted in Table~\ref{table:example_garg}. Interestingly, for these choices of parameters, our \algo algorithm is not only superior to \alggarg,  but is also able to exactly recover the fair assignment.

\begin{table}[t]
\vskip 0.15in
\begin{center}
\begin{small}
\begin{sc}
\begin{tabular}{lcccr}
\toprule
          & $\paperload = 1$ &  $\paperload = 2$ &  $\paperload = 3$ &  $\paperload = 4$\\
\midrule
$\assignmentquality{\gargassignment}$ & $0$   & $0$   & $0$   & $0$    \\
$\assignmentquality{\hardassignment}$ & $0.49$     & $0.65$     & $0.72$ & $0.76$  \\
$\assignmentquality{\fairassignment}$ & $0.49$     & $0.65$     & $0.72$   & $0.76 $    \\
\bottomrule
\end{tabular}
\end{sc}
\end{small}
\end{center}
\vskip -0.1in
\caption{Fairness of various assignment algorithms for the class of similarity matrices~\eqref{eqn:simmatrix_appendix}.}
\label{table:example_garg}
\end{table}

\subsection{Sub-optimality of \algtpms}
\label{appendix:tpms_failure}

In this section we show that assignment obtained from optimizing the objective~\eqref{eqn:unfair_criteria} can be highly sub-optimal with respect to the criterion~\eqref{eqn:gen_fair_criteria} even when $\transformation$ is the identity function. 

\begin{proposition}
	For any $\paperload \ge 1$, there exists a similarity matrix $\simmatrix$ such that $\assignmentquality{\fairassignment} = \assignmentquality{\hardassignment} \ge \frac{\paperload}{4}$ and $\assignmentquality{\tpms} = 0$.
\end{proposition}

\begin{proof}
	Consider an instance of the problem with $\numpapers = \numreviewers = 2 \paperload$, and similarities given by the block matrix
	\begin{align}
		\label{eqn:simmatrix_tpms}
        \simmatrix = \begin{array}{c@{\!\!\!}l}
            \left[ \begin{array}{c|c}
							\mathbf{1} &  \mathbf{0.4} \\
							\hline
							\smash{\underbrace{ \addstackgap[2pt] {$\mathbf{0.4}$}}_{\paperload}}   & \smash{\underbrace{ \addstackgap[2pt] {$\mathbf{0}$}}_{\paperload}} \\
			\end{array} \right]
        &
         \begin{array}[c]{@{}l@{\,}l}
           \left. \,\,\, \right\} & \text{$\paperload$} \\
           \left.  \,\,\, \right\} & \text{$\paperload$} \\
           \end{array}
           \end{array}
        \end{align}
	\\~
	Then $\tpms$ assigns the first $\paperload$ reviewers to the first $\paperload$ papers (in some arbitrary manner) and the remaining reviewers to the remaining papers, obtaining
	\begin{align*}
		\slim_{j \in [\numpapers]} \slim_{i \in \reviewerset{\tpms}(j)} \similarity_{ij} & = \paperload^2 \text{ and} \\ \assignmentquality{\tpms} & = 0 	
	\end{align*}
	
	In contrast, assignments $\fairassignment$ and $\hardassignment$ assign the first $\frac{1}{2} \numreviewers$ reviewers to the second group of papers and the remaining reviewers to the remaining papers. This assignment yields
	\begin{align*}
		\slim_{j \in [\numpapers]} \slim_{i \in \reviewerset{\fairassignment}(j)} \similarity_{ij} &= \slim_{j \in [\numpapers]} \slim_{i \in \reviewerset{\hardassignment}(j)} \similarity_{ij}  = 0.8  \paperload^2  \text{ and}  \\
		\assignmentquality{\fairassignment} &= \assignmentquality{\hardassignment}  = 0.4  \paperload \ge \frac{\paperload}{4}.
		\end{align*}
 	This concludes the proof.
\end{proof}


\subsection{Example of $1/\paperload$ approximation factor for $\fairassignment$}
\label{appendix:pr4a_failure}

Let us consider an instance of fair assignment problem with $\numpapers = \numreviewers = 4, \ \paperload = \maxrevload = 2$ and similarities represented in Table~\ref{table:exmple_break}.

\begin{table}[t]
\vskip 0.15in
\begin{center}
\begin{small}
\begin{sc}
\begin{tabular}{lcccr}
\toprule
          & Paper $a$  & Paper $b$  & Paper $c$ & Paper $d$\\
\midrule
Reviewer $1$ & $0.3 + \errorrate$   & $1$   & $1$   & $0$    \\
Reviewer $2$ & $0.3 - \errorrate$     & $0$     & $1$ & $1$  \\
Reviewer $3$ & $0$     & $0.1$     & $0 $   & $0.3 $    \\
Reviewer $4$ & $0$     & $0.1$     & $0$   & $0.3$  \\
\bottomrule
\end{tabular}
\end{sc}
\end{small}
\end{center}
\vskip -0.1in
\caption{An example of similarities that yield $1/\paperload$ approximation factor of the \algo algorithm.}
\label{table:exmple_break}
\end{table}

First, note that $\assignmentquality{\hardassignment} \le 0.6$. This is because in every feasible assignment $\assignment \in \assignmentfamily$ paper $1$ in the best case is assigned to reviewers $1$ and $2$. Moreover, there exists a feasible assignment represented as $\hardassignment$ in Table~\ref{table:exmple_cand1} which achieves a max-min fairness of $0.6$ and hence we have $\assignmentquality{\hardassignment} = 0.6$.

Let us now analyze the performance of \algo algorithm. Again, the fairness of the resulting assignment is determined in the first iteration of Step~\ref{Algostep:loopkappa} to~\ref{Algostep:enditer} of Algorithm~\ref{alg:fair_assignment}, so we restrict our attention to that part of the algorithm. It is not hard to see that after Step~\ref{Algostep:loopkappa} is executed, we have two candidates assignments, $\assignment_1$ and $\assignment_2$, represented in Table~\ref{table:exmple_cand1} (up to not important randomness in braking ties). Computing the fairness of these assignments, we obtain
\begin{align*}
	& \assignmentquality{\assignment_1} = 0.3 + \varepsilon \quad \text{and} \quad   \assignmentquality{\assignment_2} = 0.2.
\end{align*}
which implies that
\begin{align*}
	\frac{\assignmentquality{\fairassignment}}{\assignmentquality{\hardassignment}} = \frac{\max\left\{\assignmentquality{\assignment_1}, \assignmentquality{\assignment_2} \right\}}{\assignmentquality{\hardassignment}} = \frac{1}{2} + \frac{\varepsilon}{0.6}.	
\end{align*}
Setting $\errorrate$ small enough, we can see that the approximation factor is very close to $1/2 = 1/\paperload$.

\begin{table}[t]
\vskip 0.15in
\begin{center}
\begin{small}
\begin{sc}
\begin{tabular}{lcccccc}
\toprule
          & \multicolumn{2}{c}{$\hardassignment$} & \multicolumn{2}{c}{$\assignment_1$} & \multicolumn{2}{c}{$\assignment_2$}  \\
          \cmidrule(lr){2-3}
          \cmidrule(lr){4-5}
          \cmidrule(lr){6-7}
          & $1^{\text{st}}$ Reviewer & $2^{\text{nd}}$ Reviewer & $1^{\text{st}}$ Reviewer & $2^{\text{nd}}$ Reviewer & $1^{\text{st}}$ Reviewer & $2^{\text{nd}}$ Reviewer \\
\midrule
Paper $a$ &  $1$ & $2$ & $1$     & $3$  &  $1$  &  $2$  \\
Paper $b$ &  $1$ & $3$ & $1$     & $3$  &  $3$  &  $4$  \\
Paper $c$ &  $2$ & $4$ & $2$     & $4$  &  $1$  &  $2$  \\
Paper $d$ &  $3$ & $4$ & $2$     & $4$  &  $3$  &  $4$  \\
\bottomrule
\end{tabular}
\end{sc}
\end{small}
\end{center}
\vskip -0.1in
\caption{The optimal assignment as well as and \algo's intermediate assignments for the similarities in Table~\ref{table:exmple_break}.}
\label{table:exmple_cand1}
\end{table}


\section{Computational aspects}
	\label{appendix:computational_aspects}

A  na\"ive implementation of the \algo algorithm has a polynomial computational complexity (under either an arbitrary choice or one computable in polynomial-time in Step~\ref{Sstep:pick_flow}) and requires $\mathcal{O}\left(\paperload \numpapers^2 \numreviewers\right)$ iterations of the max-flow algorithm. There are a number of additional ways that the algorithm may be optimized for improved computational complexity while retaining all the approximation and statistical guarantees. 

One may use Orlin's method~\citep{orlin13maxflow, king92maxflow} to compute the max-flow which yields a computational complexity of the entire algorithm at most $\mathcal{O}\left(\paperload (\numpapers + \numreviewers)\numpapers^3 \numreviewers^2 \right)$. Instead of adding edges is Step 3 of the subroutine one by one, a binary search may be implemented, reducing the number of max-flow iterations to $\mathcal{O}\left(\paperload \numpapers \log \numpapers \numreviewers\right)$ and the total complexity to $\widetilde{\mathcal{O}}\left(\paperload (\numpapers + \numreviewers)\numpapers^2 \numreviewers \right)$. 

Finally, note that the max-min approximation guarantees (Theorem~\ref{thm:deterministic}), as well as statistical results  (Theorems~\ref{thm:main_theorem} to~\ref{thm:main_theorem_ham_subj} and corresponding corollaries) remain valid even for the assignment $\tmpassignment$ computed in Step~\ref{Algostep:select_candidate} of Algorithm~\ref{alg:fair_assignment} during the \emph{first} iteration of the algorithm. The algorithm may thus be stopped at any time after the first iteration  if there is a strict time-deadline to be met. However, the results of Corollary~\ref{corr:seq_deterministic} on optimizing the assignment for papers beyond the most worst-off will not hold any more.\footnote{If the algorithm is terminated after $\totaliter'$ iterations, then bound~\eqref{eqn:seq_deterministic} from Corollary~\ref{corr:seq_deterministic} holds for $\currentiter \in [\totaliter']$.} The computational complexity of each of the iterations is at most $\widetilde{\mathcal{O}}\left(\paperload (\numpapers + \numreviewers)\numpapers \numreviewers \right)$, and stopping the algorithm after a constant number of iterations makes it  comparable to the complexity of \algtpms algorithm which is successfully implemented in many large scale conferences.

Let us now briefly compare the computational cost of \algo and \alggarg algorithms. The full version of \alggarg algorithm requires $\mathcal{O}(\numpapers^2)$ solutions of linear programming problems. Given that finding a max-flow in a graph constructed by our subroutine can be casted as linear programming problem (with constraints similar to those in~\citealt{Garg2010papers}), we conclude that slightly optimized implementation of our algorithm results in $\mathcal{O}(\paperload \numpapers \log \numpapers \numreviewers)$ solutions of linear programming problems, which is asymptotically better. To be fair, the \alggarg algorithm also can be terminated in an earlier stage with theoretical guarantees satisfied, which brings both algorithms on a similar footing with respect to the computational complexity.


\section{Topic coverage}
\label{appendix:modifications}

\newcommand{\topic}{t}
\newcommand{\distincttopics}{\omega}
\newcommand{\topicset}{T}

In this section we discuss an additional benefit of ``topic coverage'' that can be gained from the special choice of heuristic in Step~\ref{Sstep:pick_flow} of Subroutine~\ref{alg:fair_subroutine} of our \algo algorithm.

Research is now increasingly inter-disciplinary and consequently many papers submitted to modern conferences make contributions to multiple research fields and cannot be clearly attributed to any single research area. For instance, computer scientists often work in collaboration with physicists or medical researchers resulting in papers spanning different areas of research. Thus, it is important to maintain a {broad topic coverage}, that is, to ensure that such multidisciplinary papers are assigned to reviewers who not only have high similarities with the paper, but also represent the different research areas related to the paper. For example, if a paper proposes an algorithm to detect new particles in the CERN collider, then that paper should ideally be evaluated by competent physicists, computer scientists, and statisticians.  

There are prior works both in peer-review~\citep{Long13gooadandfair} and in text mining~\citep{Lin11:submodular} which propose a submodular objective function to incentivize topic coverage. According to~\citet{Long13gooadandfair}, the appropriate measure of coverage is a  number of distinct topics of the paper covered, summed across the all papers. Let us introduce a piece of notation to formally describe the underlying optimization problem. For every paper $j \in [\numpapers]$, let $\topicset(j) = \{\topic_1^{(j)}, \ldots, \topic_{r_j}^{(j)}\}$ be related research topics and for every reviewer $i \in [\numreviewers]$, let $\topicset(i) = \{\topic_1^{(i)}, \ldots, \topic_{r_i}^{(i)} \}$ be the topics of expertise of reviewer $i$. For every assignment $\assignment$, we define $\distincttopics(\assignment)$ to be the total number of distinct topics of all papers covered by the assigned reviewers:
\begin{align}
\label{eqn:diversity_obj}
    \distincttopics(\assignment) = \slim_{j \in [\numpapers]} \card\left(\bigcup_{i \in \reviewerset{\assignment}(j)} \left( \topicset(j) \bigcap \topicset(i) \right) \right),    
\end{align}
where $\card(\mathcal{C})$ denotes the number of elements in the set $\mathcal{C}$. 
The goal in~\citet{Long13gooadandfair} is to find an assignment that maximizes $\distincttopics(\assignment)$ and respects the constraints on the paper/reviewer load. However, instead of the requirement that each paper is assigned to  $\paperload$ reviewers as in our work, \citet{Long13gooadandfair} consider a relaxed version and require each paper to be reviewed by at most $\paperload$ reviewers.

Using the submodular nature of the objective~\eqref{eqn:diversity_obj}, \citet{Long13gooadandfair} propose a greedy algorithm that is guaranteed to achieve a constant-factor approximation of the optimal coverage~\eqref{eqn:diversity_obj}. This greedy algorithm, however, has the following two important drawbacks:
{
\renewcommand{\theenumi}{(\roman{enumi})}
\renewcommand{\labelenumi}{\theenumi}
\begin{enumerate} 
    \item Like the \algtpms algorithm, the greedy algorithm aims at optimizing the global functional, and consequently may fare poorly in terms of fairness. Indeed, in order to optimize the global objective~\eqref{eqn:diversity_obj}, the greedy algorithm may sacrifice the topic coverage for some of the papers, assigning relevant reviewers to other papers. \label{Coverage:issue1}
    \item While guaranteed to achieve a constant factor approximation of the objective~\eqref{eqn:diversity_obj}, the greedy algorithm may yield an assignment in which papers are reviewed by (much) less than $\paperload$ reviewers. It is not even guaranteed that in the resulting assignment each paper has at least one reviewer. \label{Coverage:issue2}
\end{enumerate}
}
Nevertheless, both the \algo algorithm and the algorithm of~\citet{Long13gooadandfair} can benefit from each other if the latter is used as a heuristic to choose a feasible assignment in Step~\ref{Sstep:pick_flow} of the subroutine of the former. In what follows we detail the procedure to combine the two algorithms. The greedy algorithm of~\citet{Long13gooadandfair} picks (reviewer, paper) pairs one-by-one and adds them to the assignment. At each step, it picks the pair that yields the largest incremental gain to~\eqref{eqn:diversity_obj} while still meeting  the paper/reviewer load constraints. In Step~\ref{Sstep:pick_flow} of the subroutine of \algodot, we may use the greedy algorithm, restricted to the (reviewer, paper) pairs added to the network in the previous steps, to find an assignment that approximately maximizes~\eqref{eqn:diversity_obj}. Next, for every (reviewer, paper) pair that belongs to this assignment, we set the cost of the corresponding edge in the flow network to $1$ and the costs of the remaining edges to $0$. Finally, we compute the maximum flow with maximum cost in the resulting network and fix (reviewer, paper) pairs that correspond to edges employed in that flow in the final output of the subroutine. 

Let us now discuss the benefits of this approach. First, in \algo we modify only the procedure of tie-breaking among max-flows, and hence all the guarantees established in the paper continue to hold. Second, the introduced procedure allows to overcome the issue~\ref{Coverage:issue2}, because the max-flow guarantees that each paper is assigned with exactly requested number of reviewers. Third, by setting the cost of selected edges to $1$, we encourage the topic coverage (although the pproximation guarantee of the greedy algorithm no longer holds). Finally, we do not allow the algorithm of~\citet{Long13gooadandfair} to sacrifice some papers in order to maximize the global coverage~\eqref{eqn:diversity_obj}, because the subroutine ensures that in the resulting assignment all the papers are assigned to pre-selected reviewers with high similarity, thereby overcoming~\ref{Coverage:issue1}.  

\end{document}